\documentclass[twoside,11pt]{article}

\usepackage{jmlr2e}

\usepackage{amsmath}
\usepackage{amssymb}
\usepackage{mathtools}
\usepackage{booktabs}
\usepackage{makecell}
\usepackage{tikz}
\usepackage{enumitem}
\usepackage{algorithm}
\usepackage{algorithmic}
\usepackage{etoc}
\usepackage{subfig}

\newtheorem{assumption}[theorem]{Assumption}

\newcommand{\R}{\mathbb{R}}

\newcommand{\dist}{\mathrm{dist}}

\newcommand\prox[1]{\mathrm{prox}_{#1}}

\newcommand\Fnor[1]{\mathrm{F}^{\gamma}_{\mathrm{nor}}(#1)}
\newcommand\bFnor[1]{\mathbf{F}^{\gamma}_{\mathrm{nor}}(#1)}
\newcommand{\Fnori}[2]{{\mathrm{F}}^{\gamma}_{#1,\mathrm{nor}}(#2)}

\newcommand\Fnat[1]{\mathrm{F}^{\gamma}_{\mathrm{nat}}(#1)}

\newcommand\dom[1]{\mathrm{dom}(#1)}

\newcommand{\order}[1]{\mathcal{O}\left(#1\right)}
\newcommand{\torder}[1]{\tilde{\mathcal{O}}\left(#1\right)}
\newcommand{\torderi}[1]{\tilde{\mathcal{O}}(#1)}
\newcommand{\prt}[1]{\left(#1\right)}
\newcommand{\brk}[1]{\left[#1\right]}
\newcommand{\crk}[1]{\left\{#1\right\}}

\newcommand{\orderi}[1]{\mathcal{O}(#1)}
\newcommand{\prti}[1]{(#1)}
\newcommand{\brki}[1]{[#1]}
\newcommand{\crki}[1]{\{#1\}}

\newcommand{\inpro}[1]{\left\langle #1 \right\rangle}
\newcommand{\inproi}[1]{\langle #1 \rangle}
\newcommand{\condE}[2]{\E\brk{#1\middle|#2}}
\newcommand{\condEi}[2]{\E[#1|#2]}
\newcommand{\norm}[1]{\left\Vert #1 \right\Vert}
\newcommand{\normi}[1]{\Vert #1 \Vert}
\newcommand{\bs}[1]{\boldsymbol{#1}}

\newcommand{\E}{\mathbb{E}}

\newcommand{\x}{\mathbf{x}}
\newcommand{\z}{\mathbf{z}}
\newcommand{\g}{\mathbf{g}}
\newcommand{\1}{\mathbf{1}}

\newcommand{\T}{\intercal}
\newcommand{\sumn}{\sum_{i=1}^n}

\newcommand{\y}{\mathbf{y}}

\newcommand{\cF}{\mathcal{F}}

\newcommand{\cL}{\mathcal{L}}

\newcommand{\cC}{\mathcal{C}}
\newcommand{\cH}{\mathcal{H}}

\newcommand{\bc}{\mathbf{c}}

\newcommand{\teta}{\hat{\eta}}

\newcommand{\ceil}[1]{\left\lceil#1\right\rceil}
\newcommand{\ceili}[1]{\lceil#1\rceil}

\newcommand{\tc}{\mathtt{C}}
\newcommand{\normfl}{FedNMap}
\newcommand{\Ln}{L_{\mathrm{F}}}
\newcommand{\tk}{\boldsymbol{t}}

\newcommand\proxp[1]{\mathrm{prox}_{\gamma\varphi}\left(#1\right)}
\newcommand\proxpi[1]{\mathrm{prox}_{\gamma\varphi}(#1)}

\definecolor{cuhkpl}{RGB}{152,24,147}

\usepackage{lastpage}
\jmlrheading{}{2026}{1-\pageref{LastPage}}{5/26}{}{}{Kun Huang, Shi Pu, and Karl Henrik Johansson}

\ShortHeadings{Achieving Linear Speedup for Composite FL}{Huang, Pu, and Johansson}
\firstpageno{1}

\begin{document}

\title{Achieving Linear Speedup for Composite Federated Learning}

\author{\name Kun Huang \email kunhuang@kth.se \\
       \addr Department of Decision and Control Systems\\
       School of Electrical Engineering and Computer Science\\
       Digital Futures\\
       KTH Royal Institute of Technology\\
       Stockholm, 10044, Sweden
       \AND
       \name Shi Pu \email pushi@cuhk.edu.cn \\
       \addr School of Data Science\\
       The Chinese University of Hong Kong, Shenzhen\\
       Guangdong, 518172, P.R. China\\
       \AND 
         \name Karl Henrik Johansson \email kallej@kth.se\\
         \addr 
         Department of Decision and Control Systems\\
         School of Electrical Engineering and Computer Science\\
         Digital Futures\\
        KTH Royal Institute of Technology\\
            Stockholm, 10044, Sweden
       }

\editor{}

\maketitle

\begin{abstract}
This paper proposes FedNMap, a normal map-based method for composite federated learning, where the objective consists of a smooth loss and a possibly nonsmooth regularizer. FedNMap leverages a normal map-based update scheme to handle the nonsmooth term and incorporates a local correction strategy to mitigate the impact of data heterogeneity across clients. Under standard assumptions, including smooth local losses, weak convexity of the regularizer, and bounded stochastic gradient variance, FedNMap achieves linear speedup with respect to both the number of clients and the number of local updates for nonconvex losses, both with and without the Polyak-{\L}ojasiewicz condition. To the best of our knowledge, this is the first algorithm establishing linear speedup for nonconvex composite federated learning. Numerical experiments corroborate our theoretical findings and demonstrate the linear speedup of FedNMap.
\end{abstract}


\section{Introduction}

Federated learning (FL) enables a set of clients to collaboratively train a global model, enhancing computational efficiency through parallel local training \citep{suresh2017distributed,li2020federated}. While the theory of FL algorithms for smooth objective functions is well developed, many applications involve a nonsmooth regularization term $\varphi:\R^p \rightarrow(-\infty, \infty]$. Typical examples include statistical learning with sparsity-inducing norms \citep{bao2022fast}, constrained optimization \citep{zhang2024nonconvex}, and model pruning \citep{kubler2025proximal}. 
These applications motivate the following distributed composite optimization problem:
\begin{equation}
    \label{eq:P}
    \min_{x\in \R^p} \psi(x):= f(x) + \varphi(x),\; f(x) := \frac{1}{n}\sumn f_i(x),
\end{equation}
where each client $i$ has access only to its local objective function $f_i:\R^p\rightarrow\R$ and the (possibly nonsmooth) regularizer $\varphi:\R^p\rightarrow (-\infty, \infty]$. 

In this work, we assume $\varphi$ is proper, lower semicontinuous, lower bounded, and $\rho$-weakly convex, and each $f_i$ is $L$-smooth on an open set containing $\crki{x\in\R^p: \varphi(x)<\infty}$. 
We further assume each client $i$ can query an unbiased stochastic gradient $g_i(x;\xi_i)$ ($\xi_i$ is a random variable) of $\nabla f_i(x)$ with bounded variance, and that the proximal operator 
\begin{align*}
  \proxp{x}:= \arg\min_{y\in\R^{p}} \varphi(y) + \frac{1}{2\gamma}\normi{y-x}^2,
\end{align*}
is computationally tractable for $\gamma>0$.

For smooth problems, e.g., $\varphi(x)\equiv0$, a line of work has established the \textit{linear speedup} property \citep{li2019convergence,huang2023distributed,khaled2020better}. Specifically, to reach a sufficiently small target accuracy, using $n$ clients and $Q$ local steps reduces the number of communication rounds by a factor of $\order{nQ}$. 
In contrast, achieving linear speedup for the composite FL problem \eqref{eq:P} remains largely underexplored. 
Although several recent works \citep{zhang2026non,zhang2024composite, bao2022fast,zhou2025fedcanon} propose proximal FL methods for solving \eqref{eq:P}, significant gaps remain. These methods either (i) focus only on convex or strongly convex regimes \citep{yuan2021federated,zhang2024composite,bao2022fast}, (ii) require additional restrictive conditions such as homogeneous objectives \citep{yuan2021federated} or bounded subgradients of $\varphi$ \citep{zhang2024composite,zhou2025fedcanon}, or (iii) fail to establish convergence results for general nonconvex composite objectives \citep{zhang2026non,zhou2025fedcanon}.
Consequently, it remains unclear whether linear speedup can be achieved for general composite FL under standard assumptions. This motivates the central question of our work:
\begin{center}
    \textit{Can we design an FL method that achieves linear speedup for solving the composite FL problem~\eqref{eq:P} under standard assumptions?}
\end{center}

In this paper, we give an affirmative answer to this question. A critical challenge in solving problem~\eqref{eq:P} arises from the nonlinearity of the proximal operator. In particular, the proximal stochastic gradient descent (Prox-SGD) update: $x_{t + 1} = \prox{\eta \varphi}(x_t - \eta g (x_t;\xi_t))$ \citep{ghadimi2016mini,davis2019stochastic} can introduce bias, since $\E\brk{x_{t + 1}|x_t} \neq \prox{\eta \varphi}(x_t - \eta \nabla f(x_t))$ even when the stochastic gradient is unbiased. This inconsistency becomes particularly pronounced for FL methods with $Q>1$ local steps. 
To address this challenge, we leverage a so-called normal map-based update scheme \citep{robinson1992normal,qiu2023normal} that preserves unbiasedness. 
The proposed method, denoted \normfl, integrates the normal map update scheme with a local correction term to mitigate the impact of data heterogeneity across clients. Under standard assumptions, we show that \normfl\ achieves linear speedup for solving problem~\eqref{eq:P} under nonconvex losses, both with and without the Polyak--{\L}ojasiewicz (PL) condition. Moreover, \normfl\ does not require any assumptions on data heterogeneity. 

\subsection{Related Works}

FL for minimizing smooth objective functions has been extensively studied \citep{mcmahan2017communication,li2019convergence,karimireddy2020scaffold,huang2023distributed,li2020fedprox,zhang2021fedpd}. FedAvg \citep{mcmahan2017communication,li2019convergence} pioneered this line of work, but relies on restrictive assumptions of limited data heterogeneity. To mitigate this issue, several improved algorithms have been proposed. For instance, FedProx \citep{li2020fedprox} introduces a proximal term into the local subproblem to stabilize training. FedPD \citep{zhang2021fedpd} employs a primal-dual approach to enhance convergence, and 
SCAFFOLD \citep{karimireddy2020scaffold} introduces two control variates to correct local updates.
More recent work \citep{huang2023distributed} establishes convergence of FedAvg without restrictive data heterogeneity assumptions and under a general variance condition.

Several recent works have studied composite FL.
FedDA \citep{yuan2021federated} employs dual averaging and establishes linear speedup for convex settings when the objective functions are homogeneous or quadratic.
More recently, EcontrolDA \citep{gao2025composite} achieves linear speedup for convex heterogeneous objective functions but does not incorporate multiple local steps.
For strongly convex objective functions, \citet{zhang2024composite} establish convergence without linear speedup but requires bounded subgradients of the nonsmooth term. 
For nonconvex objective functions, it remains unclear whether convergence can be achieved under standard assumptions \citep{zhang2026non,zhou2025fedcanon}. Table~\ref{tab:comp} presents a detailed comparison of related works and their conditions.

Normal map-based methods \citep{robinson1992normal} have been studied for solving composite problems, both in centralized settings \citep{qiu2025new,qiu2023normal,ouyang2025trust} and in distributed settings \citep{huang2024distributed}. Compared with classical proximal stochastic gradient descent (Prox-SGD) \citep{ghadimi2016mini,davis2019stochastic}, the normal map-based update scheme preserves unbiasedness when an unbiased stochastic gradient of the smooth component $f$ is available. Furthermore, the normal map serves as a favorable stationarity measure for composite problems and recovers several classical stationarity measures \citep{qiu2023normal}. 

\begin{table}[t]
\centering
\begin{tabular}{@{}ccccc@{}}
\toprule
Method                              & $f$ or $\psi$  & $\varphi$ & \makecell[c]{Additional \\ Assumptions}   & Convergence Rate                                                                           \\ \midrule
\makecell[c]{FedDA \\ \citep{yuan2021federated}} & $f$ CVX  & CVX       & $\zeta$, Quadratic & $\order{\sqrt{\frac{\sigma^2}{nQT}} + \frac{\zeta^{2/3} + \sigma^{2/3}/Q^{1/3}}{T^{2/3}}}$ \\
\makecell[c]{EControlDA \\ \citep{gao2025composite}}  & $f$ CVX  & CVX       & $Q=1$                   & $\order{\sqrt{\frac{\sigma^2}{nT}} + \frac{\sigma^{2/3}}{T^{2/3}} + \frac{1}{T}}$                 \\
\makecell[c]{Fast-FedDA \\ \citep{bao2022fast}}       & $f$ SCVX & CVX       & $\zeta$, BI                 & $\torder{\frac{\sigma^2}{T} + \frac{Q^2 (\zeta^2 + B^2)}{T^2} }$                                \\
\makecell[c]{Zhang \\ \citep{zhang2024composite}}     & $f$ SCVX & CVX       & $B_\psi$                & $\torder{\frac{\sigma^2 + B_\psi^2}{nT}}$                                                   \\
\makecell[c]{Zhang \\ \citep{zhang2026non}}           & $\psi$ PL & CVX       & $B_\psi$                & Not established                                                             \\
\makecell[c]{FedCanon \\ \citep{zhou2025fedcanon}}     & $\psi$ PL & WCVX    & $B_\psi$                & Not established                                                                \\ \midrule
{\makecell[c]{\textbf{\normfl} \\(This work)}}                           & \textbf{$f$ NCVX} & \textbf{WCVX}    & /                       & $\bs{\order{\sqrt{\frac{\sigma^2}{nQT}} + \frac{1}{T}}}$                                        \\
{\makecell[c]{\textbf{\normfl} \\ (This work)}}                           & \textbf{$\psi$ PL}   & \textbf{WCVX}    & /                       & $\bs{\torder{\frac{\sigma^2}{nQT} + \frac{\sigma^2}{T^2}}}$                                     \\ \bottomrule
\end{tabular}
\caption{Comparison of composite FL methods. Here, CVX, SCVX, NCVX, and WCVX denote convex, strongly convex, nonconvex, and weakly convex objectives, respectively. The heterogeneity measure $\zeta$ is defined by $\normi{\nabla f_i(x) - \nabla f(x)}^2 \leq \zeta^2$. ``Quadratic'' indicates quadratic objective functions. ``BI'' denotes bounded iterates assumption. The term $B_\psi$ indicates that the subgradient of $\psi$ is bounded by $B_\psi$. The notation $\torderi{\cdot}$ hides logarithmic factors.}
\label{tab:comp}
\end{table}

\subsection{Main Contribution}
In this paper, we make the following key contributions to composite federated learning:
\begin{itemize}[leftmargin=1.5em]
    \item \textit{Linear speedup for nonconvex composite FL.} We propose \normfl\ algorithm and show that it achieves linear speedup for nonconvex composite FL (Theorem~\ref{thm:ncvx}). Specifically, \normfl\ attains an $\varepsilon$-solution with communication complexity
    \begin{align*}
        \mathcal{O}&\prt{\frac{(L + \rho)\Delta_\psi\sigma^2}{nQ\varepsilon^4} + \frac{(L + \rho)\Delta_\psi}{\varepsilon^2}+ \frac{{\frac{1}{n}\sumn\normi{\nabla f_i(x_0)}^2 + L^2\normi{x_0-z_0}^2}}{\varepsilon^2}},
    \end{align*}
    where $\Delta_\psi := \psi(x_0) - \psi^*$, $x_0 = \proxpi{z_0}$, $\psi^* := \inf_{x\in\R^p} \psi(x)$, and $\sigma^2$ bounds the stochastic gradient variance.\footnote{Here, an $\varepsilon$-solution means that $\sum_{t=0}^{T-1}\E\brki{\normi{\gamma^{-1}\brk{x_t - \proxpi{x_t - \gamma\nabla f(x_t)}}}^2}/T\leq \varepsilon^2$, where $\gamma$ is the parameter of the proximal operator.} This result is comparable to previous results for smooth objectives with $\varphi(x)\equiv0$, including SCAFFOLD \citep{karimireddy2020scaffold}, and exhibits a $1/(nQ)$ dependence in the dominant term, demonstrating linear speedup with respect to both the number of clients $n$ and the number of local steps $Q$. By comparison, prior works have not established convergence results for the nonconvex composite setting (see Table \ref{tab:comp}). 
    \item \textit{Linear speedup under the PL condition.} 
    When $\psi$ further satisfies the PL condition with modulus $\mu$, \normfl\ attains an $\varepsilon$-solution, i.e., $\E\brk{\psi(x_T) - \psi^*}\leq \varepsilon$, (Theorem~\ref{thm:PL}) with communication complexity
    \begin{align*}
        \torder{ \frac{(L+\rho)\sigma^2}{n Q \varepsilon\mu^2} + \sqrt{\frac{L^2 \sigma^2}{\varepsilon \mu (L + \rho + \mu)^2}}}.
    \end{align*}
    The dominant term again scales as $1/(nQ)$, demonstrating linear speedup. In contrast, previous works assuming strongly convex objectives have not established such a result.
    \item \textit{Mild assumptions.} The convergence guarantees of \normfl\ only rely on the smoothness of $f_i$, weak convexity of $\varphi$, and bounded stochastic gradient variance. In contrast to several prior works, \normfl\ does not require limiting assumptions on data heterogeneity or bounded subgradient assumptions \citep{yuan2021federated,zhang2024composite,bao2022fast}. 
    \item \textit{Efficient communication.}
    Despite handling composite objectives, FedNMap maintains communication load and memory cost comparable to those of state-of-the-art FL methods by using single-variable uplink communication, where each client sends only one variable to the server, thereby reducing the communication cost by $50\%$ compared to SCAFFOLD.
\end{itemize}

\subsection{Notation and Assumptions}

All vectors are column vectors unless otherwise stated. Let $x_{i,t}^\ell \in\R^p$ denote the iterate of client $i$ at the $\ell$-th local update within the $t$-th communication round. 
We use $\normi{\cdot}$ to denote the Frobenius norm for matrices and the $\ell_2$ norm for vectors. The notation $\inproi{a, b}$ stands for the inner product of two vectors $a, b\in\R^{p}$. 
For a possibly nonsmooth function $\psi$, $\partial\psi(x)$ denotes the subdifferential set.

We next introduce the standing assumptions. Assumption~\ref{as:abc} requires each client $i$ to have access to an unbiased stochastic gradient of $f_i$ with bounded variance.

\begin{assumption}
    \label{as:abc}
    Each client $i$ has access to an unbiased stochastic gradient $g_i(x;\xi_i)$ of $\nabla f_i(x)$, i.e., $\condEi{g_i(x;\xi_{i})}{x} = \nabla f_i(x)$, and there exists $\sigma\geq 0$ such that for any $i\in[n]:=\crki{1,2,\ldots, n}$,
    \begin{align*}
        \condE{\norm{g_{i}(x;\xi_{i}) - \nabla f_i(x)}^2}{x}&\leq 
        \sigma^2.
    \end{align*} 
    In addition, the stochastic gradients are independent across clients at each $t\geq 0$ and $\ell = 0,1,\ldots, Q-1$. 
\end{assumption}
Assumption~\ref{as:abc} can be relaxed to the more general ABC condition \citep{khaled2020better,lei2019stochastic,huang2023distributed}. The results in this work can be extended to that setting by following procedures similar to those in \citep{huang2024distributed}.

Assumption \ref{as:smooth} is standard and requires each $f_i$ to be smooth and the objective function $\psi$ to be lower bounded.

\begin{assumption}
    \label{as:smooth}
    Each $f_i:\R^p\rightarrow\R$ is $L$-smooth on $\dom{\varphi}:=\crki{x\in\R^p: \varphi(x)<\infty}$, i.e., $\normi{\nabla f_i(x) - \nabla f_i(x')}\leq L\normi{x - x'}$, $\forall x,x'\in\dom{\varphi}.$
    In addition, 
    $\psi$ is bounded from below, i.e., $\psi(x)\geq \psi^*:= \inf_{x\in\R^p} \psi(x)>-\infty$ for any $x\in\R^p$.
\end{assumption}

Assumption~\ref{as:phi} requires the regularizer $\varphi$ to be weakly convex, which covers many commonly used regularizers such as the indicator function of a closed convex set and the $\ell_1$ norm \citep{davis2019stochastic}.
\begin{assumption}
    \label{as:phi}
    The function $\varphi: \R^{p} \rightarrow (-\infty, \infty]$ is $\rho$-weakly convex, lower semicontinuous, lower bounded, and proper. 
\end{assumption}

Assumption~\ref{as:PL} characterizes a generalized PL condition for the composite problem \eqref{eq:P} \citep{karimi2016linear}, which holds, for example, when $f$ satisfies the PL condition and $\varphi$ is convex. Condition \eqref{eq:PL} is also known as the proximal-PL inequality. Typical examples include Lasso regression \citep{karimi2016linear}.

\begin{assumption}
    \label{as:PL}
    There exists $\mu>0$ such that the function $\psi(x)$ satisfies
    \begin{align}
        \label{eq:PL}
        2\mu\prt{\psi(x)-\psi^*} \leq \brk{\dist\prt{0, \partial \psi(x)}}^2,
    \end{align}
    for all $x \in \R^p$, where $\psi^*= \inf_{x\in\R^p}\psi(x)$ and $\dist\prti{0, \partial\psi(x)} := \min_{v\in\partial\psi(x)}\normi{v}$.
\end{assumption}

\section{\normfl: Federated Learning with Normal Map-Based Updates}
\label{sec:pre_alg}

In this section, we introduce \normfl, a novel algorithm for solving the composite FL problem \eqref{eq:P}. 
The design of \normfl\ consists of two main components: (i) a normal map-based update that addresses the nonsmooth term $\varphi$, and (ii) a correction term that mitigates the drift induced by multiple local steps.

We start by defining the normal map
\begin{equation}
    \label{eq:norm_def}
    \Fnor{z}:= \nabla f(\proxp{z}) + \gamma^{-1}\prt{z - \proxp{z}}.
\end{equation}
Based on the second prox theorem \citep{beck2017first}, it holds that $\Fnor{z}\in\partial\psi(\proxp{z})$. 
By defining the auxiliary variable $x=\proxp{z}$ and letting $g(x;\xi)$ be an unbiased stochastic gradient of $\nabla f(x)$, we have 
\begin{align*}
  \E \crk{g(x;\xi) + \gamma^{-1}\prt{z - x}\middle| x,z} &= \nabla f(x) + \gamma^{-1}\prt{z - x}= \Fnor{z}.
\end{align*}
Therefore, the unbiasedness of the stochastic gradient is preserved, and we can leverage a corrected stochastic normal map update to perform local updates at each client.

\textit{Client update.}
At communication round $t\geq 1$, each client $i$ receives $z_t$ and the aggregated direction $\frac{1}{n}\sum_{j=1}^n y_{j,t-1}$ from the server. The client sets $x_{t} = \proxp{z_t}$ and initializes its local variable as $z_{i,t}^0 = z_t$. The correction term $c_{i,t}$ is updated according to
\begin{equation}
    c_{i, t } =\begin{cases}
        0,& t=0\\
         c_{i,t-1} - y_{i,t-1} + \frac{1}{n} \sum_{j=1}^ny_{j,t-1},& t\geq 1.\\
    \end{cases}\label{eq:cit}
\end{equation}
Client $i$ then performs $Q$ local updates using the corrected normal map for $\ell= 0,1,\ldots, Q-1$.
\begin{subequations}
    \label{eq:local}
    \begin{align}
    x_{i,t}^{\ell} &= \proxp{z_{i,t}^{\ell}} \label{eq:xitell},\\
        z_{i,t}^{\ell + 1} &= z_{i,t}^\ell - \eta_a\brk{g_i(x_{i,t}^\ell;\xi_{i,t}^\ell) + \gamma^{-1}\prt{z_t - x_t} +  c_{i,t}},\label{eq:zitell}
    \end{align}
\end{subequations}
where 
$\eta_a>0$ denotes the local learning rate. 

After completing the $Q$ local updates, client $i$ computes 
\begin{equation}
    y_{i,t} = \frac{1}{\eta_a Q}\prt{z_{t} - z_{i,t}^Q}\label{eq:yit},
\end{equation}
and sends it to the server.

\textit{Server update.}
The server aggregates the received messages $\{y_{i,t}\}_{i=1}^n$ and updates the global model as
\begin{equation}
    z_{t + 1} = z_t - \frac{Q\eta_s\eta_a}{n}\sumn y_{i,t},\; x_{t + 1} = \proxp{z_{t + 1}}. \label{eq:server}
\end{equation}

The complete procedure of \normfl\ is summarized in Algorithm~\ref{alg:normfl}. 

Notably, \normfl\ requires transmitting only a single variable $y_{i,t}$ from each client to the server per communication round, thereby reducing the uplink communication cost by $50\%$ compared to SCAFFOLD \citep{karimireddy2020scaffold}, which requires each client to transmit two variables to the server.

\begin{algorithm}[tb]
	\begin{algorithmic}[1]
		\STATE Initialize $z_0$, determine parameters $Q$, $\eta_a$, and $\eta_s$.
		\FOR{$t=0, 1, 2, \ldots, T-1$}
		\FOR{Client $i = 1, 2, \ldots, n$ in parallel}
        \IF{$t =0$}
        \STATE Receive $z_{t}$ from server. Set $c_{i,t}=0$.\label{line:ci0}
        \ELSE
        \STATE Receive $(z_{t}, \sumn y_{i,t-1}/n)$ from server. 
        \STATE Update $c_{i,t} = c_{i,t-1} - y_{i,t-1} + \sumn y_{i,t-1}/n$.\label{line:ct}
        \ENDIF
        \STATE Set $x_{t} = \proxpi{z_t}$, $z_{i,t}^0 = z_{t}$.
        \FOR{Local update $\ell = 0,1,\ldots,Q-1$}
        \STATE Calculate $x_{i,t}^{\ell} = \proxpi{z_{i,t}^{\ell}}$.
        \STATE Acquire $g_{i,t}^{\ell} = g_i(x_{i, t}^{\ell};\xi_{i, t}^{\ell})$.\label{line:normfl_g}
		\STATE Update $z_{i,t}^{\ell + 1} = z_{i,t}^\ell - \eta_a\brki{g_i(x_{i,t}^\ell;\xi_{i,t}^\ell) + \gamma^{-1}\prti{z_t - x_t} +  c_{i,t}}$. \label{line:normfl_local}
        \ENDFOR
        \STATE Update $y_{i,t} = (z_{t} - z_{i,t}^Q)/(\eta_aQ)$.\label{line:yt}
        \STATE Send $y_{i,t}$ to server.
        \ENDFOR
        \STATE \textbf{Server:} Update $z_{t + 1} = z_t - \frac{Q\eta_s\eta_a}{n}\sumn y_{i,t}$, $x_{t + 1} = \proxpi{z_{t + 1}}$.\label{line:normfl_server}
        \STATE Broadcast $(z_{t+1}, \sumn y_{i,t}/n)$ to all clients.
		\ENDFOR
        \STATE \textbf{Output:} $x_T$.
	\end{algorithmic}
	\caption{\normfl: Federated Learning with Normal Map-Based Update}
	\label{alg:normfl}
\end{algorithm}

\subsection{Comparison with Existing Algorithms}

We demonstrate the distinct mechanism of \normfl\ by comparing it with some existing FL methods. The key feature of \normfl\ is the use of the normal map-based update scheme together with a correction term that compensates for the drift induced by multiple local updates, thereby ensuring that the client update directions $y_{i,t}$ track the global stochastic
normal map. 

To see this, from \eqref{eq:yit} and the update rule of the correction term, we obtain that
\begin{equation}
    \label{eq:yitp1}
    \begin{aligned}
        y_{i,t + 1} 
        &= \frac{1}{Q}\sum_{\ell=0}^{Q-1}\brk{g_{i,t + 1}^\ell + \gamma^{-1}\prt{z_{t + 1} - x_{t + 1}}} + c_{i,t + 1}\\
        &= \frac{1}{Q}\sum_{\ell=0}^{Q-1}\brk{g_{i,t + 1}^\ell + \gamma^{-1}\prt{z_{t + 1} - x_{t + 1}}} + c_{i,t} - y_{i,t} + \frac{1}{n}\sumn y_{i,t}\\
        &= \frac{1}{n}\sumn y_{i,t} + \frac{1}{Q}\sum_{\ell=0}^{Q-1}\brk{g_{i,t + 1}^\ell + \gamma^{-1}\prt{z_{t + 1} - x_{t + 1}}} - \frac{1}{Q}\sum_{\ell=0}^{Q-1}\brk{g_{i,t}^\ell + \gamma^{-1}\prt{z_{t} - x_{t}}},
    \end{aligned}
\end{equation}
where $g_{i,t}^\ell= g_i(x_{i,t}^\ell;\xi_{i,t}^\ell)$.
This recursion implies that each $y_{i,t}$ tracks the global stochastic normal map. Indeed, summing \eqref{eq:yitp1} over $i$ yields for any $t\geq 0$ that 
\begin{equation}
    \label{eq:tracking}
    \frac{1}{n}\sumn y_{i,t} = \frac{1}{Q} \sum_{\ell=0}^{Q-1}\brk{\frac{1}{n}\sumn g_{i,t}^\ell + \gamma^{-1}\prt{z_{t} - x_{t}}}.
\end{equation}

\textit{Comparison with \citep{zhang2024composite}.}
In contrast to \normfl, the method in \citep{zhang2024composite} employs an update direction $y_{i,t}^{\rm (Z)}$ that tracks the global stochastic gradient:
\begin{equation}
        \label{eq:zhang_tracking}
        \begin{aligned}
            y_{i,t + 1}^{\rm (Z)} &= \frac{1}{n}\sumn y_{i,t}^{\rm (Z)} + \frac{1}{Q}\sum_{\ell=0}^{Q-1}g_{i,t+1}^\ell - \frac{1}{Q}\sum_{\ell=0}^{Q-1}g_{i,t}^\ell.
        \end{aligned}
    \end{equation}
Consequently, the server update in \citep{zhang2024composite}, given by 
\begin{equation}
    \label{eq:zhang_normal_map}
    \begin{aligned}
        z_{t + 1} &= z_t - \frac{\teta}{Q} \sum_{\ell=0}^{Q-1} \brk{\frac{1}{n}\sumn g_i(x_{i,t}^\ell;\xi_{i,t}^\ell) + \frac{{z_t - x_t}}{\teta}},\\
        x_{t + 1} &= \proxp{z_{t + 1}},
    \end{aligned}
\end{equation}
essentially performs an approximate stochastic normal map update with $\gamma = \teta := \eta_a\eta_s Q$. 
In contrast, the server update in \normfl\ can be viewed as an approximate stochastic normal map step:
\begin{equation}
    \label{eq:normfl_server}
    \begin{aligned}
        z_{t + 1}&= z_t - \frac{\teta}{Q} \sum_{\ell=0}^{Q-1} \brk{\frac{1}{n}\sumn g_i(x_{i,t}^\ell;\xi_{i,t}^\ell) + \frac{{z_t - x_t}}{\gamma}}.
    \end{aligned}
\end{equation}
Unlike \eqref{eq:zhang_normal_map}, update \eqref{eq:normfl_server} decouples the proximal parameter $\gamma$ from the stepsize, allowing for a more flexible choice of parameters.

\textit{Comparison with SCAFFOLD \citep{karimireddy2020scaffold}.}
The update direction $y_{i,t}^{\rm (S)}$ in SCAFFOLD \citep{karimireddy2020scaffold} follows 
\begin{equation}
    \label{eq:dit_scaffold}
    \begin{aligned}
        y_{i,t + 1}^{\rm (S)} &= \frac{1}{n}\sumn y_{i,t}^{\rm (S)} + \frac{1}{Q}\sum_{\ell=0}^{Q-1}g_{i, t + 1}^{\ell} - \frac{1}{Q}\sum_{\ell=0}^{Q-1}g_{i, t}^\ell,
    \end{aligned}
\end{equation}
which coincides with \normfl\ when $\varphi\equiv0$.
Hence \normfl\ can be viewed as a generalization of SCAFFOLD to composite FL, achieved by incorporating the normal map based update scheme, while simultaneously reducing uplink communication costs.

\section{Convergence Results}

In this section, we establish convergence guarantees for \normfl. 
We analyze the algorithm under two settings: general nonconvex composite objective functions and the PL condition. In both cases, we show that \normfl\ achieves linear speedup with respect to both the number of clients $n$ and the number of local steps $Q$. 

Theorem~\ref{thm:ncvx} establishes the convergence of \normfl\ for general nonconvex composite objective functions. 

\begin{theorem}
    \label{thm:ncvx}
    Let Assumptions \ref{as:abc}--\ref{as:phi} hold. Denote $\Delta_\psi:= \psi(x_0) - \psi^*$. Set 
    \begin{align*}
    \gamma\leq \frac{1}{5(\rho + L)},\; \teta\leq \frac{1-\gamma\rho}{100m\sqrt{L^2 +1/\gamma^2}},\; \eta_a \leq \frac{1-\gamma\rho}{100Q \sqrt{m(L^2 + 1/\gamma^2)}}.
    \end{align*}
    Then, the iterates generated by \normfl\ satisfy
    \begin{equation}
        \label{eq:ncvx}
        \begin{aligned}
            &\frac{1}{T}\sum_{t=0}^{T-1}\E\brk{\norm{\Fnor{z_t}}^2} \leq \frac{27\Delta_\psi}{\teta T} + \frac{27\gamma\cC_0\norm{\Fnor{z_0}}^2}{2\teta T} + \frac{330\sigma^2 }{mnQ } +\frac{110000\eta_a^2 Q L^2 \sigma^2}{(1-\gamma\rho)^2} \\
        &\quad  + \frac{2 L^2}{(1-\gamma\rho)^2T}\norm{\proxp{z_0} - z_0}^2 + \frac{2 }{n(1-\gamma\rho)^2T}\sumn\norm{\nabla f_i(\proxp{z_0})}^2,
        \end{aligned}
    \end{equation}

    In particular, if we set $\gamma = 1/[5(\rho + L)]$, $m = \ceili{\sqrt{\sigma^2 T/[9(\rho + L)\Delta_\psi n Q]}}$,
    \begin{equation}
        \label{eq:ncvx_teta}
        \begin{aligned}
            \teta &= \frac{1}{320 \sqrt{\frac{\sigma^2 T(\rho + L)}{nQ \Delta_\psi}}},\;
            \eta_a = \frac{1}{380\prt{\frac{\sigma^2TQ^3(L+\rho)^3}{n\Delta_\psi}}^{1/4} + 240\sqrt{\frac{ (L+\rho)TQ\sigma^2}{\Delta_\psi}}},
        \end{aligned}
    \end{equation}
    then 
    \begin{equation}
        \label{eq:ncvx_order}
        \begin{aligned}
            &\frac{1}{T}\sum_{t=0}^{T-1} \E\brk{\norm{\Fnor{z_t}}^2} = \order{ \sqrt{\frac{(L + \rho)\Delta_\psi \sigma^2}{nQT}} + \frac{(L + \rho){\Delta_\psi }}{T}\right.\\
            &\left.\quad + \frac{\frac{1}{n}\sumn\normi{\nabla f_i(\proxpi{z_0})}^2 + L^2\normi{\proxpi{z_0}-z_0}^2}{T}}.
        \end{aligned}
    \end{equation}
\end{theorem}

\begin{proof}
    See Appendix~\ref{app:ncvx}.
\end{proof}

\begin{remark}
    \label{rem:ncvx_iter}
    It can be shown (e.g., \citealt{qiu2025new}) that
    \begin{equation}
    \label{eq:nat_nor}
    \begin{aligned}
        (1-\gamma\rho)\norm{\Fnat{\proxp{z}}}&\leq \dist(0, \partial \psi(\proxp{z})) \leq \norm{\Fnor{z}},\; \forall z\in\R^p
    \end{aligned}
    \end{equation}
    where $\Fnat{x}:= \gamma^{-1}\brk{x - \proxpi{x - \gamma \nabla f(x)}}$. Consequently, Theorem~\ref{thm:ncvx} implies that the communication complexity of \normfl\ to achieve an $\varepsilon$-stationary point 
    is given by
        \begin{align*}
          \mathcal{O}&\prt{\frac{(L + \rho)\Delta_\psi\sigma^2}{nQ\varepsilon^4} + \frac{(L + \rho)\Delta_\psi}{\varepsilon^2} + \frac{{\frac{1}{n}\sumn\normi{\nabla f_i(\proxpi{z_0})}^2 + L^2\normi{\proxpi{z_0}-z_0}^2}}{\varepsilon^2}}.
        \end{align*}
    When the desired accuracy is small enough, 
    the communication complexity of \normfl\ simplifies to
    \begin{align*}
        \order{\frac{(L+\rho)\Delta_\psi \sigma^2}{nQ\varepsilon^4}},
    \end{align*}
    which highlights the linear speedup property of \normfl\ for minimizing nonconvex composite objective functions.
\end{remark}

    \begin{remark}
        When $\varphi\equiv 0$, i.e., the smooth case, the communication complexity of \normfl\ reduces to 
    \begin{align*}
        \order{\frac{L\Delta_f\sigma^2}{nQ\varepsilon^4} + \frac{L\Delta_f}{\varepsilon^2} + \frac{\frac{1}{n}\sumn\normi{\nabla f_i(x_0)}^2}{\varepsilon^2}},\; \Delta_f := f(x_0) - f^*,
    \end{align*}
    which matches that of SCAFFOLD \citep{karimireddy2020scaffold} for smooth problems, but with half the communication cost. The above communication complexity is also better than that of FedAvg \citep{karimireddy2020scaffold} for minimizing smooth objective functions.
    \end{remark}


Theorem~\ref{thm:PL} establishes the convergence of \normfl\ when the overall function $\psi$ satisfies the PL condition.

\begin{theorem}
    \label{thm:PL} 
    Let Assumptions \ref{as:abc}--\ref{as:PL} hold. Set $\teta \leq 1/[120m(L + \rho + \mu)]$, $\eta_a\leq 1/[96Q(L + \rho + \mu)]$, and $\gamma\leq 1/[5(\rho + L + \mu)]$. Then, the iterates generated by \normfl\ satisfy
    \begin{equation}
        \label{eq:linear}
        \begin{aligned}
            &\E\brk{\psi(x_T)} - \psi^* 
            \leq \exp\prt{-\frac{2\teta\mu T }{9(1+\gamma\mu\cC_0)}}\E\brk{\cL_0}  +  \frac{36 \sigma^2}{mn\mu Q} + \frac{12100 \eta_a^2 Q L^2 \sigma^2}{ \mu (1-\gamma\rho)^2},
        \end{aligned}
    \end{equation}
    where $\E\brki{\cL_0}$ is a constant. 
    In particular, if we set $\gamma = 1/[5(\mu + L + \rho)]$, $\eta_s>0$, and 
    \begin{equation}
        \label{eq:PL_params}
        \begin{aligned}
            \eta_a &= \frac{\log(nQT)}{120 Q(L + \rho + \mu) T},\; m = \ceil{\frac{\mu T}{L + \rho}},\; \teta = \eta_a \eta_s Q,
        \end{aligned}
    \end{equation}
    then 
    \begin{equation}
        \label{eq:pl_order}
        \begin{aligned}
            \E\brk{\psi(x_T) - \psi^*} &= \torder{\frac{(L + \rho)\sigma^2}{nQT\mu^2} + \frac{L^2\sigma^2}{T^2\mu (L + \rho + \mu)^2}}.
        \end{aligned}
    \end{equation}
\end{theorem}

\begin{proof}
    See Appendix~\ref{app:PL}
\end{proof}

\begin{remark}
    \label{rem:pl}
    Theorem~\ref{thm:PL} indicates that to achieve an $\varepsilon$-solution, i.e., $\E\brki{\psi(x_T) - \psi^*}\leq \varepsilon$, the communication complexity of \normfl\ is
    \begin{align*}
        \torder{ \frac{(L+\rho)\sigma^2}{n Q \varepsilon\mu^2} + \sqrt{\frac{L^2 \sigma^2}{\varepsilon \mu (L + \rho + \mu)^2}}}.
    \end{align*}
    Therefore, 
    when the desired accuracy is small enough, the communication complexity of \normfl\ simplifies to
    \begin{align*}
        \torder{\frac{(L+\rho)\sigma^2}{n Q \mu^2 \varepsilon}},
    \end{align*}
    which highlights the linear speedup property of \normfl\ under the PL condition.
\end{remark}


\section{Convergence Analysis}
\label{sec:convergence_analysis}

This section presents the key analytical ingredients for proving Theorems~\ref{thm:ncvx} and \ref{thm:PL}. The analysis is based on a \textit{multistep Lyapunov approach}: instead of tracking one-round progress, we study the algorithm over intervals of $m\geq 1$ communication rounds. Specifically, we analyze the recursions between the iterates $\crki{(x_{i,t_1},z_{i,t_1})}_{i=1}^n$ and $\crki{(x_{i,t_2},z_{i,t_2})}_{i=1}^n$, where $0\leq t_1 < t_2:= t_1 + m$. This multistep viewpoint allows us to better control the accumulated error by choosing $m$ sufficiently large. Setting $m=1$ recovers the usual one-step analysis. The choice of $m$ affects only the intermediate bounds on the stepsize and does not worsen the final stepsize requirement in Theorems~\ref{thm:ncvx} and \ref{thm:PL}.

An essential component is to introduce the following Lyapunov function:
\begin{equation}
    \label{eq:lya_cL}
    \begin{aligned}
        \cL_t &:= \psi\prt{\proxp{z_t}} + \frac{\gamma\cC_0}{2}\norm{\Fnor{z_t}}^2 + \frac{25\teta L^2}{nQ(1-\gamma\rho)^2}\sum_{\ell=0}^{Q-1}\sumn\norm{z_{i,t}^\ell - z_t}^2,\;\forall t\geq 0,
    \end{aligned}
\end{equation}
and to establish its approximate descent between $t=t_1$ and $t=t_2= t_1 + m$. To achieve this, we proceed in two steps.

The first step, formalized in Lemma~\ref{lem:cH_descent}, is to analyze the following auxiliary Lyapunov function:
\begin{equation}
    \label{eq:cH_can}
    \cH_t:= \psi(\proxp{z_t})  + \frac{\gamma \cC_0}{2}\norm{\Fnor{z_t}}^2,\; \cC_0:= \frac{3-4\gamma\rho}{2\prt{3 - 4\gamma\rho+4\gamma^2 L^2}},
    \end{equation}
which has been commonly used in the analysis of normal map-based methods \citep{qiu2025new,huang2024distributed,ouyang2025trust,qiu2023normal}. The primary technical challenge in our setting lies in controlling the error arising from three sources: consensus error among clients, errors due to multiple local updates, and errors introduced by the multistep analysis. These errors are encapsulated in the following term:
\begin{equation}
    \label{eq:error_challenge}
    \sum_{t=t_1}^{t_2-1}\sum_{\ell=0}^{Q-1}\sumn\E\brki{\normi{\underbrace{z_{i,t}^\ell - \bar{z}_t^\ell}_{\text{Consensus}} + \underbrace{\bar{z}_t^\ell - z_t}_{\text{Local updates}} + \underbrace{z_t - z_{t_1}}_{\text{Multistep analysis}}}^2}.
\end{equation}

The second step is to bound the error term in \eqref{eq:error_challenge} (Lemma~\ref{lem:sum_ub}) and to establish a recursion for the term $\sum_{\ell=0}^{Q-1}\sumn\E\brki{\normi{z_{i,t}^\ell - z_t}^2}$ between $t=t_1$ and $t=t_2$ (Lemma~\ref{lem:zt2ell_zt2}). 

\begin{lemma}
    \label{lem:cH_descent}
    Let Assumptions \ref{as:abc}--\ref{as:phi} hold. Set $\teta\leq \min\crki{(1-\gamma\rho)\gamma/(10m), 1/(10mL)}$ and $\gamma\leq1/[5(\rho + L)]$.
    Then, for any $t_2= t_1 + m$ ($m\geq 1$),
    \begin{equation}
        \label{eq:cH_descent}
        \begin{aligned}
            \E\brk{\cH_{t_2}}&\leq \E\brk{\cH_{t_1}} - \frac{\teta m}{2}\prt{\cC_0 - \frac{32\teta^2 m^2}{\gamma^2(1-\gamma\rho)^2}}\E\brk{\norm{\Fnor{z_{t_1}}}^2} + \frac{5\teta \sigma^2}{nQ} \\
            &\quad + \frac{5\teta L^2}{nQ(1-\gamma\rho)^2}\sum_{t=t_1}^{t_2-1} \sum_{\ell=0}^{Q-1}\sum_{i=1}^n \E\brk{\norm{z_{t_1} - z_{i,t}^\ell}^2}.
        \end{aligned}
    \end{equation}
\end{lemma}

\begin{proof}
    See Appendix~\ref{app:cH_descent}.
\end{proof}


\begin{lemma}
    \label{lem:sum_ub}
    Let Assumptions~\ref{as:abc}--\ref{as:phi} hold. Let $\teta\leq (1-\gamma\rho)/(20m\sqrt{L^2 + 1/\gamma^2})$ and $\eta_a\leq (1-\gamma\rho)/(20Q\sqrt{L^2 + 1/\gamma^2})$.
    We have for any $t_1< t_2 = t_1 + m$ that
    \begin{equation}
    \label{eq:ztell_zt1_ub}
        \begin{aligned}
            &\sum_{t=t_1 + 1}^{t_2}\sum_{\ell=0}^{Q-1}\sumn \E\brk{\norm{z_{i, t}^\ell - z_{t_1}}^2}\leq 11\brk{3\teta^2 m^2  + \eta_a^2 Q^2 }nmQ\E\brk{\norm{\Fnor{z_{t_1}}}^2}\\
        &\quad + 55\brk{\eta_a^2 Q^2m + \frac{2\teta^2 m^2}{n}}n\sigma^2 + \frac{110\teta^2 L^2m^2}{(1-\gamma\rho)^2}\sum_{\ell=0}^{Q-1}\sumn\E\brk{\norm{z_{i,t_1}^\ell - z_{t_1}}^2} \\
        &\quad + \frac{33\eta_a^2 Q^2 L^2 m}{(1-\gamma\rho)^2}\sum_{\ell=0}^{Q-1}\sumn\E\brk{\norm{z_{i, t_1}^\ell - z_{t_1}}^2}.
        \end{aligned}
    \end{equation}

    Moreover, we have 
    \begin{equation}
    \label{eq:z0ell_z0_ub}
    \begin{aligned}
        \sumn \sum_{\ell=0}^{Q-1} \E\brk{\norm{z_{i,0}^\ell - z_0}^2} &\leq 3\eta_a^2 Q^2 n\sigma^2 + \frac{2\eta_a^2Q^3 n}{\gamma^2}\norm{\proxp{z_0} - z_0}^2 \\
        &\quad + 2\eta_a^2 Q^3\sumn\norm{\nabla f_i(\proxp{z_0})}^2.
    \end{aligned}
\end{equation}
\end{lemma}

\begin{proof}
    See Appendix~\ref{app:sum_ub}.
\end{proof}

Lemma~\ref{lem:sum_ub} demonstrates that the error term in \eqref{eq:error_challenge} can be decomposed into three components: (i) a term proportional to $\E\brki{\normi{\Fnor{z_{t_1}}}^2}$, which can be absorbed into the corresponding term in Lemma~\ref{lem:cH_descent} when $\teta$ and $\eta_a$ are sufficiently small; (ii) a noise term proportional to $\sigma^2$; and (iii) a term due to the local updates during round $t_1$. The last one motivates the construction of $\cL_t$ in \eqref{eq:lya_cL}. The following lemma establishes the recursion for the quantity $\sumn\sum_{\ell=0}^{Q-1}\E\brki{\normi{z_{i,t}^\ell - z_t}^2}$ between $t = t_1$ and $t = t_2$.

\begin{lemma}
    \label{lem:zt2ell_zt2}
    Let Assumptions~\ref{as:abc}--\ref{as:phi} hold. Let $\teta\leq (1-\gamma\rho)/[20m\sqrt{L^2 + 1/\gamma^2}]$ and $\eta_a\leq (1-\gamma\rho)/(6QL)$. We have for any $t_1< t_2 = t_1 + m$ that
    \begin{equation}
    \label{eq:zt2ell_zt2_ub}
        \begin{aligned}
            &\sum_{\ell=0}^{Q-1}\sumn\E\brk{\norm{z_{i,t_2}^\ell - z_{t_2}}^2}
            \leq \frac{14\eta_a^2Q^2L^2}{(1-\gamma\rho)^2}\sum_{\ell=0}^{Q-1}\sumn\E\brk{\norm{z_{i,t_1}^\ell - z_{t_1}}^2} +  7\eta_a^2Q^3 n \E\brk{\norm{\Fnor{z_{t_1}}}^2}  \\
            &\quad + \frac{11\eta_a^2 Q^2 L^2}{(1-\gamma\rho)^2}\sum_{t=t_{1} + 1}^{t_2 }\sum_{\ell=0}^{Q-1}\sumn\E\brk{\norm{z_{i,t}^\ell -  z_{t_1}}^2}  + 25\eta_a^2Q^2 n\sigma^2.
        \end{aligned}
\end{equation}
\end{lemma}

\begin{proof}
    See Appendix~\ref{app:zt2ell_zt2}.
\end{proof}

Lemma~\ref{lem:cL_descent} combines the results in Lemmas~\ref{lem:cH_descent}--\ref{lem:zt2ell_zt2} to establish the approximate descent of the Lyapunov function $\cL_t$ between $t=t_1$ and $t=t_2$.

\begin{lemma}
    \label{lem:cL_descent}
    Let Assumptions~\ref{as:abc}--\ref{as:phi} hold. Let 
    \begin{align*}
        \teta\leq \frac{1-\gamma\rho}{70m\sqrt{L^2 + 1/\gamma^2}},\; \eta_a\leq \frac{1-\gamma\rho}{70 Q \sqrt{m(L^2 + 1/\gamma^2)}},\; \gamma\leq \frac{1}{5(\rho + L)}.
    \end{align*}
    We have for any $t_1< t_2 = t_1 + m$ that
    \begin{equation}
    \label{eq:cL_descent}
        \begin{aligned}
            &\E\brk{\cL_{t_2}}\leq \E\brk{\cL_{t_1}}  - \frac{\teta m}{9}\E\brk{\norm{\Fnor{z_{t_1}}}^2}  + \frac{6\teta\sigma^2}{nQ} + \frac{625\teta \eta_a^2 Q L^2\sigma^2}{(1-\gamma\rho)^2} + \frac{1375\teta \eta_a^2 L^2 Q m \sigma^2}{(1-\gamma\rho)^2}.
        \end{aligned}
\end{equation}
\end{lemma}

\begin{proof}
    See Appendix~\ref{app:cL_descent}.
\end{proof}

Because the above analysis holds for any $t_1<t_2 = t_1 + m$, we define a subsequence $\crki{\tk_j}_{j=0}^R$ of the communication rounds $\crki{0,1,2,\ldots, T-1}$, where $\tk_j = jm$ for $j=0,1,\ldots, R$ and the total number of communication rounds satisfies $T = mR + S\ (0\leq S< m)$, as illustrated in Figure~\ref{fig:tkj}. By construction, this subsequence satisfies $\tk_{j + 1} - \tk_j = m$. 

\begin{figure}[tbp]
    \centering
    \resizebox{0.9\columnwidth}{!}{
    \begin{tikzpicture}
        \draw[thick,->] (0,0) -- (14,0) node[right]{$t$};
    
        \node at (0, 0) [below, yshift=-0.2cm] {$0$};
        \node at (1, 0) [below, yshift=-0.2cm] {$1$};
        \node at (2, 0) [below, yshift=-0.2cm] {$2$};
        \node at (4, 0) [below, yshift=-0.2cm] {$m$};
        \node at (5, 0) [below, yshift=-0.12cm] {$m+1$};
        \node at (7, 0) [below, yshift=-0.2cm] {$2m$};
        \node at (10, 0) [below, yshift=-0.2cm] {$Rm$};
        \node at (13, 0) [below, yshift=-0.2cm] {$T-1$};
    
        \node at (0, -1) {$\tk_0$};
        \node at (4, -1) {$\tk_1$};
        \node at (7, -1) {$\tk_2$};
        \node at (10, -1) {$\tk_R$};
        \node at (13, -1) {$\tk_R + S-1$};
        
        \node at (3, 0) [below, yshift=-0.2cm] {$\cdots$};
        \node at (6, 0) [below, yshift=-0.2cm] {$\cdots$};
        \node at (8, 0) [below, yshift=-0.2cm] {$\cdots$};
        \node at (12, 0) [below, yshift=-0.2cm] {$\cdots$};
    
        \draw[thick] (-0.3, -0.2) -- (-0.3, -1.2) -- (0.3, -1.2) -- (0.3, -0.2) -- cycle;
        \draw[thick] (3.7, -0.2) -- (3.7, -1.2) -- (4.3, -1.2) -- (4.3, -0.2) -- cycle;
        \draw[thick] (6.7, -0.2) -- (6.7, -1.2) -- (7.3, -1.2) -- (7.3, -0.2) -- cycle;
        \draw[thick] (9.7, -0.2) -- (9.7, -1.2) -- (10.3, -1.2) -- (10.3, -0.2) -- cycle;
        \draw[thick] (12.1, -0.2) -- (12.1, -1.2) -- (13.8, -1.2) -- (13.8, -0.2) -- cycle;
    \end{tikzpicture}
    }
    \caption{An illustration of the subsequence $\crki{\tk_j}_{j=0}^R$.}
    \label{fig:tkj}
\end{figure}

To complete the proof of Theorem~\ref{thm:ncvx}, it remains to relate $\sum_{t=0}^{T-1}\E\brki{\normi{\Fnor{z_t}}^2}/T$ and $\sum_{j=0}^{R}\E\brki{\normi{\Fnor{z_{\tk_j}}}^2}$. This is achieved by noting that
\begin{align*}
    &\frac{1}{T}\sum_{t=0}^{T-1}\E\brk{\norm{\Fnor{z_t}}^2} \leq \frac{2m}{T}\sum_{j=0}^{R}\E\brk{\norm{\Fnor{z_{\tk_j}}}^2} + \frac{2\Ln^2}{T}\sum_{j=0}^{R}\sum_{s=0}^{m-1}\E\brk{\norm{z_{\tk_j + s} - z_{\tk_j}}^2},
\end{align*}
where the last term corresponds to the quantities bounded by Lemma~\ref{lem:sum_ub}. Combining the preceding lemmas yields Theorem~\ref{thm:ncvx}.

Regarding Theorem~\ref{thm:PL}, we derive the linear convergence result by noting the connection between the PL condition and the normal map:
\begin{equation}
        \label{eq:PL_Fnor_sketch}
        \begin{aligned}
          2\mu\prt{\psi(x_t) - \psi^*}&\leq \brk{\dist\prt{0,\partial\psi\prt{x_t}}}^2 \leq \norm{\Fnor{z_t}}^2,\;\forall t\geq 0.
        \end{aligned}
    \end{equation}

\section{Numerical Experiments}
    
    In this section, we evaluate the practical performance of \normfl\ on composite FL tasks under \textit{heterogeneous data} settings. All methods use full client participation, the same number of communication rounds, and identical model initialization. 
    
    \subsection{A Shallow Neural Network}
    
    We consider a multi-class classification task on the MNIST data set \citep{mnist} using a one-hidden-layer neural network with sigmoid activations. The objective includes an elastic net regularizer \citep{zou2005regularization} defined as \[\varphi(x) = \nu_1 \normi{x}_1 + \nu_2\normi{x}_2^2,\]
    with $\nu_1 = 0.001$ and $\nu_2 = 0.01$. This formulation aligns with the structure of \eqref{eq:P}. 

    To simulate data heterogeneity, we sort training samples by class labels and partition them among clients without overlap. We compare \normfl\ against the method in \citep{zhang2024composite} (denoted Zhang) and FedCanon \citep{zhou2025fedcanon}. All reported results are averaged over $10$ independent trials. 

    \begin{figure}[tbp]
        \centering
        \subfloat[$n=20$, $Q=10$.]{\includegraphics[width=0.33\columnwidth]{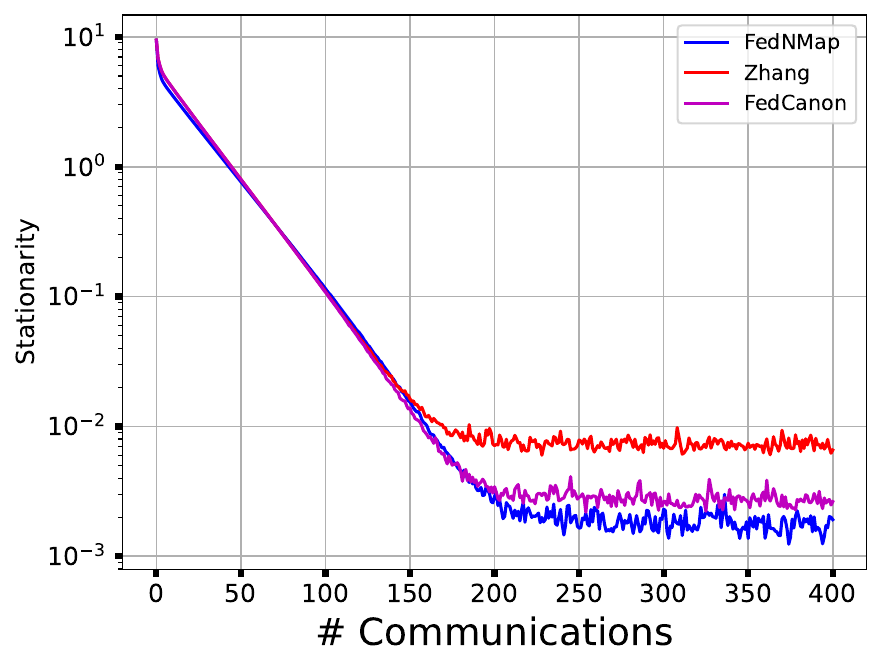}\label{fig:nn_n20Q10}}
        \subfloat[$n=50$, $Q=10$.]{\includegraphics[width=0.33\columnwidth]{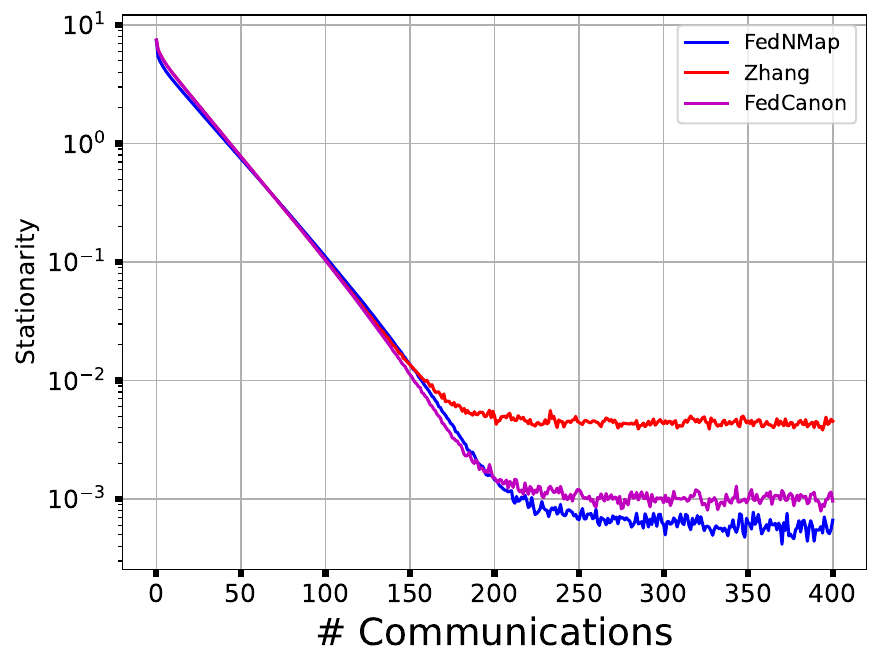}\label{fig:nn_n50Q10}}
        \subfloat[$n=100$, $Q = 10$.]{\includegraphics[width=0.33\columnwidth]{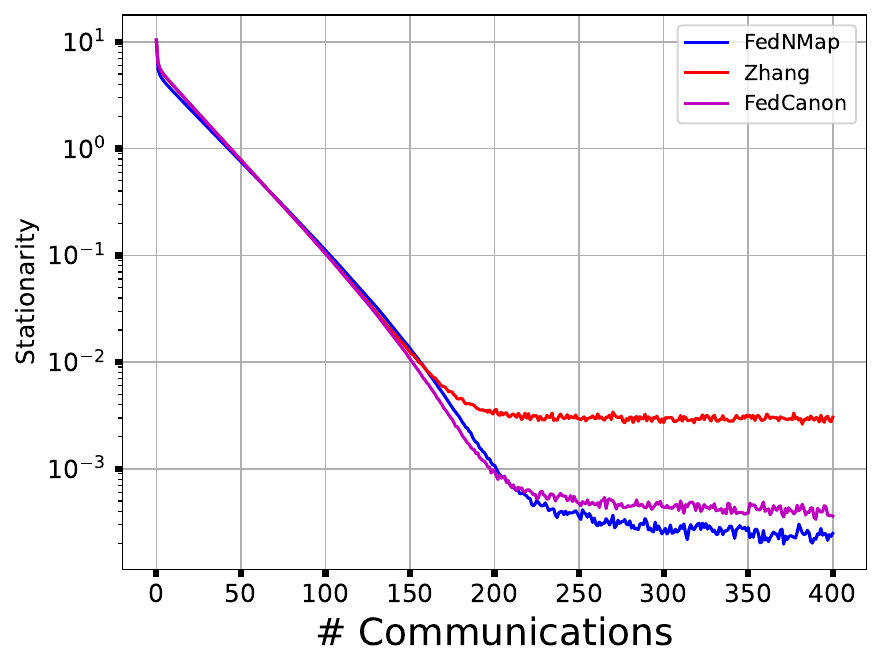}\label{fig:nn_n100Q10}}\\
        \subfloat[$n=20$, $Q=20$.]{\includegraphics[width=0.33\columnwidth]{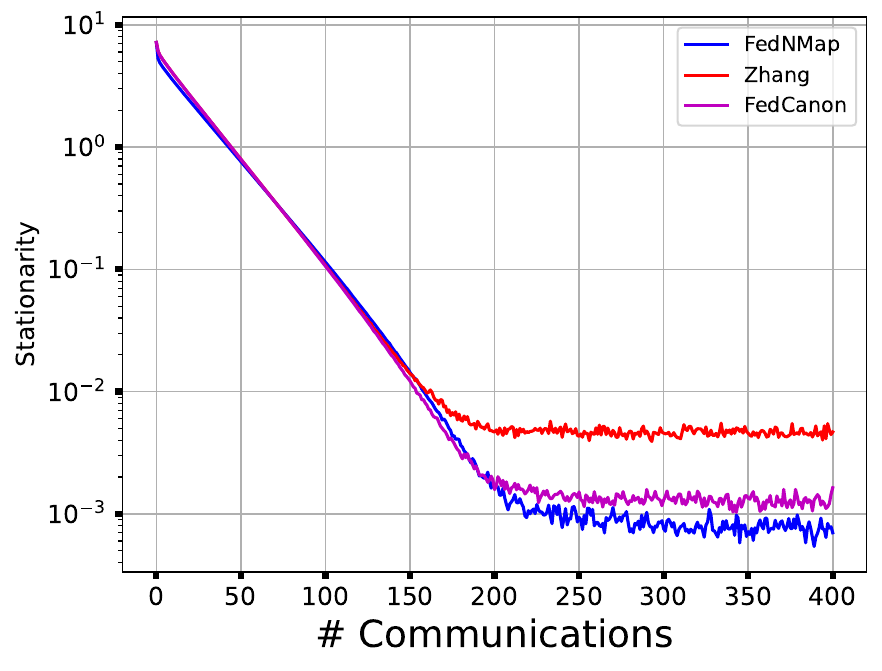}\label{fig:nn_n20Q20}}
        \subfloat[$n=50$, $Q=20$.]{\includegraphics[width=0.33\columnwidth]{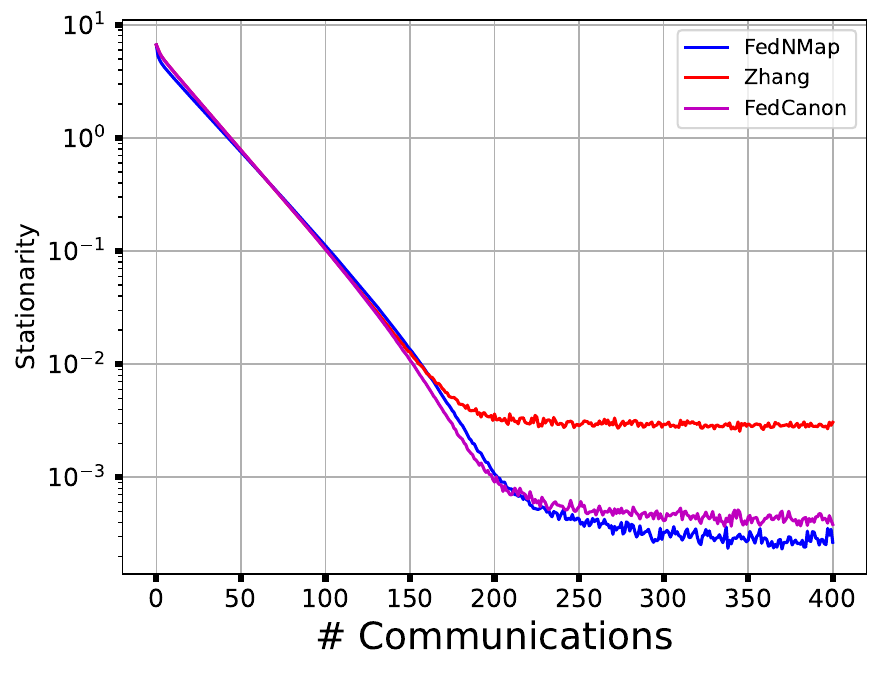}\label{fig:nn_n50Q20}}
        \subfloat[$n=100$, $Q=20$.]{\includegraphics[width=0.33\columnwidth]{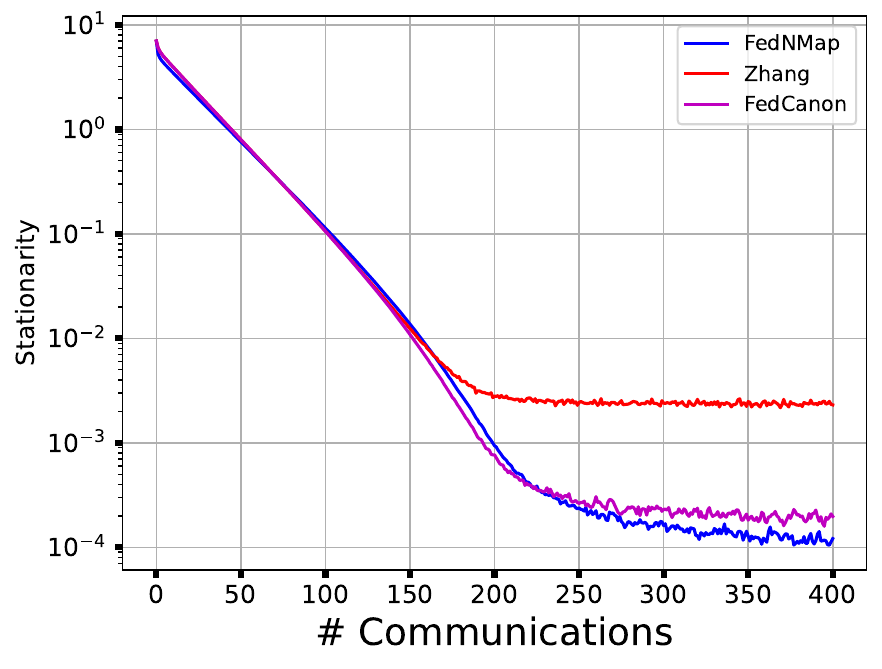}\label{fig:nn_n100Q20}}
        \caption{Comparison of \normfl, Zhang \citep{zhang2024composite}, and FedCanon~\citep{zhou2025fedcanon} for training a one-hidden-layer neural network with an elastic net regularizer on the MNIST data set. 
        The stepsizes are set to $\eta_a = 1/Q$ for local updates and $\eta_s = 1$ for the outer loop across all methods. The parameter $\gamma$ in \normfl\ is set to $4$.}
        \label{fig:nn}
    \end{figure}

    Figure~\ref{fig:nn} reports the stationarity measure $\normi{\Fnat{x_t}}^2$ against the number of communication rounds for varying numbers of clients $n$ and local updates $Q$. The stepsize rule $\eta_a = 1/Q$ ensures that the effective stepsize $\teta = \eta_a\eta_s Q$ remains unchanged for different values of $Q$. This makes the convergence speeds comparable, so lower stationarity values indicate better convergence. As shown in Figure~\ref{fig:nn}, \normfl\ consistently achieves lower stationarity values than the other two methods across all tested configurations. Moreover, increasing the number of local updates from $Q=10$ to $Q=20$ (comparing Figures~\ref{fig:nn_n20Q10} and \ref{fig:nn_n20Q20}, Figures~\ref{fig:nn_n50Q10} and \ref{fig:nn_n50Q20}, or Figures~\ref{fig:nn_n100Q10} and \ref{fig:nn_n100Q20}) improves convergence for a fixed $n$. Similarly, increasing the number of clients from $n=20$ to $n=100$ (comparing Figures~\ref{fig:nn_n20Q10}--\ref{fig:nn_n100Q10} or Figures~\ref{fig:nn_n20Q20}--\ref{fig:nn_n100Q20}) improves convergence for a fixed $Q$. These empirical observations are consistent with the dependence on $n$ and $Q$ predicted by Theorem~\ref{thm:ncvx} and further demonstrate the linear speedup for \normfl.

    To further examine the linear speedup of \normfl\ predicted by Theorem~\ref{thm:ncvx}, Figures~\ref{fig:speedup_n} and \ref{fig:speedup_Q} isolate the effects of the number of clients $n$ and the number of local updates $Q$, respectively.

    Figure~\ref{fig:speedup_n_error} plots the stationarity measure against the number of communication rounds for different values of $n$ with $Q = 10$ fixed and a common effective stepsize $\teta = \eta_a\eta_s Q$. The results show that larger $n$ leads to a smaller stationarity measure throughout the $T$ communication rounds, which is consistent with the $\orderi{1/\sqrt{nQT}}$ convergence rate established in Theorem~\ref{thm:ncvx}. Figure~\ref{fig:speedup_n_loglog} reports the final stationarity measure against $\log n$. The steeper empirical slope ($-1.436$) compared to the theoretical prediction $(-0.5)$ suggests that the practical speedup with respect to $n$ can be better in some problems.

    \begin{figure}[tbp]
        \centering
        \subfloat[Stationarity with different $n$.]{\includegraphics[width=0.35\columnwidth]{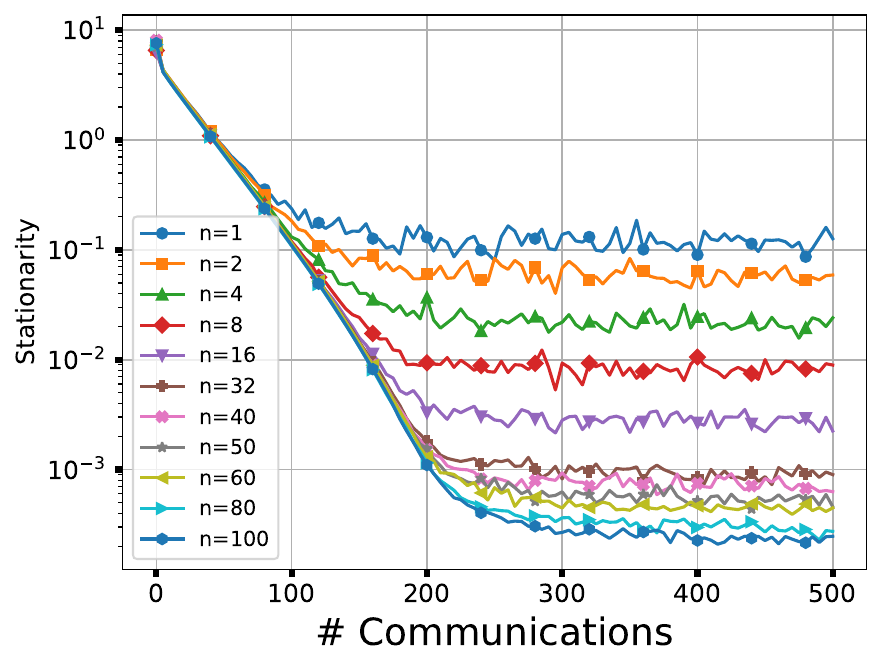}\label{fig:speedup_n_error}}
        \subfloat[Final stationarity against $\log n$.]{\includegraphics[width=0.35\columnwidth]{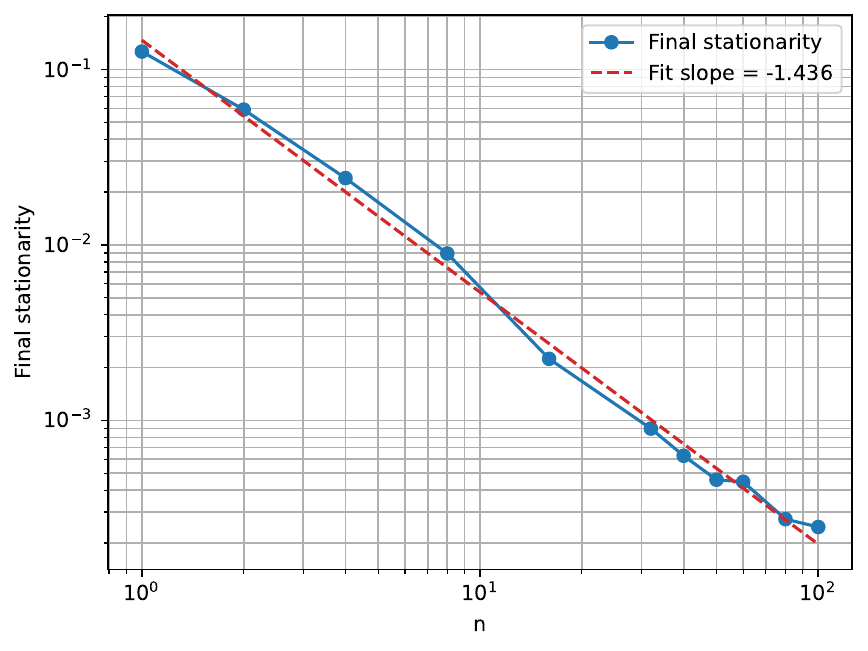}\label{fig:speedup_n_loglog}}
        \caption{Empirical evaluation of the linear speedup of \normfl\ with respect to the number of clients $n$ for fixed $Q = 10$. The stepsizes are set to $\eta_a = 0.1$ for local updates and $\eta_s = 1$ for the outer loop across all $n$. The parameter $\gamma$ in \normfl\ is set to $4$.}
        \label{fig:speedup_n}
    \end{figure}

    Figure~\ref{fig:speedup_Q_error} plots the stationarity measure against the number of communication rounds for different values of $Q$ with $n = 30$ fixed and a common effective stepsize $\teta = \eta_a\eta_s Q$. The results show that larger $Q$ leads to a smaller stationarity measure, which is consistent with the $\orderi{1/\sqrt{nQT}}$ convergence rate established in Theorem~\ref{thm:ncvx}. Figure~\ref{fig:speedup_Q_loglog} reports the final stationarity measure as a function of $\log Q$. The steeper empirical slope ($-1.181$) compared to the theoretical prediction $(-0.5)$ suggests that the practical speedup with respect to $Q$ can be better in some problems.
    \begin{figure}[t]
        \centering
        \subfloat[Stationarity with different $Q$.]{\includegraphics[width=0.35\columnwidth]{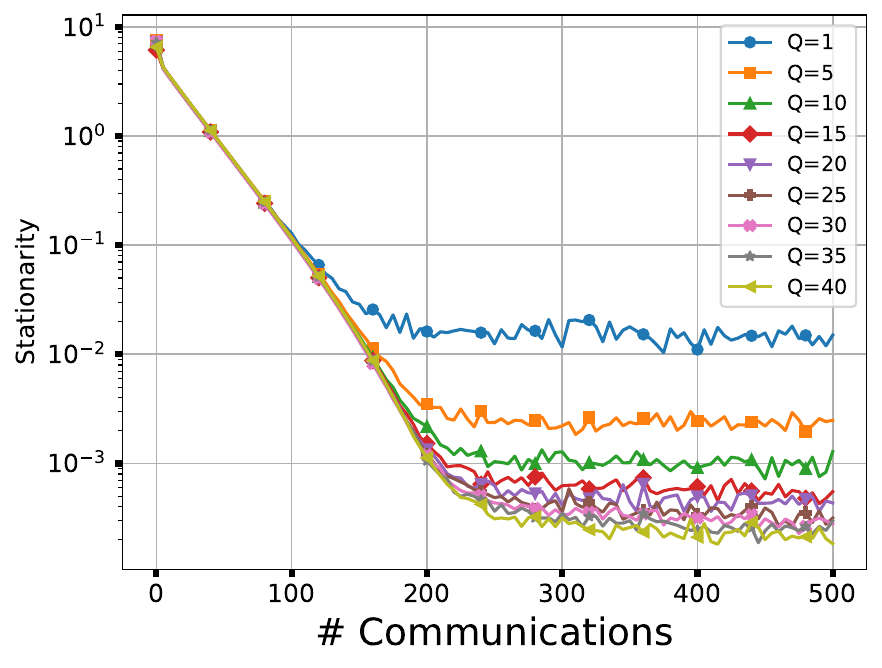}\label{fig:speedup_Q_error}}
        \subfloat[Final stationarity against $\log Q$.]{\includegraphics[width=0.35\columnwidth]{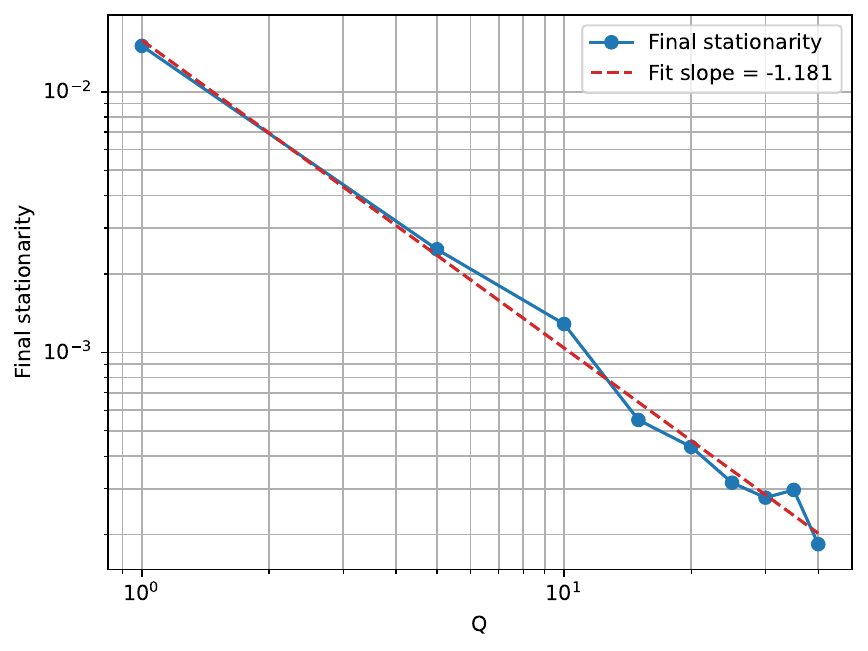}\label{fig:speedup_Q_loglog}}
        \caption{Empirical evaluation of the linear speedup property of \normfl\ with respect to the number of local updates $Q$ with a fixed $n = 30$. The stepsizes are set to $\eta_a = 1/Q$ for local updates and $\eta_s = 1$ for the outer loop across $Q$. The parameter $\gamma$ in \normfl\ is set to $4$.}
        \label{fig:speedup_Q}
    \end{figure}

    We also examine how the choice of $\eta_a$ affects convergence behavior as $Q$ varies. In Figure~\ref{fig:etaa}, we fix $n = 30$ and $\eta_s = 1$, vary $Q$, and consider three local stepsize rules: $\eta_a = 1/Q$, $\eta_a = 1/\sqrt{Q}$, and $\eta_a = 0.05$. In Figure~\ref{fig:etaa_Q} ($\eta_a = 1/Q$), all curves share the same effective stepsize $\teta = \eta_a \eta_s Q$, and larger $Q$ yields a smaller stationarity measure. In Figures~\ref{fig:etaa_sqrtQ} ($\eta_a = 1/\sqrt{Q}$) and \ref{fig:etaa_005} ($\eta_a = 0.05$), the effective stepsize $\teta$ increases with $Q$, which leads to faster convergence for larger $Q$ but a slightly higher yet comparable final stationarity measure. All the results demonstrate the linear speedup property of \normfl.

    \begin{figure}[tbp]
        \centering
        \subfloat[Stationarity with $\eta_a = 1/Q$.]{\includegraphics[width=0.33\columnwidth]{figs/speedup/normfl_all_q_n30nt10_False_prox.pdf}\label{fig:etaa_Q}}
        \subfloat[Stationarity with $\eta_a = 1/\sqrt{Q}$.]{\includegraphics[width=0.33\columnwidth]{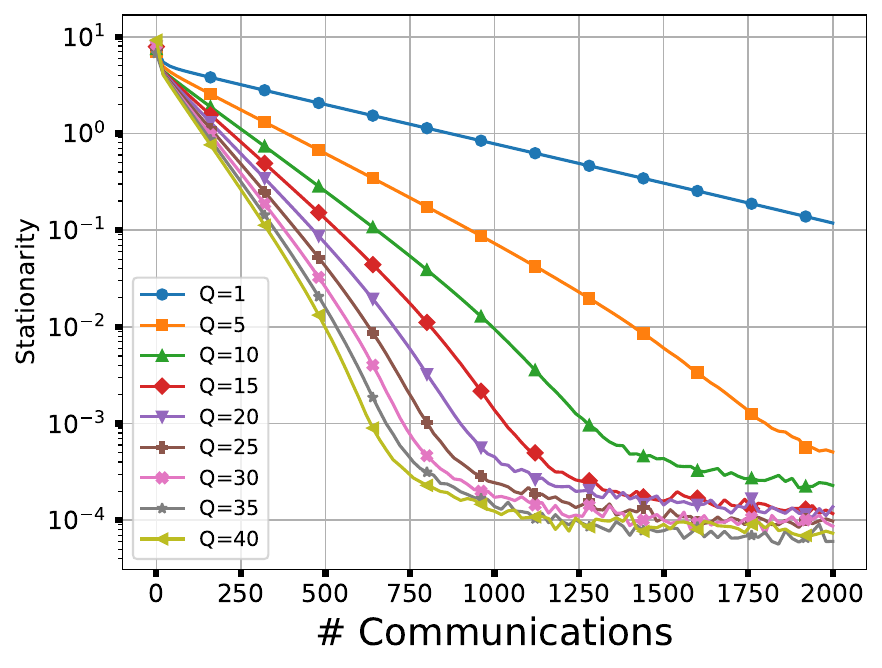}\label{fig:etaa_sqrtQ}}
        \subfloat[Stationarity with $\eta_a =0.05$.]{\includegraphics[width=0.33\columnwidth]{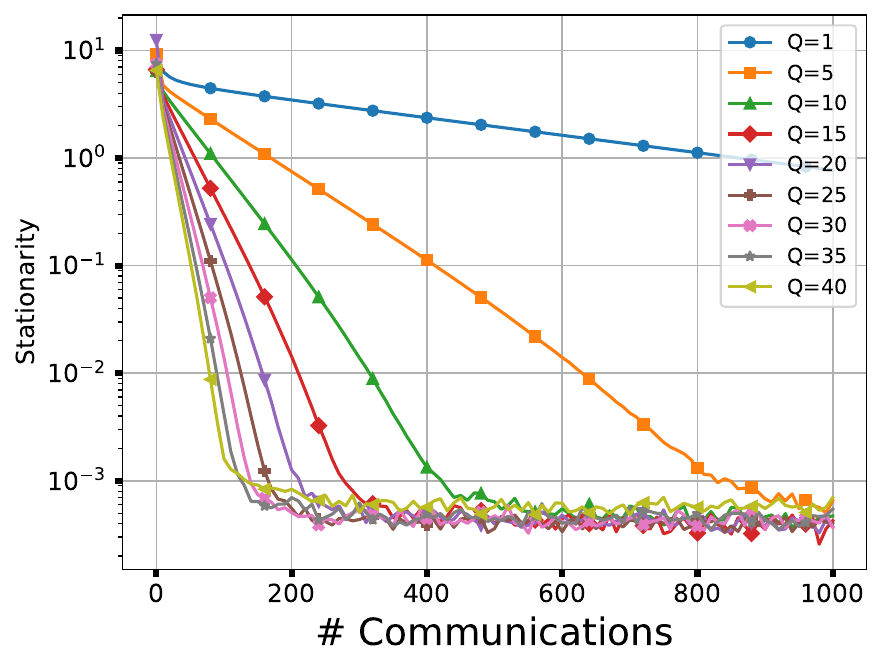}\label{fig:etaa_005}}
        \caption{Empirical evaluation of \normfl\ under different choices of the local stepsize $\eta_a$. We fix $n=30$, $\eta_s = 1$, and $\gamma = 4$.}
        \label{fig:etaa}
    \end{figure}

    \subsection{MobileNetV3-Small} 
    To assess the broader applicability of \normfl, we conduct experiments on MobileNetV3-Small \citep{howard2019searching}, a modern lightweight architecture with $1.5$M parameters designed for mobile devices. Since the model is intended for resource-constrained and low-latency scenarios, inducing sparsity is especially important. Accordingly, we add an $\ell_1$-norm regularizer to encourage sparse parameterization and use Hoyer's sparsity $s:= (\sqrt{p} - \normi{x}_1/\normi{x}_2)/(\sqrt{p}-1)$ for $x\in\R^p$ to quantify the resulting sparsity level.
    In Figure~\ref{fig:mobilenet}, \normfl\ consistently achieves higher sparsity and lower training loss than Zhang and FedCanon.
\begin{figure}[tb]
        \centering
        \subfloat[Hoyer's Sparsity, $n=20$, $Q=10$]{\includegraphics[width=0.33\columnwidth]{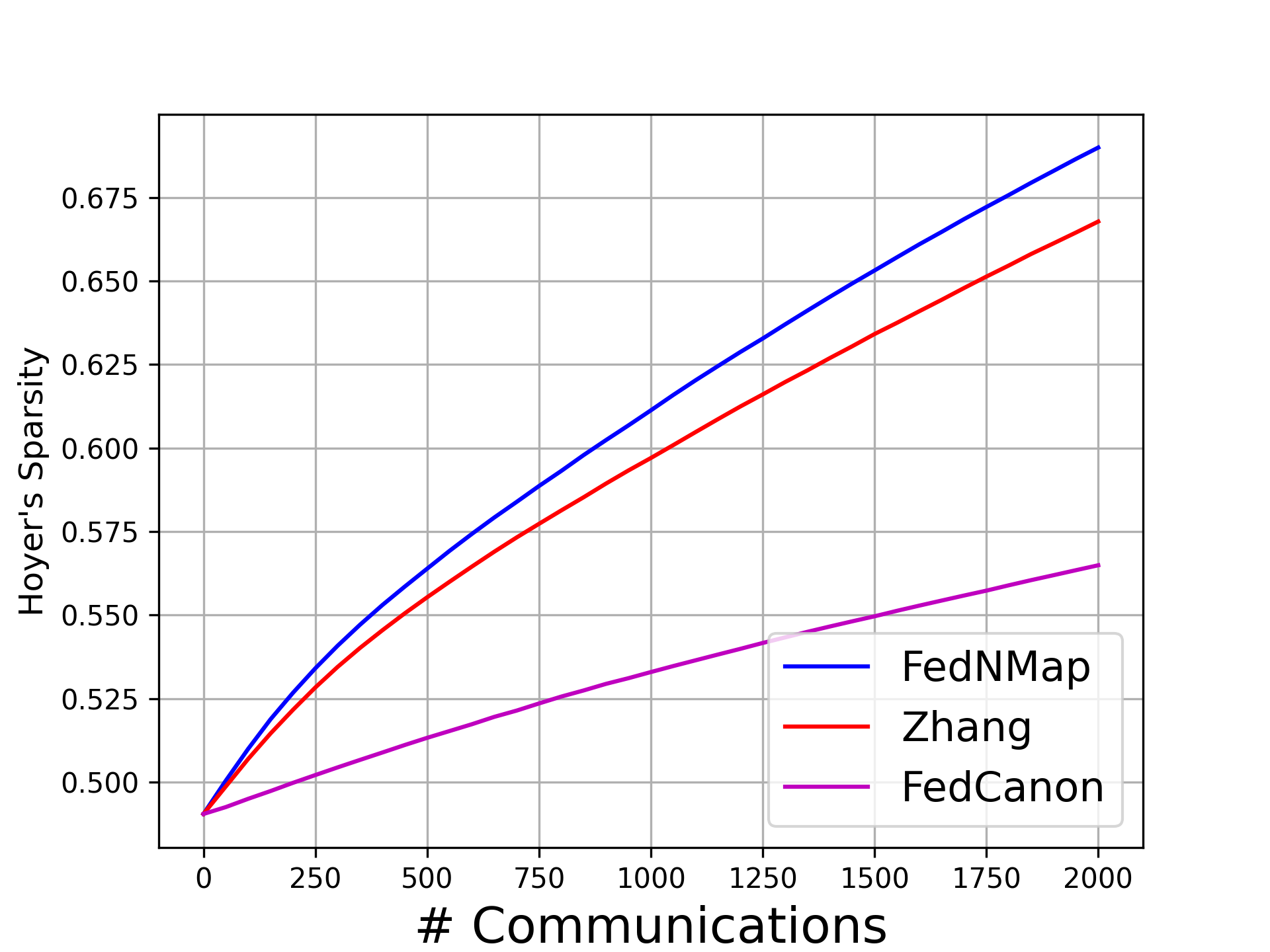}\label{fig:mb_hoyer_n20Q10}}
        \subfloat[Hoyer's Sparsity, $n=20$, $Q=20$]{\includegraphics[width=0.33\columnwidth]{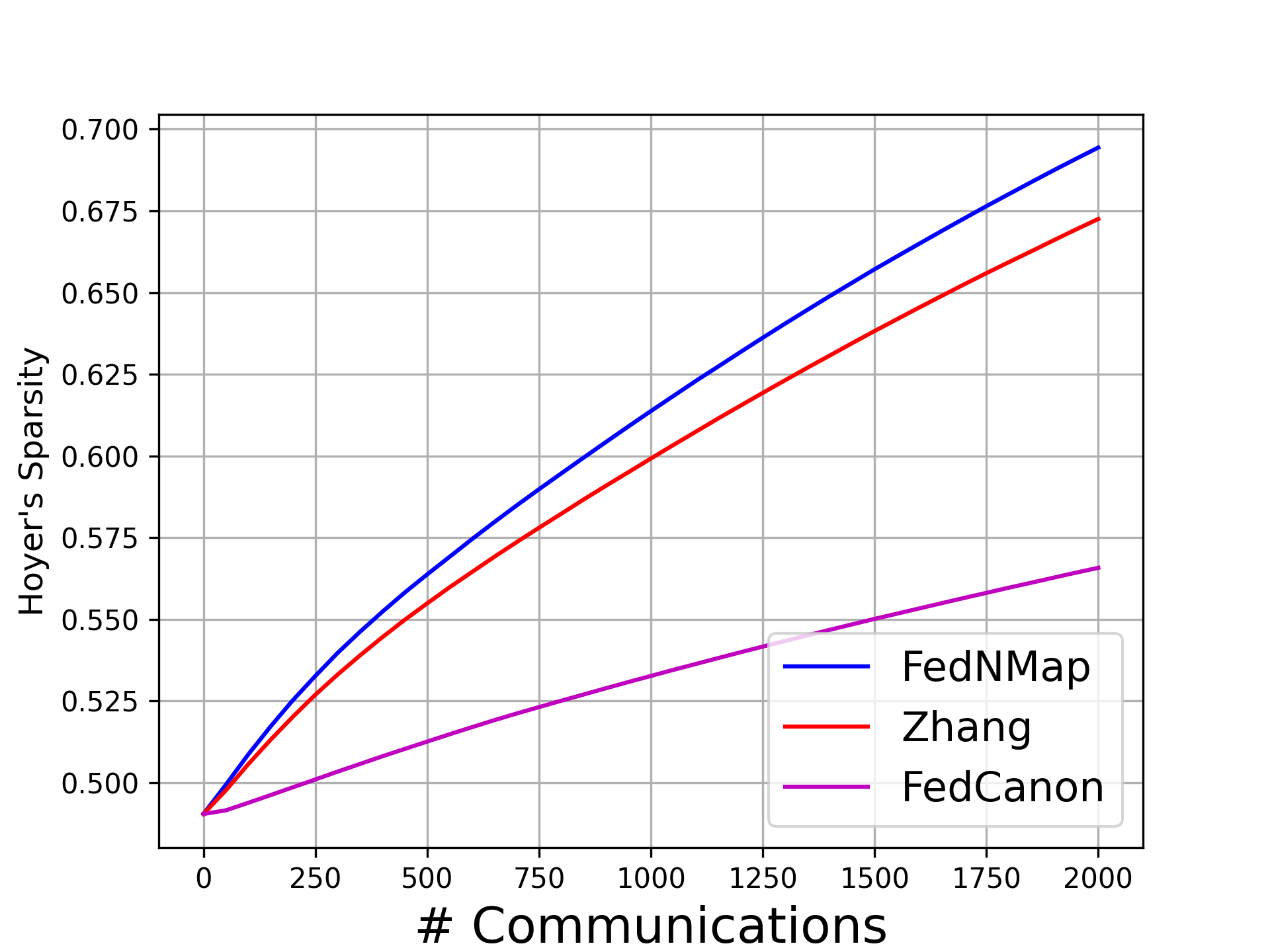}\label{fig:mb_hoyer_n20Q20}}
        \subfloat[Hoyer's Sparsity, $n=50$, $Q=10$]{\includegraphics[width=0.33\columnwidth]{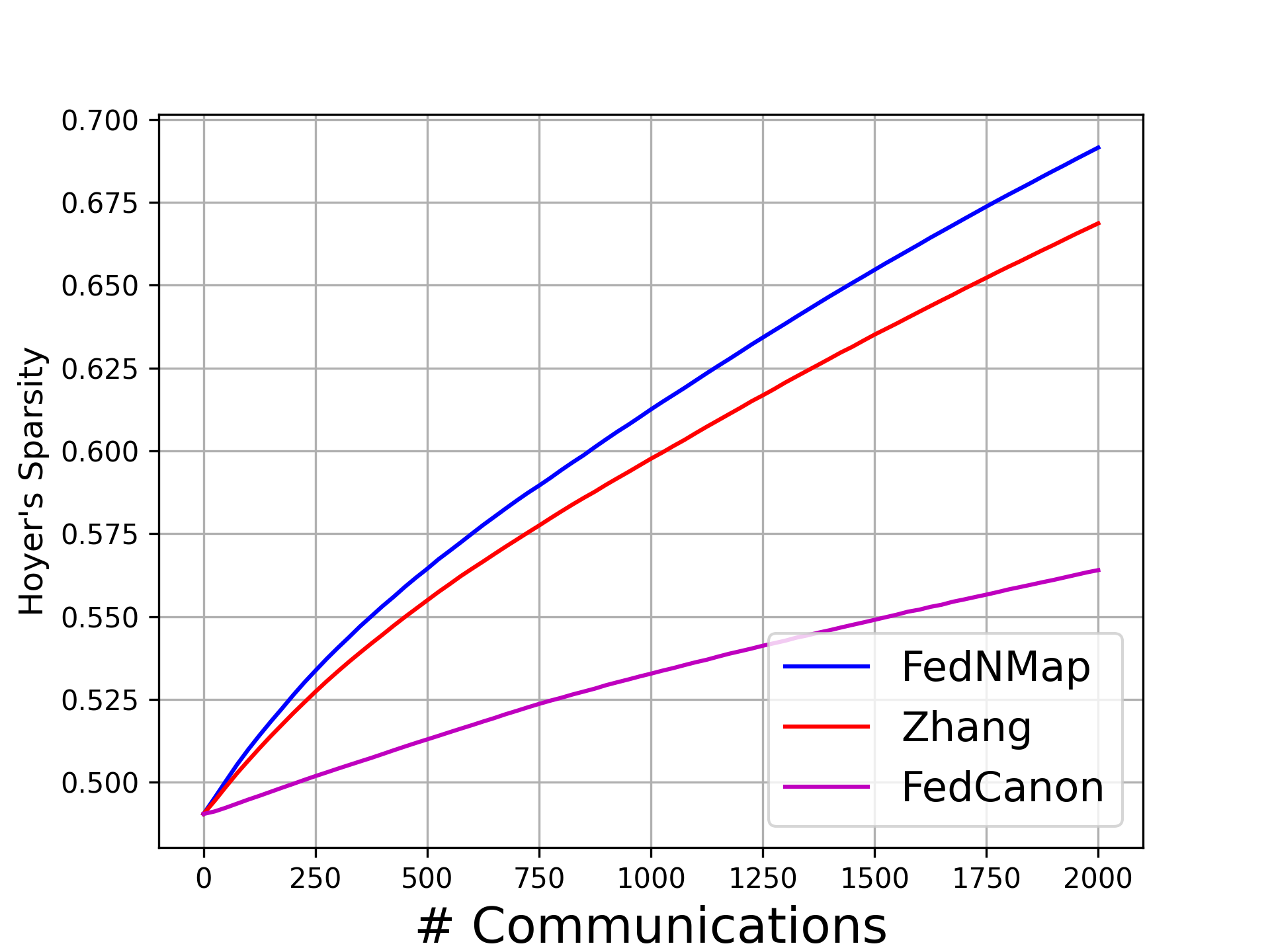}\label{fig:mb_hoyer_n50Q10}}\\
        \subfloat[Training Loss, $n=20$, $Q=10$]{\includegraphics[width=0.33\columnwidth]{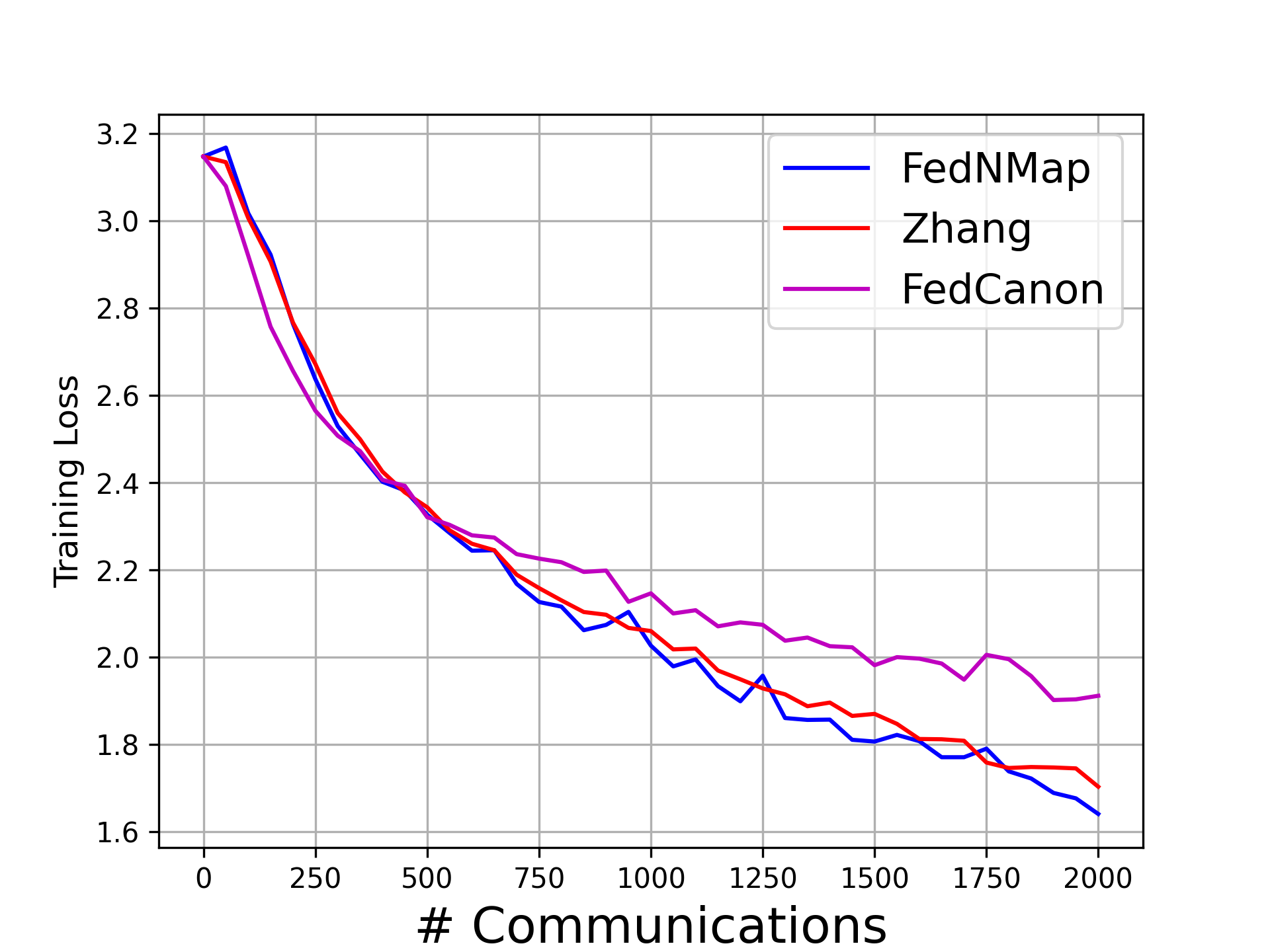}\label{fig:mb_loss_n20Q10}}
        \subfloat[Training Loss, $n=20$, $Q=20$]{\includegraphics[width=0.33\columnwidth]{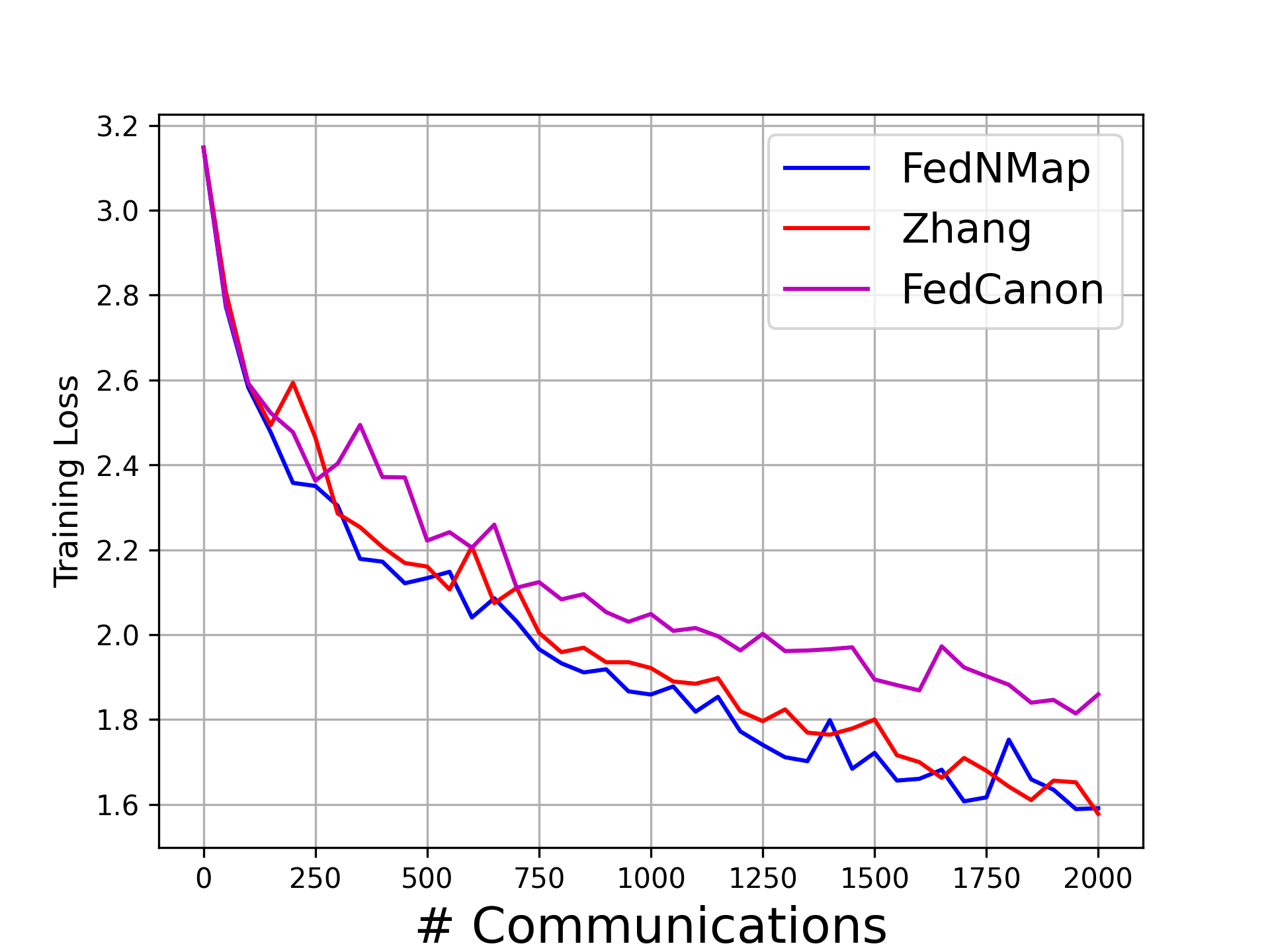}\label{fig:mb_loss_n50Q10}}
        \subfloat[Training Loss, $n=50$, $Q=10$]{\includegraphics[width=0.33\columnwidth]{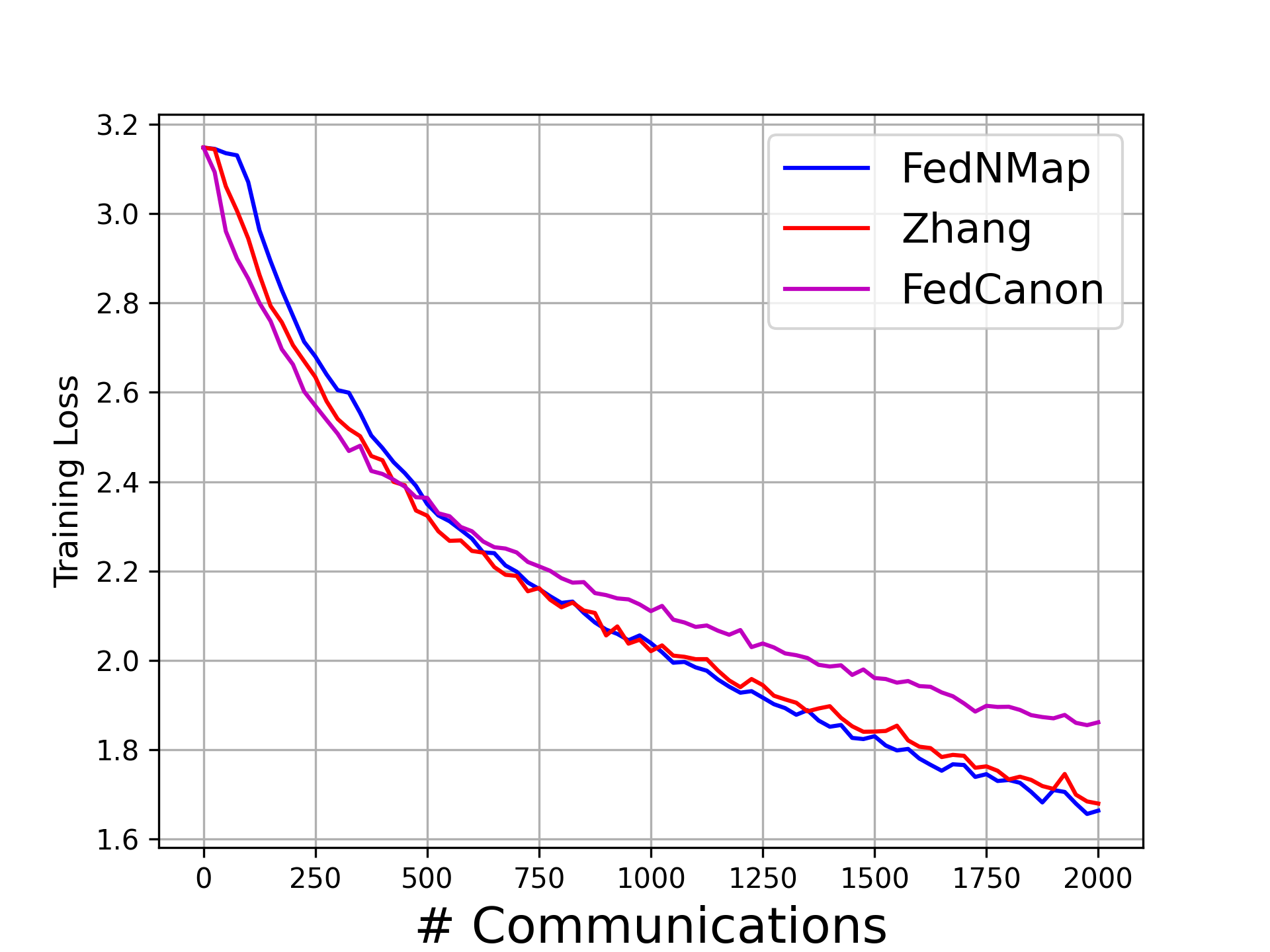}\label{fig:mb_loss_n50Q20}}
        \caption{Comparison of \normfl, Zhang, and FedCanon for training MobileNetV3-Small with an $\ell_1$-norm regularizer on the CIFAR-10. The stepsizes are set to $\eta_a = 1 / Q$ and $\eta_s=1$. The number of local updates is $Q\in\{10,20\}$, the number of clients is $n\in\{20,50\}$, and the parameter $\gamma$ in \normfl\ is set to $2.5$.}
        \label{fig:mobilenet}
    \end{figure}

\section{Conclusion}
    This paper introduced \normfl, a novel algorithm for solving composite optimization problems in federated learning. By integrating a normal map-based update with a local correction mechanism, \normfl\ effectively addresses the challenges posed by a nonsmooth regularized term and data heterogeneity. We showed that linear speedup is achievable for composite FL with respect to both the number of clients and the number of local updates, covering both general nonconvex objectives and those that fulfill the PL condition. 
   Numerical experiments validated our theoretical findings and demonstrated the effectiveness of the proposed algorithm.




\appendix
\section{Derivations for Related Methods}
\label{app:tracking}

We present the derivations for comparison with existing methods below.

\subsection{Method of Zhang et al.}
\label{app:zhang}
    We restate the update rules in \citep{zhang2024composite} and rewrite them in a form directly comparable to \normfl.
    The method in \citep{zhang2024composite} performs local updates as
    \begin{equation}
        \label{eq:local_zhang}
        \begin{aligned}
            z_{i,t}^{\ell + 1} &= z_{i,t}^\ell - \eta_a\prt{g_i(x_{i,t}^\ell;\xi_{i,t}^\ell) + c_{i,t}},\; x_{i,t}^{\ell + 1} = \prox{(\ell + 1)\eta_a\varphi}(z_{i,t}^{\ell + 1}),
        \end{aligned}
    \end{equation}
    where $z_{i,t}^0 = \prox{\teta}(z_t) = x_t$ and $x_{i,t}^0 = x_t$ are received from the server. The server update is
    \begin{equation}
        \label{eq:server_zhang}
        \begin{aligned}
            z_{t + 1} &= x_t + \eta_s\prt{\frac{1}{n}\sumn z_{i,t}^Q - x_t},\; x_{t + 1} = \prox{\teta\varphi}(z_{t + 1}).
        \end{aligned}
    \end{equation}

    The correction term $c_{i,t}$ is updated as
    \begin{equation}
        \label{eq:cit_zhang}
        \begin{aligned}
            c_{i, t + 1} &= \frac{1}{\teta}\prt{x_t - z_{t + 1}} - \frac{1}{Q}\sum_{\ell=0}^{Q-1}g_i(x_{i,t}^\ell;\xi_{i,t}^\ell)\\
            &= \frac{1}{nQ}\sum_{\ell=0}^{Q-1}\sumn g_i(x_{i,t}^\ell;\xi_{i,t}^\ell) + \frac{1}{n}\sumn c_{i,t} - \frac{1}{Q}\sum_{\ell=0}^{Q-1}g_i(x_{i,t}^\ell;\xi_{i,t}^\ell).
        \end{aligned}
    \end{equation}

    By introducing $y_{i,t}^{\rm (Z)} = \prti{x_{t} - z_{i,t}^Q}/(\eta_a Q) = \sum_{\ell=0}^{Q-1}g_i(x_{i,t}^\ell;\xi_{i,t}^\ell)/Q + c_{i,t}$, we can rewrite the correction update and the server update in terms of $y_{i,t}^{\rm (Z)}$ below. The correction update becomes
    \begin{equation}
        \label{eq:cit_yit_zhang}
        \begin{aligned}
            c_{i,t + 1} &= \frac{1}{n}\sumn y_{i,t}^{\rm (Z)} - y_{i,t}^{\rm (Z)} +c_{i,t},
        \end{aligned}
    \end{equation}
    and the server updates as
    \begin{equation}
        \label{eq:server_yit_zhang}
        \begin{aligned}
            z_{t + 1} &= x_t - \frac{Q\eta_s\eta_a}{n}\sumn y_{i,t}^{\rm (Z)},\; x_{t + 1} = \prox{\teta}(z_{t + 1}).
        \end{aligned}
    \end{equation}

    A derivation analogous to \eqref{eq:yitp1} demonstrates that $y_{i,t}^{\rm (Z)}$ tracks the stochastic gradient as in \eqref{eq:zhang_tracking}.

\subsection{SCAFFOLD}
\label{app:scaffold}
    For completeness, we also restate the update rules of SCAFFOLD \citep{karimireddy2020scaffold} and express them in a form directly comparable to \normfl.

    The local updates of SCAFFOLD are given by
    \begin{equation}
        \label{eq:local_scaffold}
        \begin{aligned}
            x_{i,t}^{\ell + 1} &= x_{i,t}^\ell - \eta_a\brk{g_i(x_{i,t}^\ell;\xi_{i,t}^\ell) - c_{i,t} + c_t},\; x_{i,t}^0 = x_t,\\
            c_{i,t+1} &= c_{i,t} - c_t + \frac{1}{\eta_a Q}\prt{x_t - x_{i,t}^Q} = \frac{1}{Q}\sum_{\ell=0}^{Q-1}g_i(x_{i,t}^\ell;\xi_{i,t}^\ell).
        \end{aligned}
    \end{equation}

    The server updates as
    \begin{equation}
        \label{eq:server_scaffold}
        \begin{aligned}
            x_{t + 1} = x_t + \frac{\eta_s}{n}\sumn\prt{x_{i,t}^Q - x_t},\;
            c_{t + 1} = c_t + \frac{1}{n}\sumn\prt{c_{i,t + 1} - c_{i,t}}.
        \end{aligned}
    \end{equation}

    By introducing $y_{i,t}^{\rm (S)} = (x_t - x_{i,t}^Q)/(\eta_a Q)$, we can rewrite \eqref{eq:server_scaffold} as
    \begin{equation}
        \label{eq:server_yit_scaffold}
        \begin{aligned}
            x_{t + 1} = x_t - \frac{Q\eta_s\eta_a}{n}\sumn y_{i,t}^{\rm (S)},\;
            c_{t + 1} = \frac{1}{n}\sumn y_{i,t}^{\rm (S)}.
        \end{aligned}
    \end{equation}

    Thus, $y_{i,t}^{\rm (S)}$ admits the tracking recursion in \eqref{eq:yit_scaffold}:
    \begin{equation}
        \label{eq:yit_scaffold}
        \begin{aligned}
            y_{i,t + 1}^{\rm (S)} &= \frac{1}{Q}\sum_{\ell=0}^{Q-1}g_{i}(x_{i,t+1}^\ell;\xi_{i,t+1}^\ell) - c_{i,t + 1} + c_{t + 1}\\
            &= \frac{1}{n}\sumn y_{i,t}^{\rm (S)} + \frac{1}{Q}\sum_{\ell=0}^{Q-1}g_{i}(x_{i,t+1}^\ell;\xi_{i,t+1}^\ell) - \frac{1}{Q}\sum_{\ell=0}^{Q-1}g_{i}(x_{i,t}^\ell;\xi_{i,t}^\ell).
        \end{aligned}
    \end{equation}

\section{Proofs of Preliminary Results}
\label{app:technical}


We begin with Lemma~\ref{lem:prox_wcvx}, which summarizes a nonexpansiveness property of the proximal mapping $\proxp{\cdot}$ under Assumption~\ref{as:phi}. A proof can be found, for example, in \citep{davis2019stochastic}.
   \begin{lemma}
        \label{lem:prox_wcvx}
        Let Assumption \ref{as:phi} hold. Set $\gamma\in (0, \rho^{-1})$. We have for all $w,v\in\R^p$ that 
        \begin{enumerate}
            \item $\inpro{w-v, \proxp{w} - \proxp{v}}\geq (1-\gamma\rho)\norm{\proxp{w} - \proxp{v}}^2$, and 
            \item $\norm{w-v}\geq (1-\gamma\rho)\norm{\proxp{w} - \proxp{v}}$.
        \end{enumerate}
    \end{lemma}

Lemma~\ref{lem:Fnor_Lip} then follows from Lemma~\ref{lem:prox_wcvx} together with the definition of $\Fnor{\cdot}$ in \eqref{eq:norm_def}. A proof can be found, for example, in \citep{huang2024distributed,qiu2023normal}.
\begin{lemma}
    \label{lem:Fnor_Lip}
    Let Assumptions \ref{as:smooth} and \ref{as:phi} hold. Set $\gamma < 1/\rho$. The normal maps $\Fnor{\cdot}:\R^p\rightarrow\R^p$ and $\Fnori{i}{\cdot}: \R^p\rightarrow\R^p$ are $\Ln$-Lipschitz continuous, where $\Ln:= \prti{L + 2/\gamma}/\prti{1 - \gamma\rho}$.
\end{lemma}

\subsection{Proof of Lemma \ref{lem:cH_descent}}
\label{app:cH_descent}

    \textit{Step I: Relating $z_{t_1}$ and $z_{t_2}$.}
    Based on \eqref{eq:zitell} and \eqref{eq:yit}, we have
    \begin{equation}
        \label{eq:yit_z}
        y_{i,t} = \frac{1}{Q}\sum_{\ell=0}^{Q-1}\brk{g_i(x_{i,t}^\ell;\xi_{i,t}^\ell) + \gamma^{-1}\prt{z_{t} - x_{t}}} + c_{i,t}.
    \end{equation}

    From \eqref{eq:cit} and the initialization $c_{i,0} = 0$, we have $\bar{c}_t = \bar{c}_0 = 0$
    Accordingly, the server update can be written as the approximate normal map step
    \begin{equation}
        \label{eq:ztp1}
        \begin{aligned}
            z_{t + 1} &= z_t - \frac{\teta}{nQ}\sum_{i=1}^n \sum_{\ell=0}^{Q-1}\brk{g_i(x_{i,t}^\ell;\xi_{i,t}^\ell) + \gamma^{-1}\prt{z_{t} - x_{t}}}.
        \end{aligned}
    \end{equation}

    Hence, we have for any $t_1<t_2=t_1 + m$ that 
    \begin{equation}
        \label{eq:zt1t2}
        \begin{aligned}
            z_{t_2} &= z_{t_1} - \teta m \Fnor{z_{t_1}} - \frac{\teta}{nQ}\sum_{t=t_1}^{t_2-1}\sum_{i=1}^n \sum_{\ell=0}^{Q-1}\brk{g_i(x_{i,t}^\ell;\xi_{i,t}^\ell) + \gamma^{-1}\prt{z_{t} - x_{t}} - \Fnor{z_{t_1}}}\\
            &= z_{t_1} - \teta m \Fnor{z_{t_1}}  + e_{t_1:t_2},
        \end{aligned}
    \end{equation}
    where 
    \begin{equation}
        \label{eq:et1t2}
        \begin{aligned}
            e_{t_1:t_2} &= - \frac{\teta}{nQ}\sum_{t=t_1}^{t_2-1}\sum_{i=1}^n \sum_{\ell=0}^{Q-1}\brk{g_i(x_{i,t}^\ell;\xi_{i,t}^\ell) + \gamma^{-1}\prt{z_{t} - x_{t}} - \Fnor{z_{t_1}}}\\
            &= - \frac{\teta}{nQ}\sum_{t=t_1}^{t_2-1}\sum_{i=1}^n \sum_{\ell=0}^{Q-1}\Delta_{i,t}^\ell - \frac{\teta}{nQ}\sum_{t=t_1}^{t_2-1}\sum_{i=1}^n \sum_{\ell=0}^{Q-1}\brk{\nabla f_i(x_{i,t}^\ell) - \nabla f_i(x_{t_1})}\\
            &\quad + \frac{\teta}{\gamma}\sum_{t=t_1}^{t_2-1} \prt{z_{t_1} - z_t} - \frac{\teta}{\gamma }\sum_{t=t_1}^{t_2-1} \prt{x_{t_1} - x_t},\\
            \Delta_{i,t}^\ell &= g_i(x_{i,t}^\ell;\xi_{i,t}^\ell) - \nabla f_i(x_{i,t}^\ell).
        \end{aligned}
    \end{equation}

    \textit{Step II: Relating $\psi(x_{t_1})$ and $\psi(x_{t_2})$.}
    Due to Assumption \ref{as:phi} that $\varphi$ is $\rho$-weakly convex, we have for any $x, x'\in\dom{\varphi}$ and any $v\in\partial\varphi(x')$ that
    \begin{equation}
        \label{eq:wcvx0}
        \varphi(x)\geq \varphi(x') + \inpro{v, x - x'} - \frac{\rho}{2}\norm{x - x'}^2.
    \end{equation} 

    Noting $x_{t_2} = \proxp{z_{t_2}}$, it follows from the second prox theorem \citep{beck2017first} that $\gamma^{-1}(z_{t_2} - x_{t_2})\in\partial \varphi(x_{t_2})$. Therefore, setting $x' = x_{t_2}$, $x = x_{t_1}$, and $v = \gamma^{-1}\prti{z_{t_2} - x_{t_2}}$ in \eqref{eq:wcvx0} leads to 
    \begin{equation}
        \label{eq:phi_t1t2}
        \begin{aligned}
            \varphi(x_{t_2}) &\leq \varphi(x_{t_1}) - \inpro{\gamma^{-1}(z_{t_2} - x_{t_2}), x_{t_1} - x_{t_2}} + \frac{\rho}{2}\norm{x_{t_2} - x_{t_1}}^2.
        \end{aligned}
    \end{equation}

    Applying the descent lemma to $f$ yields 
    \begin{equation}
        \label{eq:f_t1t2}
        \begin{aligned}
            f(x_{t_2}) &\leq f(x_{t_1}) + \inpro{\nabla f(x_{t_1}), x_{t_2} - x_{t_1}} + \frac{L}{2}\norm{x_{t_2} - x_{t_1}}^2\\
            &= f(x_{t_1}) + \inpro{\Fnor{z_{t_1}} - \gamma^{-1}(z_{t_1} - x_{t_1}), x_{t_2} - x_{t_1}} + \frac{L}{2}\norm{x_{t_2} - x_{t_1}}^2.
        \end{aligned}
    \end{equation}

    Combining \eqref{eq:phi_t1t2} and \eqref{eq:f_t1t2} leads to
    \begin{equation}
        \label{eq:psi_t1t2}
        \begin{aligned}
            \psi(x_{t_2}) &\leq \psi(x_{t_1}) + \inpro{\Fnor{z_{t_1}}, x_{t_2} - x_{t_1}} + \prt{\frac{L + \rho}{2} - \frac{1}{\gamma}}\norm{x_{t_2} - x_{t_1}}^2\\
            &\quad + \inpro{\gamma^{-1}(z_{t_2} - z_{t_1}), x_{t_2} - x_{t_1}}.
        \end{aligned}
    \end{equation}

    \textit{Step III: Relating $\normi{\Fnor{z_{t_1}}}^2$ and $\normi{\Fnor{z_{t_2}}}^2$.}
    According to \eqref{eq:zt1t2}, we have $\Fnor{z_{t_2}}= \prti{1 - {\teta m}/{\gamma}}\Fnor{z_{t_1}} + \nabla f(x_{t_2}) - \nabla f(x_{t_1}) - \gamma^{-1}(x_{t_2} - x_{t_1})  + \gamma^{-1}e_{t_1:t_2}.$
    Hence, 
    \begin{equation}
        \label{eq:Fnor_t1t2_norm}
        \begin{aligned}
            &\norm{\Fnor{z_{t_2}}}^2 = \prt{1 - \frac{\teta m}{\gamma}}^2\norm{\Fnor{z_{t_1}}}^2 + \norm{\nabla f(x_{t_2}) - \nabla f(x_{t_1})}^2 + \frac{1}{\gamma^2}\norm{x_{t_2} - x_{t_1}}^2 \\
            &\quad + \frac{1}{\gamma^2}\norm{e_{t_1:t_2}}^2 + 2\prt{1 - \frac{\teta m}{\gamma}}\inpro{\Fnor{z_{t_1}}, \nabla f(x_{t_2}) - \nabla f(x_{t_1}) - \gamma^{-1}(x_{t_2} - x_{t_1}) + \gamma^{-1}e_{t_1:t_2}}\\
            &\quad  + \frac{2}{\gamma}\inpro{\nabla f(x_{t_2}) - \nabla f(x_{t_1}) - \gamma^{-1}(x_{t_2} - x_{t_1}), e_{t_1:t_2}} - \frac{2}{\gamma}\inpro{\nabla f(x_{t_2}) - \nabla f(x_{t_1}), x_{t_2} - x_{t_1}}\\
            &\leq \prt{1 - \frac{\teta m}{\gamma}}^2\norm{\Fnor{z_{t_1}}}^2 + \prt{L + \frac{1}{\gamma}}^2 \norm{x_{t_2} - x_{t_1}}^2 + \frac{1}{\gamma^2}\norm{e_{t_1:t_2}}^2\\
            &\quad + 2\prt{1 - \frac{\teta m}{\gamma}}\inpro{\Fnor{z_{t_1}}, \nabla f(x_{t_2}) - \nabla f(x_{t_1}) - \gamma^{-1}(x_{t_2} - x_{t_1}) + \gamma^{-1}e_{t_1:t_2}}\\
            &\quad + \frac{2}{\gamma^2}\inpro{\gamma\brk{\nabla f(x_{t_2}) - \nabla f(x_{t_1})} - (x_{t_2} - x_{t_1}), e_{t_1:t_2}},
        \end{aligned}
    \end{equation}
    where we applied the Cauchy-Schwarz inequality and invoked the $L$-smoothness of $f$ in the last inequality. 

    \textit{Step IV: Relating $\cH_{t_1}$ and $\cH_{t_2}$.}
    Substituting \eqref{eq:Fnor_t1t2_norm} and \eqref{eq:psi_t1t2} into \eqref{eq:cH_can} and rearranging the terms yields 
    \begin{equation}
        \label{eq:cH_can_s1}
        \begin{aligned}
            \cH_{t_2} 
            &\leq \cH_{t_1} + \brk{\frac{L + \rho}{2} - \frac{1}{\gamma} + \frac{\gamma \cC_0}{2}\prt{L + \frac{1}{\gamma}}^2}\norm{x_{t_2} - x_{t_1}}^2 + \frac{\cC_0}{2\gamma}\norm{e_{t_1:t_2}}^2\\
            &\quad + \inpro{\Fnor{z_{t_1}} - \cC_0\prt{1-\frac{\teta m}{\gamma}}\Fnor{z_{t_1}} + \frac{z_{t_2}-z_{t_1}}{\gamma} - \frac{\cC_0 e_{t_1:t_2}}{\gamma}, x_{t_2} - x_{t_1}}\\
            &\quad +  \cC_0\prt{1 - \frac{\teta m}{\gamma}}\inpro{\Fnor{z_{t_1}}, \gamma\brk{\nabla f(x_{t_2}) - \nabla f(x_{t_1})} + e_{t_1:t_2}}\\
            &\quad + \frac{\cC_0}{\gamma}\inpro{\gamma\brk{\nabla f(x_{t_2}) - \nabla f(x_{t_1})}, e_{t_1:t_2}} - \frac{\cC_0 \teta m}{2}\prt{2- \frac{\teta m}{\gamma}}\norm{\Fnor{z_{t_1}}}^2.
        \end{aligned}
    \end{equation}

    According to \eqref{eq:zt1t2}, we have $e_{t_1:t_2} - \teta m \Fnor{z_{t_1}} = z_{t_2} - z_{t_1}$ and $\Fnor{z_{t_1}} = \prti{z_{t_1} - z_{t_2} + e_{t_1:t_2}}/(\teta m)$. Then, 
    \begin{equation}
        \label{eq:cH_inner1}
        \begin{aligned}
            &\inpro{\Fnor{z_{t_1}} - \cC_0\prt{1-\frac{\teta m}{\gamma}}\Fnor{z_{t_1}} + \frac{z_{t_2}-z_{t_1}}{\gamma} - \frac{\cC_0 e_{t_1:t_2}}{\gamma}, x_{t_2} - x_{t_1}}\\
            &= \inpro{(1-\cC_0)\prt{1-\frac{\teta m}{\gamma}}\Fnor{z_{t_1}}+\frac{1-\cC_0}{\gamma}e_{t_1:t_2}, x_{t_2} - x_{t_1}}\\
            &= \frac{1-\cC_0}{\teta m}\inpro{e_{t_1:t_2}, x_{t_2} - x_{t_1}} - \prt{1-\cC_0}\prt{\frac{1}{\teta m}-\frac{1}{\gamma}}\inpro{z_{t_2} - z_{t_1}, x_{t_2} - x_{t_1}}\\
            &\leq \brk{\frac{1-\cC_0}{4\teta m} - \prt{1-\cC_0}\prt{\frac{1}{\teta m} - \frac{1}{\gamma}}\prt{1-\gamma \rho}}\norm{x_{t_2} - x_{t_1}}^2 + \frac{1-\cC_0}{\teta m}\norm{e_{t_1:t_2}}^2,
        \end{aligned}
    \end{equation}
    where we applied Lemma \ref{lem:prox_wcvx} by noting $x_t = \proxpi{z_t}$ and setting $0<\cC_0<1$ and  $\teta m < \gamma$: 
    \begin{align*}
        - \prt{1-\cC_0}\prt{\frac{1}{\teta m}-\frac{1}{\gamma}}\inpro{z_{t_2} - z_{t_1}, x_{t_2} - x_{t_1}} &\leq (1 - \gamma \rho)\prt{1-\cC_0}\prt{\frac{1}{\teta m}-\frac{1}{\gamma}}\norm{x_{t_2} - x_{t_1}}^2,
    \end{align*}
    and Young's inequality:
    \begin{align*}
        \inpro{e_{t_1:t_2}, x_{t_2} - x_{t_1}} &\leq \frac{1}{4}\norm{x_{t_2} - x_{t_1}}^2 + \norm{e_{t_1:t_2}}^2.
    \end{align*}

    For the remaining inner products in \eqref{eq:cH_can_s1}, it follows from Young's inequality that 
    \begin{equation}
        \label{eq:cH_inners}
        \begin{aligned}
            \inpro{\Fnor{z_{t_1}}, \gamma\brk{\nabla f(x_{t_2}) - \nabla f(x_{t_1})} + e_{t_1:t_2}} &\leq \frac{\teta m}{2}\norm{\Fnor{z_{t_1}}}^2 + \frac{\gamma^2 L^2 \normi{x_{t_1} - x_{t_2}}^2 + \normi{e_{t_1:t_2}}^2}{\teta m} \\
            \frac{\cC_0}{\gamma}\inpro{\gamma\brk{\nabla f(x_{t_2}) - \nabla f(x_{t_1})}, e_{t_1:t_2}}&\leq \frac{\cC_0 L}{2}\norm{x_{t_1} - x_{t_2}}^2 + \frac{\cC_0 L}{2}\norm{e_{t_1:t_2}}^2.
        \end{aligned}
    \end{equation}

    Substituting \eqref{eq:cH_inner1} and \eqref{eq:cH_inners} into \eqref{eq:cH_can_s1} leads to 
    \begin{equation}
        \label{eq:cH_can_s2}
        \begin{aligned}
            &\cH_{t_2} \leq \cH_{t_1} - \frac{\cC_0\teta m}{2}\brk{2 - \frac{\teta m}{\gamma} - \prt{1 - \frac{\teta m}{\gamma}}}\norm{\Fnor{z_{t_1}}}^2 + \tc_1 \norm{x_{t_1} - x_{t_2}}^2\\
            &\quad + \brk{\frac{\cC_0}{2\gamma} + \frac{1-\cC_0}{\teta m}  + \frac{\cC_0(1-\teta m /\gamma)}{\teta m} + \frac{\cC_0 L}{2}}\norm{e_{t_1:t_2}}^2,
        \end{aligned}
    \end{equation}
    where 
    \begin{align*}
        \tc_1&:= \brk{\frac{L + \rho}{2} - \frac{1}{\gamma} + \frac{\gamma \cC_0}{2}\prt{L + \frac{1}{\gamma}}^2} + \brk{\frac{1-\cC_0}{4\teta m} - \prt{1-\cC_0}\prt{\frac{1}{\teta m} - \frac{1}{\gamma}}\prt{1-\gamma \rho}} \\
        &\quad + \frac{\cC_0\gamma^2 L^2}{\teta m}\prt{1 - \frac{\teta m}{\gamma}} + \frac{\cC_0 L}{2}\\
        &= -\frac{3-4\gamma\rho}{4\teta m } + \frac{\cC_0\prt{3 - 4\gamma\rho + 4\gamma^2 L^2}}{4\teta m} + \cC_0\prt{\rho - \frac{1}{2\gamma} + \frac{3L}{2} - \frac{\gamma L^2}{2}} + \frac{L-\rho}{2}.
    \end{align*}

    Based on the definition of $\cC_0$ in \eqref{eq:cH_can} and $\gamma\leq 1/[5(\rho + L)]$, we have $\cC_0 < 1/2$ and $\cC_0\brki{\rho - {1}/\prti{2\gamma} + 1.5L - 0.5\gamma L^2} + 0.5\prti{L-\rho} < 0.$
    Therefore, 
    \begin{equation}
        \label{eq:cH}
        \begin{aligned}
            \cH_{t_2} &\leq \cH_{t_1} - \frac{\cC_0 \teta m}{2}\norm{\Fnor{z_{t_1}}}^2 - \frac{1}{8\teta m}\norm{x_{t_1} - x_{t_2}}^2 + \frac{1}{\teta m}\norm{e_{t_1:t_2}}^2.
        \end{aligned}
    \end{equation}

    \textit{Step V: Upper bounding $\E\brki{\normi{e_{t_1:t_2}}^2|\cF_{t_1}}$.}
    We now bound $\E\brki{\normi{e_{t_1:t_2}}^2|\cF_{t_1}}$. We start with any $(t_1+1)\leq t\leq t_2$. 
    Let 
    \begin{equation}
        \label{eq:et1t}
        \begin{aligned}
            e_{t_1:t}&:= - \frac{\teta}{nQ}\sum_{r=t_1}^{t-1}\sum_{\ell=0}^{Q-1}\sumn\brk{g_i(x_{i,r}^\ell;\xi_{i,r}^\ell) + \gamma^{-1}\prt{z_r - x_r} - \Fnor{z_{t_1}}},\; t_1+1\leq t\leq t_2.
        \end{aligned}
    \end{equation}
    
    It follows from \eqref{eq:ztp1} that 
    \begin{equation}
    \label{eq:zt_zt1}
        \begin{aligned}
            z_t 
            &= z_{t_1} - \teta (t - t_1) \Fnor{z_{t_1}} + e_{t_1:t},\; t_1 + 1 \leq t \leq t_2.
        \end{aligned}
    \end{equation}

    Noting that $ \Delta_{i,t}^\ell = g_i(x_{i,t}^\ell;\xi_{i,t}^\ell) - \nabla f_i(x_{i,t}^\ell)$, we obtain the following relation for $e_{t_1:t}$:
    \begin{equation}
        \label{eq:et1t_split}
        \begin{aligned}
            e_{t_1:t} &= - \frac{\teta}{nQ}\sum_{r=t_1}^{t-1} \sum_{\ell=0}^{Q-1} \sum_{i=1}^n\Delta_{i,r}^\ell - \frac{\teta}{nQ}\sum_{r=t_1}^{t-1} \sum_{\ell=0}^{Q-1}\sum_{i=1}^n\brk{\nabla f_i(x_{i,r}^\ell) - \nabla f_i(x_{t_1})}\\
            &\quad + \frac{\teta}{\gamma }\sum_{r=t_1}^{t-1} \prt{z_{t_1} - z_{r}} - \frac{\teta}{\gamma }\sum_{r=t_1}^{t-1}  \prt{x_{t_1} - x_{r}},\;  t_1 + 1 \leq t \leq t_2.
        \end{aligned}
    \end{equation}

    Taking the squared norm and conditional expectation on both sides of \eqref{eq:et1t_split} and invoking Assumption~\ref{as:smooth} yields
    \begin{equation}
        \label{eq:et1t_norm}
        \begin{aligned}
            &\frac{1}{4}\condE{\norm{e_{t_1:t}}^2}{\cF_{t_1}}\leq \frac{\teta^2 L^2 (t-t_1)}{nQ}\sum_{r=t_1}^{t-1} \sum_{\ell=0}^{Q-1}\sum_{i=1}^n\condE{\norm{\proxp{z_{i,r}^\ell} - \proxp{z_{t_1}}}^2}{\cF_{t_1}} \\
            &\quad + \frac{\teta^2 (t-t_1)\sigma^2}{nQ} + \frac{\teta^2 (t-t_1)}{ \gamma^2} \sum_{r=t_1}^{t-1}  \condE{\norm{z_{t_1} - z_{r}}^2}{\cF_{t_1}}\\
            &\quad + \frac{\teta^2(t-t_1)}{\gamma^2}\sum_{r=t_1}^{t-1}\condE{\norm{\proxp{z_{t_1}} - \proxp{z_r}}^2}{\cF_{t_1}}\\
            &\leq \frac{\teta^2 (t-t_1)\sigma^2}{nQ} + \frac{\teta^2 L^2 (t-t_1)}{nQ(1-\gamma\rho)^2 } \sum_{r=t_1}^{t-1} \sum_{\ell=0}^{Q-1}\sum_{i=1}^n \condE{\norm{z_{t_1} - z_{i,r}^\ell}^2}{\cF_{t_1}}\\
            &\quad + \frac{2\teta^2 (t-t_1)}{\gamma^2(1-\gamma\rho)^2}\sum_{r=t_1}^{t-1}\condE{\norm{z_{t_1} - z_r}^2}{\cF_{t_1}},\;  t_1 + 1 \leq t \leq t_2,
        \end{aligned}
    \end{equation}
    where we invoked Lemma~\ref{lem:prox_wcvx} in the last inequality. We next bound the last summation in \eqref{eq:et1t_norm}.
    It follows from \eqref{eq:zt_zt1} that
    \begin{equation}
        \label{eq:zr_zt1}
        \begin{aligned}
            &\condE{\norm{z_t - z_{t_1}}^2}{\cF_{t_1}} \leq 2\teta^2 (t - t_1)^2\norm{\Fnor{z_{t_1}}}^2 + 2\condE{\norm{e_{t_1:t}}^2}{\cF_{t_1}}\\
            &\leq 2\teta^2 (t - t_1)^2\norm{\Fnor{z_{t_1}}}^2 + \frac{8\teta^2 (t-t_1)\sigma^2}{nQ} + \frac{16\teta^2 (t-t_1)}{\gamma^2(1-\gamma\rho)^2}\sum_{r=t_1}^{t-1}\condE{\norm{z_{t_1} - z_r}^2}{\cF_{t_1}} \\
            &\quad + \frac{8\teta^2 L^2 (t-t_1)}{nQ(1-\gamma\rho)^2 } \sum_{r=t_1}^{t-1} \sum_{\ell=0}^{Q-1}\sum_{i=1}^n \condE{\norm{z_{t_1} - z_{i,r}^\ell}^2}{\cF_{t_1}} ,\; \forall t_1 + 1 \leq t \leq t_2.
        \end{aligned}
    \end{equation}
    
    By convention, an empty sum is zero. Then \eqref{eq:zr_zt1} also holds for $t = t_1$.
    Summing $t = t_1$, $t_1 + 1$, $t_1 + 2$, $\ldots$, $t_2$ on both sides of \eqref{eq:zr_zt1} and noting $t_2 - t_1 = m$ yields
    \begin{equation}
        \label{eq:sum_zt_zt1}  
        \begin{aligned}
            &\brk{1 - \frac{8\teta^2 m^2}{\gamma^2(1-\gamma\rho)^2}}\sum_{t=t_1}^{t_2}\condE{\norm{z_t - z_{t_1}}^2}{\cF_{t_1}} \leq \teta^2 m^3 \norm{\Fnor{z_{t_1}}}^2 + \frac{4\teta^2 m^2\sigma^2}{nQ} \\
            &\quad + \frac{4\teta^2 L^2 m^2}{nQ(1-\gamma\rho)^2 } \sum_{t=t_1}^{t_2-1} \sum_{\ell=0}^{Q-1}\sum_{i=1}^n \condE{\norm{z_{t_1} - z_{i,t}^\ell}^2}{\cF_{t_1}}.
        \end{aligned}
    \end{equation}

    Letting $\teta\leq (1-\gamma\rho)\gamma/(4m)$ yields 
    \begin{equation}
        \label{eq:sum_zt_zt1_ub}  
        \begin{aligned}
            &\sum_{t=t_1}^{t_2-1}\condE{\norm{z_t - z_{t_1}}^2}{\cF_{t_1}} \leq \sum_{t=t_1}^{t_2}\condE{\norm{z_t - z_{t_1}}^2}{\cF_{t_1}} \leq 2\teta^2 m^3 \norm{\Fnor{z_{t_1}}}^2 + \frac{8\teta^2 m^2\sigma^2}{nQ} \\
            &\quad + \frac{8\teta^2 L^2 m^2}{nQ(1-\gamma\rho)^2 } \sum_{t=t_1}^{t_2-1} \sum_{\ell=0}^{Q-1}\sum_{i=1}^n \condE{\norm{z_{t_1} - z_{i,t}^\ell}^2}{\cF_{t_1}}.
        \end{aligned}
    \end{equation}

    Substituting \eqref{eq:sum_zt_zt1_ub} into \eqref{eq:et1t_norm} leads to 
    \begin{equation}
        \label{eq:et1t_ub}
        \begin{aligned}
            &\condE{\norm{e_{t_1:t}}^2}{\cF_{t_1}} \leq \frac{4\teta^2 (t-t_1)\sigma^2}{nQ} + \frac{4\teta^2 L^2 (t-t_1)}{nQ(1-\gamma\rho)^2 } \sum_{r=t_1}^{t_2-1} \sum_{\ell=0}^{Q-1}\sum_{i=1}^n \condE{\norm{z_{t_1} - z_{i,r}^\ell}^2}{\cF_{t_1}}\\
            &\quad \frac{8\teta^2(t-t_1)}{\gamma^2(1-\gamma\rho)^2}\crk{2\teta^2 m^3 \norm{\Fnor{z_{t_1}}}^2 + \frac{8\teta^2 m^2\sigma^2}{nQ} \right.\\
            &\left.\quad  + \frac{8\teta^2 L^2 m^2}{nQ(1-\gamma\rho)^2 } \sum_{r=t_1}^{t_2-1} \sum_{\ell=0}^{Q-1}\sum_{i=1}^n \condE{\norm{z_{t_1} - z_{i,r}^\ell}^2}{\cF_{t_1}} }\\
            &\leq \frac{5\teta^2 (t-t_1)\sigma^2}{nQ} + \frac{5\teta^2 L^2 (t-t_1)}{nQ(1-\gamma\rho)^2 } \sum_{r=t_1}^{t_2-1} \sum_{\ell=0}^{Q-1}\sum_{i=1}^n \condE{\norm{z_{t_1} - z_{i,r}^\ell}^2}{\cF_{t_1}}\\
            &\quad + \frac{16\teta^4m^3(t-t_1)}{\gamma^2(1-\gamma\rho)^2}\norm{\Fnor{z_{t_1}}}^2,
        \end{aligned}
    \end{equation}
    where we invoked $\teta\leq (1-\gamma\rho)\gamma/(10 m)$ in the last inequality.

    Substituting \eqref{eq:et1t_ub} (set $t = t_2$) into \eqref{eq:cH} yields the desired result \eqref{eq:cH_descent}.

\subsection{Proof of Lemma \ref{lem:sum_ub}}
\label{app:sum_ub}

    Denote $\Fnori{i}{z_t}:=  \nabla f_i(\proxp{z_t}) + \gamma^{-1}\prt{z - \proxp{z_t}},$
    and the stacked variables ($n$ by $p$ matrices):
    \begin{align*}
        \z_r^{\ell} &:= \prt{z_{1,r}^\ell, z_{2,r}^\ell, \ldots, z_{n,r}^{\ell}}^{\T},\;
        \y_r := \prt{y_{1,r}, y_{2,r}, \ldots, y_{n,r}}^{\T},\;
        \x_r^{\ell} := \prt{x_{1,r}^\ell, x_{2,r}^\ell, \ldots, x_{n,r}^{\ell}}^{\T},\\
        \bc_r &:= \prt{c_{1,r},c_{2,r},\ldots, c_{n,r}}^{\T},\;
        \bFnor{\1 z_{t}^{\T}} := \prt{\Fnori{1}{z_t}, \Fnori{2}{z_t}, \ldots, \Fnori{n}{z_t}}^{\T},\\
        \bs{\Delta}_r^\ell &:= \prt{\Delta_{1,r}^\ell, \Delta_{2,r}^\ell, \ldots, \Delta_{n,r}^\ell}^{\T},\;
        \nabla F(\x_r^\ell) := \prt{\nabla f_1(x_{1,r}^\ell), \nabla f_2(x_{2,r}^\ell), \ldots, \nabla f_n(x_{n,r}^\ell)}^{\T}.
    \end{align*}
    
    For any $(t_1+1)\leq r\leq t_2$, it follows from the local update scheme in \eqref{eq:zitell}, the update for $c_{i,t}$ in \eqref{eq:cit}, and the multistep relation~\eqref{eq:zt_zt1} that 
    \begin{equation}
        \label{eq:zt_ztell}
            \begin{aligned}
                \z_r^\ell &= \z_r^0 - \eta_a \sum_{s=0}^{\ell-1}\brk{\g_r^s + \gamma^{-1}\prt{\1 z_r^{\T} - \1 x_r^{\T}} + \bc_r}\\
                &= \1 z_{t_1}^{\T} - \teta(r-t_1)\1 \brk{\Fnor{z_{t_1}}}^{\T} + \1 e_{t_1:r}^{\T} - \eta_a \ell \brk{\bc_{r} + \bFnor{\1 z_{t_1}^{\T}}} + \eta_a\ell \bFnor{\1 z_{t_1}^{\T}} \\
                &\quad - \eta_a \sum_{s=0}^{\ell-1}\brk{\bs{\Delta}_r^s + \nabla F(\x_r^s) + \gamma^{-1}\prt{\1 z_r^{\T}  - \1 x_r^{\T}} } \\
                &= \1 z_{t_1}^{\T} - \teta(r-t_1)\1 \brk{\Fnor{z_{t_1}}}^{\T} + \1 e_{t_1:r}^{\T} - \eta_a \ell \Pi \brk{\bc_{r} + \bFnor{\1 z_{t_1}^{\T}}} - \eta_a \ell \1 \brk{\Fnor{z_{t_1}}}^{\T} \\
                &\quad - \eta_a \sum_{s=0}^{\ell-1}\brk{\bs{\Delta}_r^s + \nabla F(\x_r^s) - \nabla F(\1 x_{t_1}^{\T}) + \gamma^{-1}\prt{\1 z_r^{\T} - \1 z_{t_1}^{\T} + \1 x_{t_1}^{\T}  - \1 x_r^{\T}} },
            \end{aligned}
        \end{equation}
    where we invoked 
    \begin{equation}
        \label{eq:cr_Fnor}
        \begin{aligned}
            \bc_r + \bFnor{\1 z_{t_1}^{\T}} &= \Pi\brk{\bc_r + \bFnor{\1 z_{t_1}^{\T}}} + \1 \brk{\Fnor{z_{t_1}}}^{\T},\\
            \bFnor{\1 z_{t_1}^{\T}}&= \nabla F(\1 x_{t_1}^{\T}) + \gamma^{-1}\prt{\1 z_{t_1}^{\T} - \1 x_{t_1}^{\T}}.
        \end{aligned}
    \end{equation}
        
    We can rewrite $\Pi\brki{\bc_r + \bFnor{\1 z_{t_1}^{\T}}}$ for $ t_1 + 1 \leq r \leq t_2$ as follows:
    \begin{equation}
        \label{eq:Picr_Fnor}
        \begin{aligned}
            \Pi\brk{\bc_{r} + \bFnor{\1 z_{t_1}^{\T}}} &= -\frac{1}{Q} \sum_{\ell=0}^{Q-1}\Pi \brk{\bs{\Delta}_{t_1}^\ell + \nabla F(\x_{t_1}^\ell) - \nabla F(\1 x_{t_1}^{\T})} - \sum_{p=t_1 + 1}^{r - 1}\Pi \y_p.
        \end{aligned}
    \end{equation}
    To derive \eqref{eq:Picr_Fnor}, we proceed as follows. According to \eqref{eq:zitell} and \eqref{eq:yit}, we have
    \begin{equation}
        \label{eq:yr}
        \begin{aligned}
            \y_r = \bc_r + \frac{1}{Q}\sum_{\ell=0}^{Q-1}\brk{\bs{\Delta}_r^\ell + \nabla F(\x_r^\ell) + \gamma^{-1}\prt{\1 z_r^{\T} - \1 x_r^{\T}}}.
        \end{aligned}
    \end{equation}
    It follows from \eqref{eq:cit} that 
    \begin{equation}
        \label{eq:ctp1}
        \begin{aligned}
            \bc_{r + 1} &= \bc_r  -\Pi  \y_r.
        \end{aligned}
    \end{equation}
    On one hand, unrolling \eqref{eq:ctp1} leads to 
    \begin{equation}
        \label{eq:Picr}
        \begin{aligned}
            \Pi\bc_{r} &= \Pi\bc_{t_1 + 1} - \sum_{p=t_1 + 1}^{r - 1}\Pi \y_p,\; t_1 + 1 \leq r \leq t_2,
        \end{aligned}
    \end{equation}
    where $\sum_{p=t_1 + 1}^{t_1} \Pi \y_p:= \mathbf{0}$.
    On the other hand, substituting \eqref{eq:yr} into \eqref{eq:ctp1} and noting $\Pi^2 = \Pi$ yields
    \begin{equation}
        \label{eq:Pict1}
        \begin{aligned}
            \Pi\bc_{t_1 + 1} &= \Pi \bc_{t_1} - \Pi\bc_{t_1} - \frac{1}{Q} \sum_{\ell=0}^{Q-1}\Pi \brk{\bs{\Delta}_{t_1}^\ell + \nabla F(\x_{t_1}^\ell) + \gamma^{-1}\prt{\1 z_{t_1}^{\T} - \1 x_{t_1}^{\T}}} \\
            &= -\Pi \bFnor{\1 z_{t_1}^{\T}} - \frac{1}{Q} \sum_{\ell=0}^{Q-1}\Pi \brk{\bs{\Delta}_{t_1}^\ell + \nabla F(\x_{t_1}^\ell) - \nabla F(\1 x_{t_1}^{\T})}.
        \end{aligned}
    \end{equation}

    Combining \eqref{eq:Picr} and \eqref{eq:Pict1} yields \eqref{eq:Picr_Fnor}.
    Substituting \eqref{eq:Picr_Fnor} into \eqref{eq:zt_ztell} yields for any $(t_1+1)\leq r\leq t_2$,
    \begin{equation}
        \label{eq:zt_ztell1}
        \begin{aligned}
            &\z_r^\ell - \1 z_{t_1}^{\T} = - \teta(r-t_1)\1 \brk{\Fnor{z_{t_1}}}^{\T} + \1 e_{t_1:r}^{\T} + \frac{\eta_a\ell}{Q} \sum_{\ell=0}^{Q-1}\Pi \brk{\bs{\Delta}_{t_1}^\ell + \nabla F(\x_{t_1}^\ell) - \nabla F(\1 x_{t_1}^{\T})}  \\
            &\quad  - \eta_a \sum_{s=0}^{\ell-1}\brk{\bs{\Delta}_r^s + \nabla F(\x_r^s) - \nabla F(\1 x_{t_1}^{\T}) + \gamma^{-1}\prt{\1 z_r^{\T} - \1 z_{t_1}^{\T} + \1 x_{t_1}^{\T}  - \1 x_r^{\T}} }\\
            &\quad + \eta_a\ell \sum_{p=t_1 + 1}^{r - 1}\Pi \y_p - \eta_a \ell \1 \brk{\Fnor{z_{t_1}}}^{\T}.
        \end{aligned}
    \end{equation}


Taking squared norms and conditional expectations on both sides of \eqref{eq:zt_ztell1} and invoking Assumptions~\ref{as:abc}, \ref{as:smooth}, and Lemma~\ref{lem:Fnor_Lip} yields 
\begin{equation}
    \label{eq:zr_zrell_norm}
    \begin{aligned}
        &\frac{1}{10}\condE{\norm{\z_r^\ell - \1 z_{t_1}^{\T}}^2}{\cF_{t_1}}\leq \teta^2(r-t_1)^2n\norm{\Fnor{z_{t_1}}}^2 + n\condE{\norm{e_{t_1:r}}^2}{\cF_{t_1}} + \frac{\eta_a^2 \ell^2 n\sigma^2}{Q}\\
        &\quad + \frac{\eta_a^2\ell^2 L^2}{Q(1-\gamma\rho)^2}\sum_{\ell=0}^{Q-1}\condE{\norm{\z_{t_1}^\ell - \1 z_{t_1}^{\T}}^2}{\cF_{t_1}} + \eta_a^2 \ell^2 \condE{\norm{\sum_{p=t_1 + 1}^{r - 1}\Pi\y_p}^2}{\cF_{t_1}} \\
        &\quad + \eta_a^2\ell^2 n\norm{\Fnor{z_{t_1}}}^2 + \eta_a^2\ell n\sigma^2 + \frac{\eta_a^2\ell L^2}{(1-\gamma\rho)^2}\sum_{s=0}^{\ell-1}\condE{\norm{\z_r^s - \1 z_{t_1}^{\T}}^2}{\cF_{t_1}} \\
        &\quad + \frac{\eta_a^2\ell^2 n}{\gamma^2}\condE{\norm{z_r - z_{t_1}}^2}{\cF_{t_1}} + \frac{\eta_a^2\ell^2 n}{\gamma^2(1-\gamma\rho)^2}\condE{\norm{z_r - z_{t_1}}^2}{\cF_{t_1}},\; t_1 + 1 \leq r \leq t_2.
    \end{aligned}
\end{equation}

We now bound the term $\condEi{\normi{\sum_{p=t_1 + 1}^{r - 1}\Pi\y_p}^2}{\cF_{t_1}}$. It follows from \eqref{eq:yitp1} and $\Pi\cdot \1 = \mathbf{0}$ that
\begin{equation}
    \label{eq:Piyr}
    \begin{aligned}
        \Pi\y_p &= \frac{1}{Q}\sum_{\ell=0}^{Q-1}\Pi\brk{\g_{p}^\ell - \g_{p-1}^\ell + \gamma^{-1}\prt{\1 z_{p}^{\T} - \1 z_{p-1}^{\T}} - \gamma^{-1}\prt{  \1 x_{p}^{\T} - \1 x_{p-1}^{\T}}}\\
        &= \frac{1}{Q}\sum_{\ell=0}^{Q-1}\Pi\brk{\g_{p}^\ell - \g_{p-1}^\ell}.
    \end{aligned}
\end{equation}
Summing \eqref{eq:Piyr} from $p = t_1 + 1$ to $r - 1$ yields for any $(t_1 + 1) \leq r \leq t_2$,
\begin{equation}
    \label{eq:Piyr_sum}
    \begin{aligned}
        &\sum_{p=t_1 + 1}^{r-1}\Pi\y_p = \frac{1}{Q}\sum_{\ell=0}^{Q-1}\Pi\brk{\g_{r-1}^\ell - \g_{t_1}^\ell}\\
        &= \frac{1}{Q}\sum_{\ell=0}^{Q-1}\Pi\brk{\bs{\Delta}_{r-1}^\ell + \nabla F(\x_{r-1}^\ell) - \nabla F(\1 x_{t_1}^{\T}) + \nabla F(\1 x_{t_1}^{\T}) - \nabla F(\x_{t_1}^\ell)  - \bs{\Delta}_{t_1}^\ell }.
    \end{aligned}
\end{equation}
Taking the squared norm and conditional expectation on both sides of \eqref{eq:Piyr_sum} and invoking Assumptions~\ref{as:abc} and \ref{as:smooth} yields for any $(t_1 + 1) \leq r \leq t_2$,
\begin{equation}
    \label{eq:Piyr_norm}
    \begin{aligned}
        &\frac{1}{4}\condE{\norm{\sum_{p=t_1+1}^{r-1}\Pi\y_p}^2}{\cF_{t_1}} 
        \leq \frac{2n\sigma^2}{Q} + \frac{L^2}{Q}\sum_{\ell=0}^{Q-1}\condE{\norm{\x_{r-1}^\ell - \1 x_{t_1}^{\T}}^2}{\cF_{t_1}} \\
        &\quad + \frac{L^2}{Q}\sum_{\ell=0}^{Q-1}\condE{\norm{\x_{t_1}^\ell - \1 x_{t_1}^{\T}}^2}{\cF_{t_1}}.
    \end{aligned}
\end{equation}
Substituting \eqref{eq:Piyr_norm} into \eqref{eq:zr_zrell_norm} and invoking Lemma~\ref{lem:prox_wcvx} leads to 

\begin{equation}
    \label{eq:zr_zrell_norm_s2}
    \begin{aligned}
        &\frac{1}{10}\condE{\norm{\z_r^\ell - \1 z_{t_1}^{\T}}^2}{\cF_{t_1}}\leq \brk{\teta^2(r-t_1)^2 + \eta_a^2\ell^2}n\norm{\Fnor{z_{t_1}}}^2 + n\condE{\norm{e_{t_1:r}}^2}{\cF_{t_1}} \\
        &\quad + \frac{9\eta_a^2\ell(\ell + Q) n\sigma^2}{Q} + \frac{5\eta_a^2\ell^2 L^2}{Q(1-\gamma\rho)^2}\sum_{\ell=0}^{Q-1}\condE{\norm{\z_{t_1}^\ell - \1 z_{t_1}^{\T}}^2}{\cF_{t_1}} \\
        &\quad + \frac{\eta_a^2\ell L^2}{(1-\gamma\rho)^2}\sum_{s=0}^{\ell-1}\condE{\norm{\z_r^s - \1 z_{t_1}^{\T}}^2}{\cF_{t_1}} + \frac{4\eta_a^2\ell^2 L^2}{Q(1-\gamma\rho)^2}\sum_{\ell=0}^{Q-1}\condE{\norm{\z_{r-1}^\ell - \1 z_{t_1}^{\T}}^2}{\cF_{t_1}} \\
        &\quad + \frac{2\eta_a^2\ell^2 n}{\gamma^2(1-\gamma\rho)^2}\condE{\norm{z_r - z_{t_1}}^2}{\cF_{t_1}}.
    \end{aligned}
\end{equation}

It follows from \eqref{eq:zt_zt1} that 
\begin{equation}
    \label{eq:zrp1_zt1_norm}
    \begin{aligned}
        \condE{\norm{z_{r} - z_{t_1}}^2}{\cF_{t_1}} &\leq 2\teta^2(r-t_1)^2\norm{\Fnor{z_{t_1}}}^2 + 2\condE{\norm{e_{t_1:r}}^2}{\cF_{t_1}}.
    \end{aligned}
\end{equation}
Combining \eqref{eq:et1t_ub}, \eqref{eq:zr_zrell_norm_s2}, and \eqref{eq:zrp1_zt1_norm} leads to 
\begin{equation}
    \label{eq:zr_zrell_norm_s3}
    \begin{aligned}
        &\frac{1}{10}\condE{\norm{\z_r^\ell - \1 z_{t_1}^{\T}}^2}{\cF_{t_1}}\leq \brk{\teta^2(r-t_1)^2 + \eta_a^2\ell^2 + \frac{4\teta^2\eta_a^2\ell^2(r-t_1)^2}{\gamma^2(1-\gamma\rho)^2} }n\norm{\Fnor{z_{t_1}}}^2\\
        &\quad + \brk{1 + \frac{4\eta_a^2\ell^2 n}{\gamma^2(1-\gamma\rho)^2}}n \condE{\norm{e_{t_1:r}}^2}{\cF_{t_1}} + 10\eta_a^2\ell n\sigma^2  \\
        &\quad + \frac{5\eta_a^2\ell^2 L^2}{Q(1-\gamma\rho)^2}\sum_{\ell=0}^{Q-1}\condE{\norm{\z_{t_1}^\ell - \1 z_{t_1}^{\T}}^2}{\cF_{t_1}} + \frac{\eta_a^2\ell L^2}{(1-\gamma\rho)^2}\sum_{s=0}^{\ell-1}\condE{\norm{\z_r^s - \1 z_{t_1}^{\T}}^2}{\cF_{t_1}} \\
        &\quad + \frac{4\eta_a^2\ell^2 L^2}{Q(1-\gamma\rho)^2}\sum_{\ell=0}^{Q-1}\condE{\norm{\z_{r-1}^\ell - \1 z_{t_1}^{\T}}^2}{\cF_{t_1}}\\
        &\leq \brk{2\teta^2(r-t_1)^2 + \eta_a^2\ell^2 + \frac{32\teta^4 m^3(r-t_1)}{\gamma^2(1-\gamma\rho)^2}}n\norm{\Fnor{z_{t_1}}}^2 + 10\brk{\eta_a^2\ell + \frac{\teta^2(r-t_1)}{nQ}} n\sigma^2 \\
        &\quad + \frac{10\teta^2 L^2(r-t_1)}{Q(1-\gamma\rho)^2} \sum_{r=t_1}^{t_2-1} \sum_{\ell=0}^{Q-1}\condE{\norm{\z_{r}^\ell - \1 z_{t_1}^{\T}}^2}{\cF_{t_1}} + \frac{5\eta_a^2\ell^2 L^2}{Q(1-\gamma\rho)^2}\sum_{\ell=0}^{Q-1}\condE{\norm{\z_{t_1}^\ell - \1 z_{t_1}^{\T}}^2}{\cF_{t_1}} \\
        &\quad + \frac{\eta_a^2\ell L^2}{(1-\gamma\rho)^2}\sum_{s=0}^{\ell-1}\condE{\norm{\z_r^s - \1 z_{t_1}^{\T}}^2}{\cF_{t_1}} + \frac{4\eta_a^2\ell^2 L^2}{Q(1-\gamma\rho)^2}\sum_{\ell=0}^{Q-1}\condE{\norm{\z_{r-1}^\ell - \1 z_{t_1}^{\T}}^2}{\cF_{t_1}},
    \end{aligned}
\end{equation}
where we let $\eta_a\leq (1-\gamma\rho)\gamma/(2Q)$.
Summing over $r=t_1+1$, $t_1 + 2$, $\ldots, t_2$ and $\ell=0,1,\ldots, Q-1$ on both sides of \eqref{eq:zr_zrell_norm_s3} leads to 
\begin{equation}
    \label{eq:zr_zrell_norm_sum}
    \begin{aligned}
        &\tc_2\sum_{r=t_1 + 1}^{t_2}\sum_{\ell=0}^{Q-1}\condE{\norm{\z_r^\ell - \1 z_{t_1}^{\T}}^2}{\cF_{t_1}}\leq \brk{2\teta^2 m^3 Q + \eta_a^2 Q^3 m + \frac{32\teta^4 m^5 Q}{\gamma^2(1-\gamma\rho)^2}}n\norm{\Fnor{z_{t_1}}}^2\\
        &\quad + 10\brk{\frac{\eta_a^2 Q^2m}{2} + \frac{\teta^2 m^2}{n}}n\sigma^2 + \frac{10\teta^2 L^2m^2}{(1-\gamma\rho)^2}\sum_{\ell=0}^{Q-1}\condE{\norm{\z_{t_1}^\ell - \1 z_{t_1}^{\T}}^2}{\cF_{t_1}} \\
        &\quad + \frac{5\eta_a^2 Q^2 L^2 m}{3(1-\gamma\rho)^2}\sum_{\ell=0}^{Q-1}\condE{\norm{\z_{t_1}^\ell - \1 z_{t_1}^{\T}}^2}{\cF_{t_1}} + \frac{4\eta_a^2 Q^2 L^2}{3(1-\gamma\rho)^2}\sum_{\ell=0}^{Q-1}\condE{\norm{\z_{t_1}^\ell - \1 z_{t_1}^{\T}}^2}{\cF_{t_1}},
    \end{aligned}
\end{equation}
where $\tc_2:= \brki{{(1-\gamma\rho)^2}/{10} - 10\teta^2 m^2 L^2 - \eta_a^2 Q^2 L^2 - \prti{4\eta_a^2 Q^2 L^2}/{3}}/{(1-\gamma\rho)^2 }.$
Letting $\teta\leq\prti{1-\gamma\rho}/\prti{20m\sqrt{L^2 + 1/\gamma^2}},\; \eta_a\leq \prti{1-\gamma\rho}/\prti{20QL},$
yields the desired result \eqref{eq:ztell_zt1_ub}. 

We now bound the term $\sum_{\ell=0}^{Q-1}\E\brki{\norm{\z_0^\ell - \1 z_0^{\T}}^2}$ following a similar procedure as in \eqref{eq:zt_ztell}. We have 
\begin{equation}
    \label{eq:z0_z0ell}
    \begin{aligned}
        \z_0^\ell&= \1 z_{0}^{\T}  - \eta_a \sum_{s=0}^{\ell-1}\brk{\bs{\Delta}_0^s + \nabla F(\x_0^s) - \nabla F(\1 x_0^{\T}) + \gamma^{-1}\prt{\1 z_0^{\T}  - \1 x_0^{\T}} } - \eta_a \ell \nabla F(\1 x_0^{\T}).
    \end{aligned}
\end{equation}
Taking squared norms and full expectations on both sides of \eqref{eq:z0_z0ell}, and invoking Assumptions~\ref{as:abc} and \ref{as:smooth}, and Lemma~\ref{lem:prox_wcvx} yields
\begin{equation}
    \label{eq:z0ell_z0_ub0}
    \begin{aligned}
       \frac{1}{4}\E\brk{\norm{\z_0^\ell - \1 z_0^{\T}}^2} &\leq \eta_a^2 \ell n\sigma^2 + \frac{\eta_a^2\ell L^2}{(1-\gamma\rho)^2}\sum_{s=0}^{\ell-1}\E\brk{\norm{\z_0^s - \1 z_0^{\T}}^2} + \frac{\eta_a^2\ell^2 n}{\gamma^2}\norm{\proxp{z_0} - z_0}^2 \\
       &\quad + \eta_a^2 \ell^2\sumn\norm{\nabla f_i(\proxp{z_0})}^2.
    \end{aligned}
\end{equation}
Summing over $\ell=0,1,\ldots, Q-1$ on both sides of \eqref{eq:z0ell_z0_ub0} and letting $\eta_a \leq (1-\gamma\rho)/(4QL)$ yields \eqref{eq:z0ell_z0_ub}.

\subsection{Proof of Lemma \ref{lem:zt2ell_zt2}}
\label{app:zt2ell_zt2}

Following similar derivations as in \eqref{eq:zt_ztell}, we have 
\begin{equation}
    \label{eq:zt2ell_zt2}
    \begin{aligned}
        \z_{t_2}^\ell 
        &= \1 z_{t_2}^{\T} - \eta_a \ell \Pi\brk{\bFnor{\1 z_{t_2}^{\T}} - \bFnor{\1 z_{t_1}^{\T}}} -\frac{\eta_a\ell}{Q}\sum_{\ell=0}^{Q-1} \Pi\brk{\bs\Delta_{t_1}^\ell + \nabla F(\x_{t_1}^\ell) - \nabla F(\1 x_{t_1}^{\T})} \\
        &\quad - \eta_a\ell\sum_{p=t_1 + 1}^{t_2 - 1}\Pi \y_p - \eta_a \ell \1 \brk{\Fnor{z_{t_2}}}^{\T} - \eta_a \sum_{s=0}^{\ell-1}\brk{\bs{\Delta}_{t_2}^s + \nabla F(\x_{t_2}^s) - \nabla F(\1 x_{t_2}^{\T})},
    \end{aligned}
\end{equation}
where we set $r = t_2$ in \eqref{eq:Picr} and substitute \eqref{eq:Pict1}. 
Taking the squared norm and conditional expectation on both sides of \eqref{eq:zt2ell_zt2}, and invoking Assumptions~\ref{as:abc} and \ref{as:smooth}, and Lemmas~\ref{lem:prox_wcvx} and \ref{lem:Fnor_Lip} yields
\begin{equation}
    \label{eq:zt2ell_zt2_norm}
    \begin{aligned}
        &\frac{1}{7}\condE{\norm{\z_{t_2}^\ell - \1 z_{t_2}^{\T}}^2}{\cF_{t_1}} \leq \eta_a^2\ell^2\Ln^2 n \condE{\norm{z_{t_1} - z_{t_2}}^2}{\cF_{t_1}} + \frac{\eta_a^2\ell^2 n\sigma^2}{Q} \\
        &\quad + \frac{\eta_a^2\ell^2 L^2}{Q(1-\gamma\rho)^2}\sum_{\ell=0}^{Q-1}\condE{\norm{\z_{t_1}^\ell - \1 z_{t_1}^{\T}}^2}{\cF_{t_1}} + \eta_a^2\ell^2 \condE{\norm{\sum_{p=t_1 + 1}^{t_2 - 1}\Pi \y_p}^2}{\cF_{t_1}} \\
        &\quad + \eta_a^2\ell^2 n \condE{\norm{\Fnor{z_{t_2}}}^2}{\cF_{t_1}} + \eta_a^2 \ell n\sigma^2 + \frac{\eta_a^2\ell L^2}{(1-\gamma\rho)^2} \sum_{s=0}^{\ell-1}\condE{\norm{\z_{t_2}^s - \1 z_{t_2}^{\T}}^2}{\cF_{t_1}}\\
        &\leq 3\eta_a^2\ell^2\Ln^2 n \condE{\norm{z_{t_1} - z_{t_2}}^2}{\cF_{t_1}} + 2\eta_a^2\ell^2 n \norm{\Fnor{z_{t_1}}}^2 + 2\eta_a^2\ell n\sigma^2 \\
        &\quad + \frac{\eta_a^2\ell^2 L^2}{Q(1-\gamma\rho)^2}\sum_{\ell=0}^{Q-1}\condE{\norm{\z_{t_1}^\ell - \1 z_{t_1}^{\T}}^2}{\cF_{t_1}} + \eta_a^2\ell^2 \condE{\norm{\sum_{p=t_1 + 1}^{t_2 - 1}\Pi \y_p}^2}{\cF_{t_1}} \\
        &\quad + \frac{\eta_a^2\ell L^2}{(1-\gamma\rho)^2} \sum_{s=0}^{\ell-1}\condE{\norm{\z_{t_2}^s - \1 z_{t_2}^{\T}}^2}{\cF_{t_1}}.
    \end{aligned}
\end{equation}
Substituting \eqref{eq:Piyr_norm} (setting $r = t_2$) into \eqref{eq:zt2ell_zt2_norm} and invoking Lemma~\ref{lem:prox_wcvx} leads to
\begin{equation}
    \label{eq:zt2ell_zt2_norm1}
    \begin{aligned}
        &\frac{1}{7}\condE{\norm{\z_{t_2}^\ell - \1 z_{t_2}^{\T}}^2}{\cF_{t_1}} \leq 3\eta_a^2\ell^2\Ln^2 n \condE{\norm{z_{t_1} - z_{t_2}}^2}{\cF_{t_1}} + 2\eta_a^2\ell^2 n \norm{\Fnor{z_{t_1}}}^2 + 6\eta_a^2\ell n\sigma^2 \\
        &\quad + \frac{5\eta_a^2\ell^2 L^2}{Q(1-\gamma\rho)^2}\sum_{\ell=0}^{Q-1}\condE{\norm{\z_{t_1}^\ell - \1 z_{t_1}^{\T}}^2}{\cF_{t_1}} + \frac{4\eta_a^2\ell^2 L^2}{Q(1-\gamma\rho)^2}\sum_{s=0}^{Q-1}\condE{\norm{\z_{t_2-1}^s - \1 z_{t_1}^{\T} }^2}{\cF_{t_1}} \\
        &\quad  + \frac{\eta_a^2\ell L^2}{(1-\gamma\rho)^2} \sum_{s=0}^{\ell-1}\condE{\norm{\z_{t_2}^s - \1 z_{t_2}^{\T}}^2}{\cF_{t_1}}.
    \end{aligned}
\end{equation}
Summing over $\ell=0,1,\ldots, Q-1$ on both sides of \eqref{eq:zt2ell_zt2_norm1} and letting $\eta_a \leq (1-\gamma\rho)/(6QL)$ yields 
\begin{equation}
    \label{eq:zt2ell_zt2_sum}
    \begin{aligned}
        &\sum_{\ell=0}^{Q-1}\condE{\norm{\z_{t_2}^\ell - \1 z_{t_2}^{\T}}^2}{\cF_{t_1}} \leq 8\eta_a^2Q^3\Ln^2 n \condE{\norm{z_{t_1} - z_{t_2}}^2}{\cF_{t_1}} + 6\eta_a^2Q^3 n \norm{\Fnor{z_{t_1}}}^2  \\
        &\quad + 24\eta_a^2Q^2 n\sigma^2 + \frac{40\eta_a^2Q^2 L^2}{3(1-\gamma\rho)^2}\sum_{\ell=0}^{Q-1}\condE{\norm{\z_{t_1}^\ell - \1 z_{t_1}^{\T}}^2}{\cF_{t_1}} \\
        &\quad + \frac{32\eta_a^2Q^2 L^2}{3(1-\gamma\rho)^2}\sum_{\ell=0}^{Q-1}\condE{\norm{\z_{t_2-1}^\ell - \1 z_{t_1}^{\T} }^2}{\cF_{t_1}}\\
        &\leq \frac{32\eta_a^2Q^2 L^2}{3(1-\gamma\rho)^2}\sum_{\ell=0}^{Q-1}\condE{\norm{\z_{t_2-1}^\ell - \1 z_{t_1}^{\T} }^2}{\cF_{t_1}} +  \frac{80\teta^2mL^2\eta_a^2 Q^2\Ln^2}{(1-\gamma\rho)^2}\sum_{r=t_{1}}^{t_2 - 1}\sum_{\ell=0}^{Q-1}\condE{\norm{\z_{r}^\ell - \1 z_{t_1}^{\T}}^2}{\cF_{t_1}}\\
        &\quad + 6\eta_a^2Q^3\prt{1 + \frac{8\teta^2 m^2\Ln^2}{3} + \frac{128\teta^4 m^4\Ln^2}{3\gamma^2(1-\gamma\rho)^2}} n \norm{\Fnor{z_{t_1}}}^2 \\
        &\quad + 24\eta_a^2Q^2 n\sigma^2\prt{1 + \frac{10\teta^2 m^2\Ln^2}{3n}} + \frac{40\eta_a^2Q^2 L^2}{3(1-\gamma\rho)^2}\sum_{\ell=0}^{Q-1}\condE{\norm{\z_{t_1}^\ell - \1 z_{t_1}^{\T}}^2}{\cF_{t_1}},
    \end{aligned}
\end{equation}
where we invoked the relation for $z_{t_1}$ and $z_{t_2}$ in \eqref{eq:zt1t2} and the upper bound for $\condEi{\norm{e_{t_1:t_2}}^2}{\cF_{t_1}}$ in \eqref{eq:et1t_ub}. 
Letting $\teta\leq (1-\gamma\rho)/[20m\sqrt{L^2 + 1/\gamma^2}]$ yields 
\begin{equation}
    \label{eq:zt2ell_zt2_ub0}
        \begin{aligned}
            &\sum_{\ell=0}^{Q-1}\E\brk{\norm{\z_{t_2}^\ell - \1 z_{t_2}^{\T}}^2}
            \leq  \frac{80\teta^2mL^2\eta_a^2 Q^2\Ln^2}{(1-\gamma\rho)^2}\sum_{r=t_{1}}^{t_2 - 1}\sum_{\ell=0}^{Q-1}\E\brk{\norm{\z_{r}^\ell -  \1z_{t_1}^{\T}}^2}\\
            &\quad + 7\eta_a^2Q^3 n \E\brk{\norm{\Fnor{z_{t_1}}}^2} + 25\eta_a^2Q^2 n\sigma^2 + \frac{40\eta_a^2Q^2L^2}{3(1-\gamma\rho)^2}\sum_{\ell=0}^{Q-1}\E\brk{\norm{\z_{t_1}^\ell - \1 z_{t_1}^{\T}}^2}\\
            &\quad + \frac{32\eta_a^2Q^2 L^2}{3(1-\gamma\rho)^2}\sum_{\ell=0}^{Q-1}\E\brk{\norm{\z_{t_2-1}^\ell - \1 z_{t_1}^{\T}}^2}.
        \end{aligned}
\end{equation}

Regarding the last term on the right-hand side of \eqref{eq:zt2ell_zt2_ub0}, we have
    \begin{align}
        \label{eq:sum_align_t2m1}
        \sum_{\ell=0}^{Q-1} \E\brk{\norm{\1 z_{t_1}^{\T} - \z_{t_2-1}^\ell}^2} \leq \sum_{r = t_1}^{t_2-1} \sum_{\ell=0}^{Q-1} \E\brk{\norm{\1 z_{t_1}^{\T} - \z_{r}^\ell}^2}.
    \end{align}
This term has the same form as the first term on the right-hand side of \eqref{eq:zt2ell_zt2_ub0}. We handle it using the following relation:
    \begin{equation}
        \label{eq:sum_align_t1t2}
        \begin{aligned}
            \sum_{r=t_1}^{t_2-1} \sum_{\ell=0}^{Q-1} \E\brk{\norm{\1 z_{t_1}^{\T} - \z_{r}^\ell}^2} &= \sum_{r=t_1 + 1}^{t_2} \sum_{\ell=0}^{Q-1} \E\brk{\norm{\1 z_{t_1}^{\T} - \z_{r}^\ell}^2} + \sum_{\ell=0}^{Q-1} \E\brk{\norm{\1 z_{t_1}^{\T} - \z_{t_1}^\ell}^2} \\
            &\quad - \sum_{\ell=0}^{Q-1} \E\brk{\norm{\1 z_{t_1}^{\T} - \z_{t_2}^\ell}^2}\\
            &\leq \sum_{r=t_1 + 1}^{t_2} \sum_{\ell=0}^{Q-1} \E\brk{\norm{\1 z_{t_1}^{\T} - \z_{r}^\ell}^2} + \sum_{\ell=0}^{Q-1} \E\brk{\norm{\1 z_{t_1}^{\T} - \z_{t_1}^\ell}^2}.
        \end{aligned}
    \end{equation}

Substituting \eqref{eq:sum_align_t2m1} and \eqref{eq:sum_align_t1t2} into \eqref{eq:zt2ell_zt2_ub0} and letting $\teta \leq 1/(20m\Ln)$ yields the desired result~\eqref{eq:zt2ell_zt2_ub}.

\subsection{Proof of Lemma \ref{lem:cL_descent}}
\label{app:cL_descent}

Substituting \eqref{eq:sum_align_t1t2} into \eqref{eq:cH_descent} yields
\begin{equation}
    \label{eq:cH_descent_align}
    \begin{aligned}
        &\E\brk{\cH_{t_2}}\leq \E\brk{\cH_{t_1}} - \frac{\teta m}{2}\prt{\cC_0 - \frac{32\teta^2 m^2}{\gamma^2(1-\gamma\rho)^2}}\E\brk{\norm{\Fnor{z_{t_1}}}^2} + \frac{5\teta \sigma^2}{nQ} \\
        &\quad + \frac{5\teta L^2}{nQ(1-\gamma\rho)^2}\sum_{t=t_1 + 1}^{t_2} \sum_{\ell=0}^{Q-1} \E\brk{\norm{\1 z_{t_1}^{\T} - \z_{t}^\ell}^2} + \frac{5\teta L^2}{nQ(1-\gamma\rho)^2}\sum_{\ell =0}^{Q-1}\E\brk{\norm{\z_{t_1}^{\ell} - \1 z_{t_1}^{\T}}^2}.
    \end{aligned}
\end{equation}
Combining \eqref{eq:zt2ell_zt2_ub} and \eqref{eq:cH_descent_align} and invoking the definition of $\cL_t$ in \eqref{eq:lya_cL} leads to 
\begin{equation}
    \label{eq:cL_t1t2_s1}
    \begin{aligned}
        &\E\brk{\cL_{t_2}}\leq \E\brk{\cH_{t_1}} + \brk{\frac{5}{25} + \frac{14\eta_a^2 Q^2L^2}{(1-\gamma\rho)^2}} \frac{25\teta L^2}{nQ(1-\gamma\rho)^2}\sum_{\ell=0}^{Q-1}\E\brk{\norm{\z_{t_1}^\ell - \1 z_{t_1}^{\T}}^2} \\
        &\quad - \frac{\teta m}{2}\brk{\cC_0 - \frac{32\teta^2 m^2}{\gamma^2(1-\gamma\rho)^2} - \frac{350\eta_a^2 Q^2 L^2}{m(1-\gamma\rho)^2}}\E\brk{\norm{\Fnor{z_{t_1}}}^2} + \frac{5\teta\sigma^2}{nQ} \\
        &\quad + \brk{\frac{1}{5} + \frac{11\eta_a^2 Q^2 L^2}{(1-\gamma\rho)^2}} \frac{25\teta L^2}{nQ(1-\gamma\rho)^2}\sum_{t=t_1 + 1}^{t_2} \sum_{\ell=0}^{Q-1} \E\brk{\norm{\1 z_{t_1}^{\T} - \z_{t}^\ell}^2} + \frac{625\teta \eta_a^2 Q L^2\sigma^2}{(1-\gamma\rho)^2}.
    \end{aligned}
\end{equation}

Letting $\eta_a\leq (1-\gamma\rho)/(20QL)$ gives
\begin{align*}
    \frac{1}{5} + \frac{11\eta_a^2 Q^2 L^2}{(1-\gamma\rho)^2} \leq \frac{1}{5} + \frac{14\eta_a^2 Q^2 L^2}{(1-\gamma\rho)^2}\leq \frac{1}{4}.
\end{align*}
Substituting \eqref{eq:ztell_zt1_ub} into \eqref{eq:cL_t1t2_s1} yields
\begin{equation}
    \label{eq:cL_t1t2_s2}
    \begin{aligned}
        &\E\brk{\cL_{t_2}}\leq \E\brk{\cH_{t_1}} + \brk{\frac{1}{4} + \frac{55\teta^2 m^2 L^2}{2(1-\gamma\rho)^2} + \frac{33\eta_a^2 Q^2 L^2 m}{4(1-\gamma\rho)^2} } \frac{25\teta L^2}{nQ(1-\gamma\rho)^2}\sum_{\ell=0}^{Q-1}\E\brk{\norm{\z_{t_1}^\ell - \1 z_{t_1}^{\T}}^2} \\
        &\quad - \frac{\teta m}{2}\brk{\cC_0 - \frac{32\teta^2 m^2}{\gamma^2(1-\gamma\rho)^2} - \frac{350\eta_a^2 Q^2 L^2}{m(1-\gamma\rho)^2}- \frac{825\teta^2 m^2 L^2}{2(1-\gamma\rho)^2} - \frac{275\eta_a^2 Q^2 L^2}{2(1-\gamma\rho)^2}}\E\brk{\norm{\Fnor{z_{t_1}}}^2}  \\
        &\quad + \frac{5\teta\sigma^2}{nQ} + \frac{625\teta \eta_a^2 Q L^2\sigma^2}{(1-\gamma\rho)^2} + \frac{1375\teta L^2\brk{\eta_a^2 Q^2m + \frac{2\teta^2 m^2}{n}}\sigma^2}{Q(1-\gamma\rho)^2}.
    \end{aligned}
\end{equation}
Finally, letting $\teta\leq (1-\gamma\rho)/(70m \sqrt{L^2 + 1/\gamma^2})$ and $\eta_a\leq (1-\gamma\rho)/(70Q L\sqrt{m})$ yields the desired result~\eqref{eq:cL_descent}.

\section{Proof of Main Results}
\label{app:thm}

This section provides proofs of Theorems~\ref{thm:ncvx} and \ref{thm:PL}. Throughout, we follow the multistep analysis described in the main text and use the subsequence $\{\tk_j\}_{j=0}^R$ introduced in Figure~\ref{fig:tkj}.

\subsection{Proof of Theorem \ref{thm:ncvx}}
\label{app:ncvx}

We start by relating the average $\sum_{t=0}^{T-1}\E\brki{\normi{\Fnor{z_t}}^2}$ to the subsequence terms \\
$\sum_{j=0}^R \E\brki{\normi{\Fnor{z_{\tk_j}}}^2}$ and error $\sum_{j=0}^R\sum_{s=0}^{m-1}\E\brki{\normi{z_{\tk_j+s}-z_{\tk_j}}^2}$.
Letting $\teta\leq 1/(m\Ln)$ leads to 
\begin{equation}
    \label{eq:Fnor_t}
    \begin{aligned}
        &\frac{1}{T}\sum_{t=0}^{T-1}\E\brk{\norm{\Fnor{z_t}}^2} = \frac{1}{T}\sum_{j=0}^{R-1}\sum_{s=0}^{m-1}\E\brk{\norm{\Fnor{z_{\tk_j + s}}}^2} + \sum_{s=0}^{S-1} \E\brk{\norm{\Fnor{z_{\tk_R + s}}}^2} \\
        &\leq \frac{1}{T}\sum_{j=0}^{R}\sum_{s=0}^{m-1}\E\brk{\norm{\Fnor{z_{\tk_j + s}}}^2}\\
        &\leq \frac{2m}{T}\sum_{j=0}^{R}\E\brk{\norm{\Fnor{z_{\tk_j}}}^2} + \frac{2\Ln^2}{T}\sum_{j=0}^{R}\sum_{s=0}^{m-1}\E\brk{\norm{z_{\tk_j + s} - z_{\tk_j}}^2}.
    \end{aligned}
\end{equation}

We now bound $\sum_{j=0}^R\sum_{s=0}^{m-1}\E\brki{\normi{z_{\tk_j + s} - z_{\tk_j}}^2}$ starting from \eqref{eq:zt_zt1} ($t_1 = \tk_j$ and $t = \tk_j + s$).
\begin{equation}
    \label{eq:ztkjs_ztkj}
    \begin{aligned}
        &\E\brk{\norm{z_{\tk_j + s} - z_{\tk_j}}^2} \leq 2\teta^2 s^2 \E\brk{\norm{\Fnor{z_{\tk_j}}}^2} + 2\E\brk{\norm{e_{\tk_j:(\tk_j + s)}}^2}\\
        &\leq \brk{2\teta^2 s^2 + \frac{32\teta^4 m^3s}{\gamma^2(1-\gamma\rho)^2} }\E\brk{\norm{\Fnor{z_{\tk_j}}}^2} + \frac{10\teta^2 s \sigma^2}{nQ} \\
        &\quad + \frac{10\teta^2 L^2 s}{nQ(1-\gamma\rho)^2}\sum_{r=\tk_j}^{\tk_{j+1} - 1}\sum_{\ell=0}^{Q-1}\E\brk{\norm{\z_r^\ell - \1 z_{\tk_j}^{\T}}^2},
    \end{aligned}
\end{equation}
where we invoked \eqref{eq:et1t_ub} by setting $t_1 = \tk_j$ and $t = \tk_{j} + s$. Summing over $s=0,1,\ldots, m-1$ and $j=0,1,\ldots,R$ yields
\begin{equation}
    \label{eq:ztkjs_ztkj_sum}
    \begin{aligned}
        &\sum_{j=0}^R\sum_{s=0}^{m-1}\E\brk{\norm{z_{\tk_j + s} - z_{\tk_j}}^2} \leq \brk{1 + \frac{16\teta^2 m^2}{\gamma^2(1-\gamma\rho)^2} }\teta^2 m^3 \sum_{j=0}^R\E\brk{\norm{\Fnor{z_{\tk_j}}}^2} \\
        &\quad  + \frac{5\teta^2 m^2 (R+1) \sigma^2}{nQ} + \frac{5\teta^2 L^2 m^2}{nQ(1-\gamma\rho)^2}\sum_{j=0}^R\sum_{r=\tk_j}^{\tk_{j+1} - 1}\sum_{\ell=0}^{Q-1}\E\brk{\norm{\z_r^\ell - \1 z_{\tk_j}^{\T}}^2}.
    \end{aligned}
\end{equation}
Substituting \eqref{eq:ztkjs_ztkj_sum} into \eqref{eq:Fnor_t} and letting $\teta\leq (1-\gamma\rho)/(4m\Ln)$ yields
\begin{equation}
    \label{eq:Fnor_t_ub}
    \begin{aligned}
        &\frac{1}{T}\sum_{t=0}^{T-1}\E\brk{\norm{\Fnor{z_t}}^2} \leq \frac{17m}{8T}\sum_{j=0}^{R}\E\brk{\norm{\Fnor{z_{\tk_j}}}^2} + \frac{10\teta^2m^2 \Ln^2 (R+1) \sigma^2}{nQT}\\
        &\quad + \frac{10\teta^2 L^2 m^2\Ln^2}{nQ(1-\gamma\rho)^2T}\sum_{j=0}^R\sum_{r=\tk_j}^{\tk_{j+1} - 1}\sum_{\ell=0}^{Q-1}\E\brk{\norm{\z_r^\ell - \1 z_{\tk_j}^{\T}}^2}.
    \end{aligned}
\end{equation}

We now handle the last term in \eqref{eq:Fnor_t_ub} based on Lemma~\ref{lem:sum_ub}. Setting $t_1 = \tk_j$ and $t_2 = \tk_{j+1}$ in \eqref{eq:ztell_zt1_ub} 
and summing over $j=0,1,\ldots,R$ yields 
\begin{equation}
    \label{eq:sum_ub_ztkj_sum}
    \begin{aligned}
        &\sum_{j=0}^R\sum_{t=\tk_j + 1}^{\tk_{j + 1}}\sum_{\ell=0}^{Q-1}\E\brk{\norm{\z_{t}^\ell - \1 z_{\tk_j}^{\T}}^2}\leq 11\brk{3\teta^2 m^2  + \eta_a^2 Q^2 }nmQ\sum_{j=0}^R\E\brk{\norm{\Fnor{z_{\tk_j}}}^2}\\
        &\quad + 55(R+1)\brk{\eta_a^2 Q^2m + \frac{2\teta^2 m^2}{n}}n\sigma^2 + \frac{110\teta^2 L^2m^2}{(1-\gamma\rho)^2}\sum_{j=1}^R\sum_{\ell=0}^{Q-1}\E\brk{\norm{ \z_{\tk_j}^\ell - \1 z_{\tk_j}^{\T}}^2} \\
        &\quad + \frac{33\eta_a^2 Q^2 L^2 m}{(1-\gamma\rho)^2}\sum_{j=1}^R\sum_{\ell=0}^{Q-1}\E\brk{\norm{\z_{\tk_j}^\ell - \1 z_{\tk_j}^{\T}}^2} + \frac{110\teta^2 L^2m^2}{(1-\gamma\rho)^2}\sum_{\ell=0}^{Q-1}\E\brk{\norm{ \z_{\tk_0}^\ell - \1 z_{\tk_0}^{\T}}^2} \\
        &\quad + \frac{33\eta_a^2 Q^2 L^2 m}{(1-\gamma\rho)^2}\sum_{\ell=0}^{Q-1}\E\brk{\norm{\z_{\tk_0}^\ell - \1 z_{\tk_0}^{\T}}^2}\\
        &\leq 11\brk{3\teta^2 m^2  + \eta_a^2 Q^2 }nmQ\sum_{j=0}^R\E\brk{\norm{\Fnor{z_{\tk_j}}}^2}\\
        &\quad + 55(R+1)\brk{\eta_a^2 Q^2m + \frac{2\teta^2 m^2}{n} + \frac{6\teta^2m^2\eta_a^2 Q^2 L^2}{(1-\gamma\rho)^2} + \frac{99\eta_a^4 Q^4mL^2}{55(1-\gamma\rho)^2}}n\sigma^2 \\
        &\quad + \frac{22\eta_a^2Q^3 n}{\gamma^2(1-\gamma\rho)^2}\prt{10\teta^2 m^2 L^2 + 3\eta_a^2Q^2L^2 m}\norm{\proxp{z_0} - z_0}^2 \\
        &\quad + \frac{22\eta_a^2Q^3 }{(1-\gamma\rho)^2}\prt{10\teta^2 m^2 L^2 + 3\eta_a^2Q^2L^2 m}\sumn\norm{\nabla f_i(\proxp{z_0})}^2\\
        &\quad + \frac{110\teta^2 L^2m^2}{(1-\gamma\rho)^2}\sum_{j=1}^R\sum_{\ell=0}^{Q-1}\E\brk{\norm{ \z_{\tk_j}^\ell - \1 z_{\tk_j}^{\T}}^2} + \frac{33\eta_a^2 Q^2 L^2 m}{(1-\gamma\rho)^2}\sum_{j=1}^R\sum_{\ell=0}^{Q-1}\E\brk{\norm{\z_{\tk_j}^\ell - \1 z_{\tk_j}^{\T}}^2},
    \end{aligned}
\end{equation}
where we invoked \eqref{eq:z0ell_z0_ub} in the last inequality by noting that $\tk_0 = 0$.
Notably, the left-hand side of \eqref{eq:sum_ub_ztkj_sum} can be split as 
\begin{equation}
    \label{eq:sum_ub_ztkj_sum_left}
    \begin{aligned}
        \sum_{j=0}^R\sum_{t=\tk_j + 1}^{\tk_{j + 1}}\sum_{\ell=0}^{Q-1}\E\brk{\norm{\z_{t}^\ell - \1 z_{\tk_j}^{\T}}^2} &= \sum_{j=1}^R\sum_{\ell=0}^{Q-1}\E\brk{\norm{\z_{\tk_j}^\ell - \1 z_{\tk_j}^{\T}}^2}\\
        &\quad + \sum_{j=0}^{R}\sum_{t=\tk_j + 1}^{\tk_{j + 1}}\sum_{\ell=0}^{Q-1}\E\brk{\norm{\z_{t}^\ell - \1 z_{\tk_j}^{\T}}^2}.
    \end{aligned}
\end{equation}
Letting $\teta\leq \prti{1-\gamma\rho}/\prti{24 m L},\; \eta_a\leq \prti{1-\gamma\rho}/\prti{24QL\sqrt{m}}$ yields
\begin{equation}
    \label{eq:sum_ub_ztkj_sum_ub0}
    \begin{aligned}
        &\sum_{j=0}^R\sum_{t=\tk_j + 1}^{\tk_{j + 1}}\sum_{\ell=0}^{Q-1}\E\brk{\norm{\z_{t}^\ell - \1 z_{\tk_j}^{\T}}^2}
        \leq 22\brk{3\teta^2 m^2  + \eta_a^2 Q^2 }nmQ\sum_{j=0}^R\E\brk{\norm{\Fnor{z_{\tk_j}}}^2}\\
        &\quad + 110(R+1)\brk{\eta_a^2 Q^2m + \frac{2\teta^2 m^2}{n} + \frac{6\teta^2m^2\eta_a^2 Q^2 L^2}{(1-\gamma\rho)^2} + \frac{99\eta_a^4 Q^4mL^2}{55(1-\gamma\rho)^2}}n\sigma^2 \\
        &\quad + \frac{44\eta_a^2Q^3 n}{\gamma^2(1-\gamma\rho)^2}\prt{10\teta^2 m^2 L^2 + 3\eta_a^2Q^2L^2 m}\norm{\proxp{z_0} - z_0}^2 \\
        &\quad + \frac{44\eta_a^2Q^3 }{(1-\gamma\rho)^2}\prt{10\teta^2 m^2 L^2 + 3\eta_a^2Q^2L^2 m}\sumn\norm{\nabla f_i(\proxp{z_0})}^2.
    \end{aligned}
\end{equation}
Noting that 
    \begin{align*}
        \sum_{j=0}^R\sum_{t=\tk_j}^{\tk_{j+1}-1} \sum_{\ell=0}^{Q-1} \E\brk{\norm{\1 z_{\tk_j}^{\T} - \z_{t}^\ell}^2} &\leq \sum_{j=0}^R\sum_{t=\tk_j + 1}^{\tk_{j + 1}}\sum_{\ell=0}^{Q-1}\E\brk{\norm{\z_{t}^\ell - \1 z_{\tk_j}^{\T}}^2} + \sum_{\ell=0}^{Q-1}\E\brk{\norm{\z_{0}^\ell - \1 z_{0}^{\T}}^2},
    \end{align*}
we obtain from \eqref{eq:sum_ub_ztkj_sum_ub0} and \eqref{eq:z0ell_z0_ub} that
\begin{equation}
    \label{eq:sum_ub_ztkj_sum_ub}
    \begin{aligned}
        &\sum_{j=0}^R\sum_{t=\tk_j}^{\tk_{j+1}-1} \sum_{\ell=0}^{Q-1} \E\brk{\norm{\1 z_{\tk_j}^{\T} - \z_{t}^\ell}^2} \leq 22\brk{3\teta^2 m^2  + \eta_a^2 Q^2 }nmQ\sum_{j=0}^R\E\brk{\norm{\Fnor{z_{\tk_j}}}^2}\\
        &\quad + 110(R+1)\brk{\eta_a^2 Q^2m + \frac{2\teta^2 m^2}{n} + \frac{6\teta^2m^2\eta_a^2 Q^2 L^2}{(1-\gamma\rho)^2} + \frac{99\eta_a^4 Q^4mL^2}{55(1-\gamma\rho)^2} + \frac{3\eta_a^2 Q^2}{R + 1}}n\sigma^2 \\
        &\quad + \frac{44\eta_a^2Q^3 n}{\gamma^2(1-\gamma\rho)^2}\brk{10\teta^2 m^2 L^2 + 3\eta_a^2Q^2L^2 m + \frac{(1-\gamma\rho)^2}{22}}\norm{\proxp{z_0} - z_0}^2 \\
        &\quad + \frac{44\eta_a^2Q^3 }{(1-\gamma\rho)^2}\brk{10\teta^2 m^2 L^2 + 3\eta_a^2Q^2L^2 m + \frac{(1-\gamma\rho)^2}{22}}\sumn\norm{\nabla f_i(\proxp{z_0})}^2.
    \end{aligned}
\end{equation}
Substituting \eqref{eq:sum_ub_ztkj_sum_ub} into \eqref{eq:Fnor_t_ub}, letting $\teta\leq \min\crki{(1-\gamma\rho)/(20m L), (1-\gamma\rho)/(4m\Ln)}$, and $\eta_a \leq (1-\gamma\rho)/(20 QL)$ yields

\begin{equation}
    \label{eq:Fnor_t_ub1}
    \begin{aligned}
        &\frac{1}{T}\sum_{t=0}^{T-1}\E\brk{\norm{\Fnor{z_t}}^2} \leq \frac{3m}{T}\sum_{j=0}^{R}\E\brk{\norm{\Fnor{z_{\tk_j}}}^2} + \frac{10\teta^2m^2 \Ln^2 (R+1) \sigma^2}{nQT}\\
        &\quad + \frac{1100(R+1)\teta^2 m^2 L^2 \Ln^2}{nQ(1-\gamma\rho)^2 T}\brk{2\eta_a^2 Q^2m + \frac{2\teta^2 m^2}{n} + \frac{3\eta_a^2 Q^2}{R + 1}}n\sigma^2 \\
        &\quad + \frac{440\teta^2 m^2 L^2 \Ln^2\eta_a^2Q^2 }{\gamma^2(1-\gamma\rho)^4T}\norm{\proxp{z_0} - z_0}^2 \\
        &\quad + \frac{440\teta^2 m^2 L^2\Ln^2\eta_a^2Q^2 }{n(1-\gamma\rho)^4T}\sumn\norm{\nabla f_i(\proxp{z_0})}^2.
    \end{aligned}
\end{equation}

It remains to bound $\sum_{j=0}^R\E\brki{\normi{\Fnor{z_{\tk_j}}}^2}$ in \eqref{eq:Fnor_t_ub1}. This is obtained from the descent property of $\cL_t$ established in Lemma~\ref{lem:cL_descent}.
Setting $t_1 = \tk_j$ and $t_2 = \tk_{j+1}$ in \eqref{eq:cL_descent} and summing over $j=0,1,\ldots,R$ yields
\begin{equation}
    \label{eq:cL_tkj1_sum}
    \begin{aligned}
        \sum_{j=0}^R\E\brk{\norm{\Fnor{z_{\tk_j}}}^2}&\leq \frac{9\crk{\E\brk{\cL_{\tk_0}} - \E\brki{\cL_{\tk_{R+1}}}}}{\teta m} + \frac{54 \sigma^2 (R+1)}{m nQ}+ \frac{18000 \eta_a^2 QL^2 \sigma^2(R+1)}{(1-\gamma\rho)^2}.
    \end{aligned}
\end{equation}

Substituting \eqref{eq:cL_tkj1_sum} into \eqref{eq:Fnor_t_ub1}, invoking \eqref{eq:z0ell_z0_ub}, noting that $m(R+1)\leq 2T$, and letting $\teta\leq (1-\gamma\rho)/(70m\Ln)$ and $\eta_a\leq (1-\gamma\rho)/(70Q\sqrt{m(L^2 + 1/\gamma^2)})$ yields the desired result~\eqref{eq:ncvx}.

For $\teta$, $\eta_a$, $m$, and $\gamma$ satisfying \eqref{eq:ncvx_teta}, we have $1/m\leq \sqrt{9(L+\rho)nQ\Delta_\psi/(\sigma^2 T)}$, and
\begin{equation}
    \label{eq:teta_res}
    \begin{aligned}
        &\frac{\Delta_\psi}{\teta T} = 320\sqrt{\frac{\sigma^2(\rho + L)\Delta_\psi}{nQ T}},\; \teta^2\leq \frac{nQ\Delta_\psi}{320^2 (L + \rho) T \sigma^2},\;\eta_a^2\leq \frac{\Delta_\psi}{240^2 (L + \rho) T Q\sigma^2}.
    \end{aligned}
\end{equation}

Substituting \eqref{eq:teta_res} into \eqref{eq:ncvx} yields \eqref{eq:ncvx_order}.

\subsection{Proof of Theorem \ref{thm:PL}}
\label{app:PL}

Since $\Fnor{z_t}\in\partial \psi(\proxp{z})$ and $x_t = \proxp{z_t}$, Assumption~\ref{as:PL} implies \eqref{eq:PL_Fnor_sketch}.
    Therefore,
    \begin{equation}
        \label{eq:PL_cH}
        \frac{2\mu}{1 + \gamma\cC_0\mu}\prt{\cH_t - \psi^*} \leq \norm{\Fnor{z_t}}^2,\;\forall t\geq 0.
    \end{equation}
    Letting $t_1 = \tk_j$ and $t_2 = \tk_{j+1}$ in \eqref{eq:cL_t1t2_s2} and invoking \eqref{eq:PL_cH} yields 
    \begin{equation}
        \label{eq:cL_s2}
        \begin{aligned}
            &\E\brk{\cL_{\tk_{j + 1}}} - \psi^* \leq \brk{1 - \frac{2\teta m\mu}{9(1+\gamma\cC_0 \mu)}}\crk{\E\brk{\cH_{\tk_j}} - \psi^*} + \frac{6\teta \sigma^2 }{nQ} + \frac{2000\teta\eta_a^2 L^2 Qm\sigma^2}{(1-\gamma\rho)^2} \\
            &\quad + \brk{\frac{1}{4} + \frac{55\teta^2 m^2 L^2}{2(1-\gamma\rho)^2} + \frac{33\eta_a^2 Q^2 L^2 m}{4(1-\gamma\rho)^2} } \frac{25\teta L^2}{nQ(1-\gamma\rho)^2}\sum_{\ell=0}^{Q-1}\E\brk{\norm{\z_{\tk_j}^\ell - \1 z_{\tk_j}^{\T}}^2}.
        \end{aligned}
    \end{equation}
    Letting 
    \begin{align*}
        \teta\leq \min\crk{\frac{1-\gamma\rho}{20mL }, \frac{1+\gamma\mu \cC_0}{m\mu}},\; \eta_a\leq \frac{1-\gamma\rho}{20\sqrt{m}QL}
    \end{align*}
    leads to 
    \begin{align*}
        \frac{1}{4} + \frac{55\teta^2 m^2 L^2}{2(1-\gamma\rho)^2} + \frac{33\eta_a^2 Q^2 L^2 m}{4(1-\gamma\rho)^2}\leq 1 - \frac{2\teta m\mu}{9(1 + \gamma\cC_0 \mu)}.
    \end{align*}
    Consequently, we obtain from \eqref{eq:cL_s2} that
    \begin{equation}
        \label{eq:cLtkj1_contr}
        \begin{aligned}
            \E\brk{\cL_{\tk_{j + 1}}} &\leq \brk{1 - \frac{2\teta m\mu}{9(1 + \gamma\cC_0 \mu)}}\E\brk{\cL_{\tk_j}} + \frac{6\teta \sigma^2 }{nQ} + \frac{2000\teta\eta_a^2 L^2 Qm\sigma^2}{(1-\gamma\rho)^2}.
        \end{aligned}
    \end{equation}

    For $\tk_R\leq t\leq \tk_R + S = T$, a similar derivation to \eqref{eq:cLtkj1_contr} yields
    \begin{equation}
        \label{eq:cLtkj1_contr_R}
        \begin{aligned}
            \E\brk{\cL_{\tk_{R} + S}} &\leq \brk{1 - \frac{2\teta S\mu}{9(1 + \gamma\cC_0 \mu)}}\E\brk{\cL_{\tk_R}}+ \frac{6\teta \sigma^2 }{nQ} + \frac{2000\teta\eta_a^2 L^2 Qm\sigma^2}{(1-\gamma\rho)^2}.
        \end{aligned}
    \end{equation}

    We then unroll the recursion \eqref{eq:cLtkj1_contr} across the subsequence $\{\tk_j\}$ together with \eqref{eq:cLtkj1_contr_R} to obtain
    \begin{equation}
        \label{eq:cL_t}
        \begin{aligned}
            &\E\brk{\cL_{T}} \leq  \brk{1 - \frac{2\teta S\mu}{9(1 + \gamma\cC_0 \mu)}}\brk{1 - \frac{2\teta m\mu}{9(1 + \gamma\cC_0 \mu)}}^R\E\brk{\cL_0} + \frac{6\teta \sigma^2}{nQ} + \frac{2000\teta\eta_a^2 L^2 Qm\sigma^2}{(1-\gamma\rho)^2}\\
            &\quad + \frac{27 (1+\gamma\mu \cC_0) \sigma^2}{ \mu mnQ} + \frac{9000(1+\gamma\mu\cC_0) \eta_a^2 Q L^2 \sigma^2}{ \mu (1-\gamma\rho)^2}\\
            &\leq \brk{1 - \frac{2\teta S\mu}{9(1 + \gamma\cC_0 \mu)}}\brk{1 - \frac{2\teta m\mu}{9(1 + \gamma\cC_0 \mu)}}^R\E\brk{\cL_0} + \frac{36\sigma^2}{\mu mnQ} + \frac{12100\eta_a^2 Q L^2 \sigma^2}{ \mu (1-\gamma\rho)^2},
        \end{aligned}
    \end{equation}
    where we let $\teta \leq (1 + \gamma\mu\cC_0)/(10 m\mu)$ and $\gamma\leq 1/(5\mu)$. 

    Finally, using the inequality $ (1-x)^k \leq \exp(-kx),\;\forall x\in[0,1],\ k\geq 0,$
    we obtain \eqref{eq:linear} from \eqref{eq:cL_t}.

\vskip 0.2in
\bibliography{reference_all}

\end{document}